\documentclass{article}
\usepackage[hyphens]{url} 
\usepackage{graphicx}
\usepackage[square,numbers]{natbib}  
\usepackage{caption} 

\usepackage{algorithm}
\usepackage{algorithmic}
\usepackage{newfloat}
\usepackage{listings}
\usepackage{fullpage}

\usepackage{url}
\usepackage{hyperref}       
\usepackage{booktabs}

\usepackage{olo}
\usepackage{color}
\usepackage{xcolor}
\usepackage{adjustbox}
\usepackage{multirow}
\usepackage{wrapfig}
\usepackage{amsmath}
\usepackage{amssymb}
\usepackage{subcaption}
\usepackage{enumitem}

\hypersetup{
    colorlinks=true,
    linkcolor=blue,
    filecolor=blue,      
    urlcolor=blue,
    citecolor=red
    }

\newcommand{\code}{{\url{https://github.com/purushottamkar/svam/}}}

\newcommand{\step}{\xi}

\newcommand{\gem}{\textsc{SVAM}\xspace}
\newcommand{\gemrr}{\textsc{SVAM-RR}\xspace}

\newcommand{\gemme}{\textsc{SVAM-ME}\xspace}

\newcommand{\gemlr}{\textsc{SVAM-LR}\xspace}
\newcommand{\gemgam}{\textsc{SVAM-Gamma}\xspace}
\newcommand{\vwo}{\vw^\ast}
\newcommand{\hvw}{\hat\vw}
\newcommand{\tvw}{\tilde\vw}
\newcommand{\ty}{\tilde y}

\newcommand{\vmuo}{\vmu^\ast}

\newcommand{\hvmu}{\hat\vmu}

\newcommand{\betao}{\beta^\ast}
\newcommand{\cons}{\tilde c}

\newcommand{\ncons}{\tilde q}
\newcommand{\rcons}{\tilde g}

\newcommand{\bcn}[1]{\{{#1}\}}

\newcommand{\deff}{\overset{\mathrm{def}}{=}}

\graphicspath{{figs/}}

\def\mytitle{Corruption-tolerant Algorithms for Generalized Linear Models}\relax

\title{\mytitle}

\author{
	Bhaskar Mukhoty\thanks{Mohamed Bin Zayed University of Artificial Intelligence, Abu Dhabi, UAE} \thanks{Work done while the author was a student at IIT Kanpur}
	\and Debojyoti Dey\thanks{Indian Institute of Technology Kanpur, Uttar Pradesh, India}
	\and Purushottam Kar\footnotemark[2]\\
	bhaskar.mukhoty@mbzuai.ac.ae, \{debojyot,purushot\}@cse.iitk.ac.in
}

\begin{document}

\maketitle

\begin{abstract}
This paper presents \gem (Sequential Variance-Altered MLE), a unified framework for learning generalized linear models under adversarial label corruption in training data. \gem extends to tasks such as least squares regression, logistic regression, and gamma regression, whereas many existing works on learning with label corruptions focus only on least squares regression. \gem is based on a novel variance reduction technique that may be of independent interest and works by iteratively solving weighted MLEs over variance-altered versions of the GLM objective. SVAM offers provable model recovery guarantees superior to the state-of-the-art for robust regression even when a constant fraction of training labels are adversarially corrupted. \gem also empirically outperforms several existing problem-specific techniques for robust regression and classification. Code for \gem is available at \code
\end{abstract}
\section{Introduction}
\label{sec:intro}
Generalized linear models (GLMs) \cite{NelderWedderburn1972} are effective models for a variety of discrete and continuous label spaces, allowing the prediction of binary or count-valued labels (logistic, Poisson regression) as well as real-valued labels (gamma, least-squares regression). Inference in a GLM involves two steps: given a feature vector $\vx \in \bR^d$ and model parameters $\vwo$, a \emph{canonical parameter} is generated as $\theta := \ip\vwo\vx$ then the label $y$ is sampled from the exponential family distribution
\[
\P{y \cond \theta} = \exp( y\cdot\theta - \psi(\theta) - h(y)),
\]
where the function $h(\cdot)$ is specific to the GLM and $\psi(\cdot)$ is a normalization term, also known as log partition function. It is common to use a \emph{non-canonical link} such as $\theta := \exp(\ip\vwo\vx)$ for gamma distribution. GLMs also admit vector valued label $\vy\in\bR^n$ by substituting the scalar product by inner product $\ip\vy\veta$ where $\veta := \vX\vwo$ is the canonical parameter and $\vX\in\bR^{n\times d}$ is the covariate matrix.

\textbf{Problem Description:}
Given data $\bcn{(\vx^i,y_i)}_{i=1}^n$ generated using a known GLM but unknown model parameters $\vwo$, statistically efficient techniques exist to recover a consistent estimate of the model $\vwo$ \cite{McCullaghNelder1989}. However, these techniques break down if several observed labels $y_i$ are corrupted, not just by random statistical noise but by adversarially generated structured noise. Suppose $k < n$ labels are corrupted i.e. for some $k$ data points $i_1,\ldots,i_k$, the actual label $y_{i_j}, j = 1,\ldots,k$ generated by the GLM are replaced by the adversary with corrupted ones say $\tilde y_{i_j}$. Can we still recover $\vwo$? Note that the learning algorithm is unaware of the points that are corrupted.

\textbf{Breakdown Point:} The largest fraction $\alpha = k/n$ of corruptions that a learning algorithm can tolerate while still offering an estimate of $\vwo$ with bounded error is known as its breakdown point. This paper proposes the \gem algorithm that can tolerate $k = \Om n$ corruptions i.e. $\alpha = \Om1$.

\textbf{Adversary Models:} 
Contamination of the training labels $y_1,\ldots,y_n$ by an adversary can misguide the learning algorithm into selecting model parameters of the adversary's choice. An adversary has to choose (1) which labels $i_1,\ldots,i_k$ to corrupt and (2) what corrupted labels $\ty_{i_1},\ldots,\ty_{i_k}$ to put there. Adversary models emerge based on what information the adversary can consult while making these choices. The \textit{oblivious} adversary must make both these choices with no access to the original data $\bcn{(\vx^i, y_i)}_{i=1}^n$ or true model $\vwo$ and thus, can only corrupt a random/fixed subset of $k$ labels by sampling $\ty_{i_j}$ from some predetermined noise distribution. This is also known as the \textit{Huber} noise model. On the other hand, a \textit{fully adaptive} adversary has full access to the original data and true model while making both choices. Finally, the \textit{partially adaptive} adversary must choose the corruption locations without knowledge of original data or true model but has full access to these while deciding the corrupted labels. See Appendix~\ref{app:adversary} for details. 

\textbf{Contributions:} This paper describes the \gem (Sequential Variance-Altered MLE) framework that offers:\\
\textbf{1.} robust estimation with a breakdown point $\alpha = \Om1$ against partially and fully adaptive adversaries for robust least-squares regression and mean estimation and $\alpha = \Om{1/\sqrt d}$ for robust gamma regression. Prior works do not offer any breakdown point for gamma regression.\\
\textbf{2.} exact recovery of the true model $\vwo$ against a fully-adaptive adversary for the case of least squares regression,\\
\textbf{3.} the use of variance reduction technique (see \S\ref{sec:variance}) in robust learning, which is novel to the best of our knowledge,\\
\textbf{4.} extensive empirical evaluation demonstrating that despite being a generic framework, \gem is competitive to or outperforms algorithms specifically designed to solve problems such as least-squares and logistic regression.

\section{Related Works}
\label{sec:related}

In the interest of space, we review aspects of literature most related to \gem and refer to others \cite{diakonikolas2019robust,MukhotyGJK2019} for a detailed review. 

Robust GLM learning has been studied in a variety of settings. \citep{CantoniRonchetti2001} considered an oblivious adversary (Huber's noise model) but offered a breakdown point of $\alpha = \bigO{\frac1{\sqrt n}}$ i.e. tolerate $k \leq \bigO{\sqrt n}$ corruptions. \citep{YangTR2013} solve robust GLM estimation by solving M-estimation problems. However, they require the magnitude of the corruptions to be upper-bounded by some constant i.e. $\abs{y_i - \ty_i} \leq \bigO1$ and offer a breakdown point of $\alpha = \bigO{\frac1{\sqrt n}}$. Moreover, their approach solves $L_1$-regularized problems using projected gradient descent that offers slow convergence. In contrast, \gem offers a linear rate of convergence, offers a breakdown point of $\alpha = \Om1$ i.e. tolerate $k = \Om n$ corruptions and can tolerate corruptions with unbounded magnitude introduced by a partially or fully adaptive adversary.

Specific GLMs such as robust regression have received focused attention. Here the model is $\vy = X\vwo + \vb$ where $X \in \bR^{n \times d}$ is the feature matrix and $\vb$ is $k$-sparse corruption vector denoting the adversarial corruptions. A variant of this, studies a \emph{hybrid} noise model that replaces the zero entries of $\vb$ with Gaussian noise $\cN(0,\sigma^2)$. \citep{NguyenT2013,WrightM2010} solve an $L_1$ minimization problem which is slow in practice. (see \S\ref{sec:exps}). \citep{BhatiaJK2015} use hard thresholding techniques to estimate the subset of uncorrupted points while \citep{MukhotyGJK2019} modify the IRLS algorithm to do so. However, \cite{BhatiaJK2015, MukhotyGJK2019} are unable to offer consistent model estimates in the hybrid noise model even if the corruption rate $\alpha = k/n \rightarrow 0$ which is surprising since $\alpha \rightarrow 0$ implies vanishing corruption. In contrast, \gem offers consistent model recovery in the hybrid noise model against a fully adaptive adversary when $\alpha \rightarrow 0$. \cite{SuggalaBRJ2019} also offer consistent recovery with breakdown points $\alpha > 0.5$ but assume an oblivious adversary.

Robust classification with $y_i \in \bc{-1,+1}$ has been explored using robust surrogate loss functions \cite{natarajan2013learning} and ranking \cite{FengXMY2014,northcutt2017rankpruning} techniques. These works do not offer breakdown points but offer empirical comparisons.

Robust mean estimation entails recovering an estimate $\hvmu \in \bR^d$ of the mean $\vmuo$ of a multivariate Gaussian $\cN(\vmuo, \Sigma)$ given $n$ samples of which an $\alpha$ fraction are corrupted~\cite{lai2016agnostic}. Estimation error is known to be lower bounded $\norm{\hvmu - \vmuo}_2 \geq \Om{\alpha\sqrt{\log\frac{1}{\alpha}}}$ for this application even if $n\to \infty$~\cite{diakonikolas2019recent}. \citep{diakonikolas2019robust} use convex programming techniques and offer $\bigO{\alpha\log^{\frac32}{\frac1\alpha}}$ error given $n \geq \softOm{\frac{d^2}{\alpha^2}}$ samples and a $\text{poly}\br{n,d,\frac1\alpha}$ runtime. \cite{cheng2019high} improve the running time to $\frac{\softO{nd}}{\text{poly}(\alpha)}$. The recent work of \cite{DalalyanM2022} uses an IRLS-style approach that internally relies on expensive SDP-calls but offers high breakdown points. \gem uses $n = \bigO{\log^2\frac1\alpha}$ samples and offers a recovery error of $\bigO{\textit{trace}(\Sigma)(\log\frac1\alpha)^{-1/2}}$. This is comparable to existing works if $\textit{trace}(\Sigma)=\bigO{1}$. Moreover, \gem is much faster and simpler to implement in practice.

\emph{Meta algorithms} such as robust gradient techniques, median-of-means \citep{LecueL2020}, tilted ERM \citep{LiBSS2020} and maximum correntropy criterion \citep{FengHSYS2015} have been studied. SEVER~\cite{DiakonikolasKKLSS2019} uses gradient covariance matrix to filter out the outliers along its largest eigenspace while RGD~\cite{PrasadSBR2018} uses robust gradient estimates to perform robust first-order optimization directly. While convenient to execute, they may require larger training sets, e.g., SEVER requires $n > d^5$ samples for robust least-squares regression whereas \gem requires $n > \Om{d\log(d)}$. In terms of recovery guarantees, for least-squares regression without Gaussian noise, \gem and other methods \cite{BhatiaJK2015,MukhotyGJK2019}) offer exact recovery of $\vwo$ so long as the fraction of corrupted points is less than the breakdown point while SEVER's error continues to be bounded away from zero. RGD only considers an oblivious/Huber adversary while \gem can tolerate partially/fully adaptive adversaries. SEVER does not report an explicit breakdown point, RGD offers a breakdown point of $\alpha =  1/\log d$ (see Thm 2 in their paper) while \gem offers an explicit breakdown point independent of $d$. \gem also offers faster convergence than existing methods such as SEVER and RGD.

\section{The \gem Algorithm}
\label{sec:gem}
A popular approach in robust learning is to assign weights to data points, hoping that large weights would be given to uncorrupted points and low weights to corrupted ones, followed by weighted likelihood maximization. Often the weights are updated, and the process is repeated.  \cite{CantoniRonchetti2001} use Huber style weighing functions used in Mallow's type M-estimators, \cite{MukhotyGJK2019} use truncated inverse residuals, and \cite{ValdoraYohai2014} use Mahalanobis distance-based weights.

\gem notes that the label likelihood offers a natural measure of how likely the point is to be uncorrupted. Given a model estimate $\hvw^t$ at iteration $t$, the weight $s_i = \P{y_i\cond\eta^t_i} = \exp(y_i\cdot\eta_i^t - \psi(\eta_i^t) - h(y_i))$ can be assigned to the $i\nth$ point where $\eta^t_i = \ip{\hvw^t}{\vx^i}$. This gives us the weighted MLE\footnote{Recall that for gamma/Poisson regression we need to set $\eta^t_i = \exp(\ip{\hvw^t}{\vx^i})$ given the non-canonical link for these problems.} $\tilde Q(\vw \cond \hvw^t) = -\sum_{i=1}^n s_i \cdot \log \P{y_i\cond\ip{\vw}{\vx^i}}$ solving which gives us the next model iterate as
\begin{equation}
\hvw^{t+1} = \arg\min_{\vw \in \bR^d}\ \tilde Q(\vw \cond \hvw^t)
\label{eq:wmle}
\end{equation}
However, as \S\ref{sec:exps} will show, this strategy does not perform well. If the initial model $\hvw^1$ is far from $\vwo$, it may result in imprecise weights $s_i$ that are large for the corrupted points. For example, if the adversary introduces corruptions using a different model $\tvw$ i.e. $\tilde y_{i_j} \sim \P{y_i\cond\ip{\tvw}{\vx^{i_j}}}, j \in [k]$ and we happen to initialize close to $\tvw$ i.e. $\hvw^1 \approx \tvw$, then it is the corrupted points that would get large weights initially that may cause the algorithm to converge to $\tvw$ itself.

\textbf{Key Idea:} It is thus better to avoid drastic decisions, say setting $s_i \gg 0$ in the initial stages no matter how much a data point appears to be clean. \gem implements this intuition by setting weights using a label likelihood distribution with very large variance initially. This ensures that no data point (not even the uncorrupted ones) gets large weight (c.f. the uniform distribution that has large variance and assigns to point a high density). As \gem progresses towards $\vwo$, it starts using likelihood distributions with progressively lower variance. Note that this allows data points (hopefully the uncorrupted ones) to get larger weights (c.f. the Dirac delta distribution that has vanishing variance and assigns high density to isolated points).

\subsection{Mode-preserving Variance-altering Likelihood Transformations}
\label{sec:variance}

\begin{table*}[t]
	\centering
	\caption{\small{Some common distributions and their variance altered forms. Note that in all cases, the form of the distribution is preserved after transformation, as well as that the variance asymptotically goes down at the rate $\Theta(1/\beta)$ as $\beta \rightarrow \infty$. The figures on the right show examples of how the Gaussian and gamma likelihood functions change with varying values of $\beta$ while still being order/mode preserving.}}
	\label{tab:variance}
	
	\begin{minipage}{0.75\linewidth}
		\resizebox{\textwidth}{!}{
			\centering
			\begin{tabular}{ccccc}
				\toprule
				Name & Standard Form & Variance Altered Form & Variance & Asymptotic Form \\
				& (Mass/Density function) & $(\beta)$ & & $(\beta \rightarrow \infty)$\\
				\midrule
				Gaussian (univariate) & \multirow2*{$\sqrt{\frac1{2\pi}}\exp(-\frac12\br{y-\eta}^2)$} & \multirow2*{$\sqrt{\frac\beta{2\pi}}\exp(-\frac\beta2\br{y-\eta}^2)$} & \multirow2*{$\frac1\beta$} & \multirow2*{$\delta_{\eta}(y)$}\\
				$\cN(y \cond \eta)$ & & & &\\
				\midrule
				Gaussian (multivariate) & \multirow2*{$\br{\frac1{2\pi}}^{\frac d2}\exp(-\frac12\norm{\vy-\veta}_2^2)$} & \multirow2*{$\br{\frac\beta{2\pi}}^{\frac d2}\exp(-\frac\beta2\norm{\vy-\veta}_2^2)$} & \multirow2*{$\frac1\beta$} & \multirow2*{$\delta_{\veta}(\vy)$}\\
				$\cN(\vy \cond \veta)$ & & & &\\
				\midrule
				Bernoulli & $\P{y = 1 \cond \eta} = \pi$ & $\P{y = 1 \cond \eta} = \tilde\pi$ & \multirow2*{$< \frac1{\beta\eta}$} & \multirow2*{$\delta_{\sign(\eta)}(y)$}\\
				{$y \in \bc{-1,+1}$} & $\pi = (1 + \exp(-y\eta))^{-1}$ & $\tilde\pi = (1 + \exp(-\beta y\eta))^{-1}$ & &\\
				\midrule
				& \multirow2*{$\frac{1}{y \Gamma(\frac{1}{\phi})}\left(\frac{y\eta}{\phi}\right)^\frac{1}{\phi}\exp(-\frac{y\eta}{\phi})$} & \multirow2*{$\frac{1}{y \Gamma(\frac{1}{\tilde\phi_\beta})}\left(\frac{y\tilde\eta_\beta}{\tilde\phi_\beta}\right)^\frac{1}{\tilde\phi_\beta}\exp(-\frac{y\tilde\eta_\beta}{\tilde\phi_\beta})$} & \multirow4*{$\frac\phi{\eta^2}\frac{\phi+\beta(1 - \phi)}{\beta^2}$} & \multirow4*{$\delta_{\frac{1-\phi}\eta}(y)$}\\
				Gamma & & & &\\
				$\cG(y \cond \eta, \phi)$ & $\phi < 1$ & \scriptsize{$\tilde\phi_\beta = \phi/\br{\phi + \beta(1-\phi)}$} & &\\
				& \textbf{Note}: $\eta = \exp(\ip\vw\vx)$ & \scriptsize{$\tilde\eta_\beta = \eta\beta/\br{\phi + \beta(1-\phi)}$} & &\\
				\bottomrule
			\end{tabular}
		}
	\end{minipage}\hfill
	\begin{minipage}{0.25\linewidth}
	    \centering
		\includegraphics[width=0.7\textwidth]{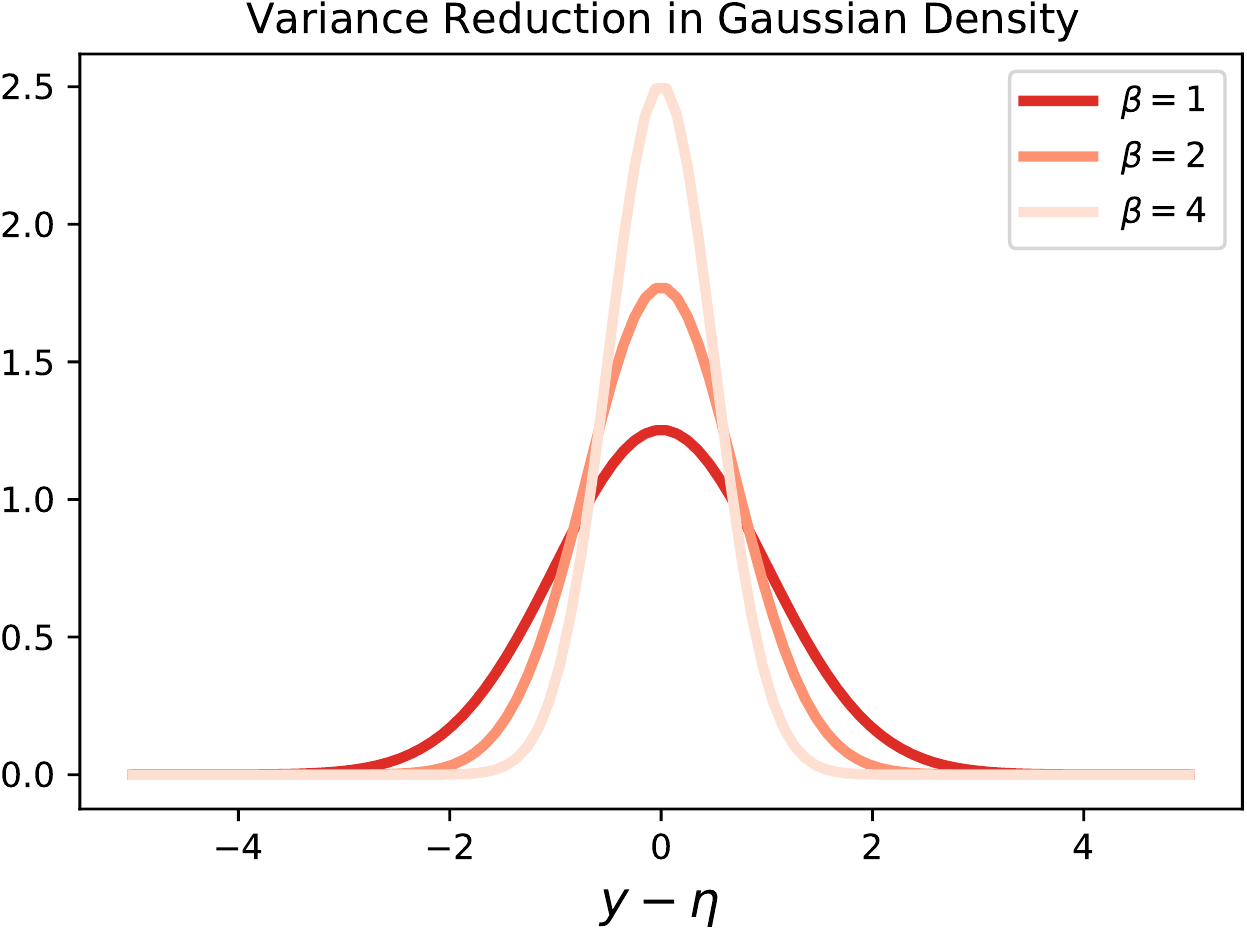}
		\includegraphics[width=0.7\textwidth]{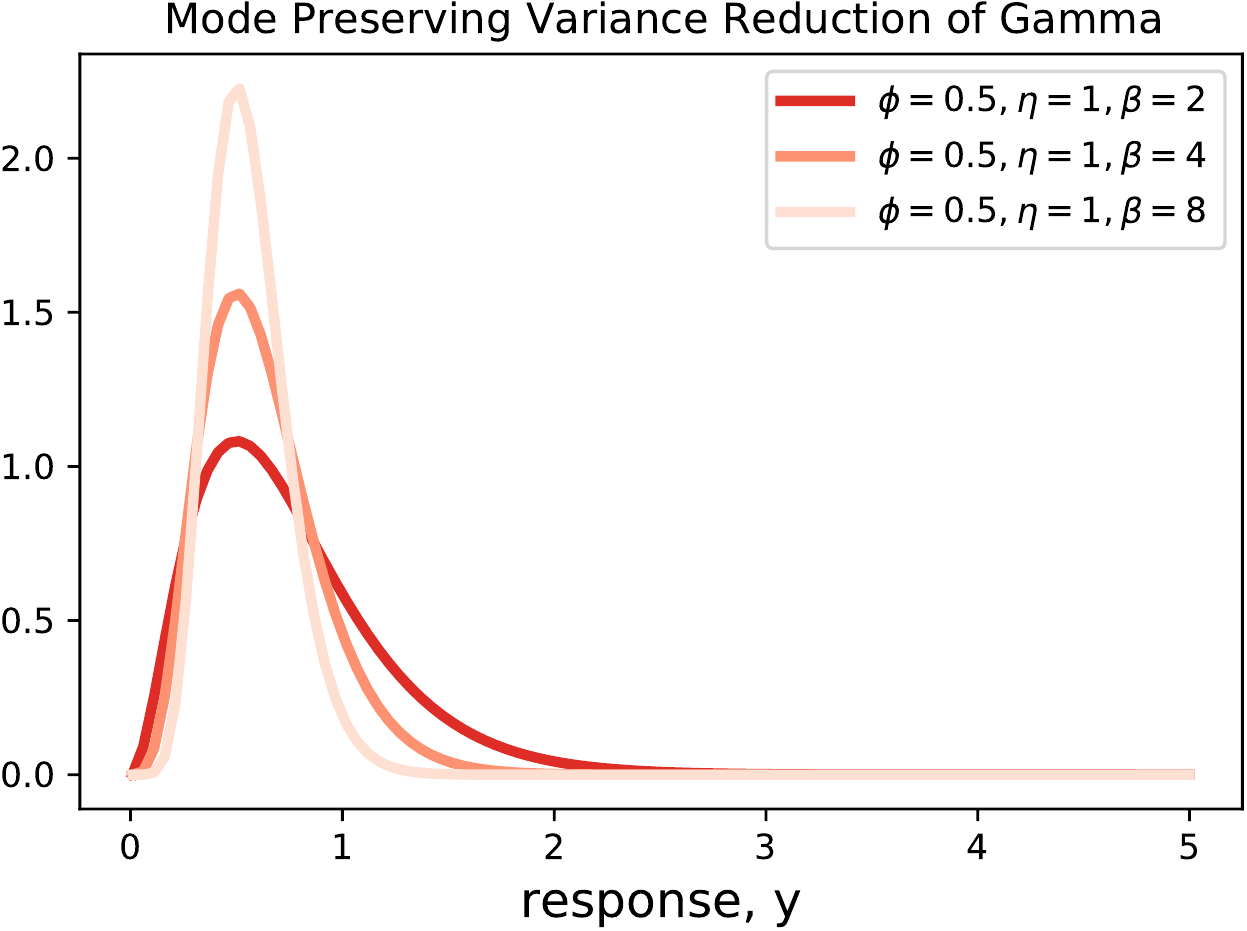}
	\end{minipage}
\end{table*}

To implement the above strategy, \gem (Algorithm~\ref{algo:gem}) needs techniques to alter the variance of a likelihood distribution at will. Note that the likelihood values of the altered distributions must be computable as they will be used as weights $s_i$ i.e. merely being able to sample the distribution is not enough. Moreover, the transformation must be order-preserving -- say the original and transformed distributions are $\bP$ and $\tilde\bP$ resp., then for every pair of labels $y, y'$ and every parameter value $\eta$, we must have $\P{y \cond \eta} > \P{y' \cond \eta} \Leftrightarrow \tilde\bP\bs{y \cond \eta} > \tilde\bP\bs{y' \cond \eta}$. If this is not true, then \gem could exhibit anomalous behavior.

\textbf{The Transformation:} If $\P{y \cond \eta} = \exp(y\cdot\eta - \psi(\eta) - h(y))$ is an exponential family distribution with parameter $\eta$ and log-partition function $\psi(\eta) = \log\int\exp(y\cdot\eta - h(y))\ dy$, then for any $\beta > 0$, we get the variance-altered density,
\[
\tilde\bP_\beta\bs{y \cond \eta} = \frac1{Z(\eta,\beta)}\exp(\beta\cdot(y\cdot\eta - \psi(\eta) - h(y))),
\]
where $Z(\eta,\beta) = \int\exp(\beta\cdot(y\cdot\eta - \psi(\eta) - h(y)))\ dy$. This transformation is order and mode preserving since $x^\beta$ is an increasing function for any $\beta > 0$. This generalized likelihood distribution has variance~\cite{NelderWedderburn1972}
$\frac1\beta\nabla^2\psi(\eta)$, which tends to $0$ as $\beta \rightarrow \infty$. Table~\ref{tab:variance} lists a few popular distributions, their variance altered versions, and asymptotic versions as $\beta \rightarrow \infty$.

We note that \citep{JiangKJ2012} also study  variance altering transformations for learning hidden Markov models, topic models, etc.. However, their transformations are unsuitable for use in \gem for a few reasons:\\
\textbf{1.} \gem's transformed distributions are always available in closed form whereas those of \citep{JiangKJ2012} are not necessarily available in closed form.\\
\textbf{2.} \gem's transformations are \emph{order}-preserving while \citep{JiangKJ2012} offer \emph{mean}-preserving that are not assured to be order-preserving.

\textbf{The Algorithm:} As presented in Algorithm~\ref{algo:gem}, \gem repeatedly constructs weighted MLEs $\tilde Q_\beta(\vw \cond \hvw^t)$ that take $\beta$-\textit{variance altered} weights $s_i=\tilde\bP_{\beta}[y_i\cond\ip{\vw}{\vx^i}]$ for all $i \in [n]$ and solves them to get new model estimates.

\begin{algorithm}[t]
	\caption{\small \gem: Sequential Variance-Altered MLE}
	\label{algo:gem}
	\begin{algorithmic}[1]
	{\small
		\REQUIRE Data $\bc{(\vx^i,y_i)}_{i=1}^n$, initial model $\hvw^1$, initial scale $\beta_1$, scale increment $\step > 1$, likelihood dist. $\P{\cdot\cond\cdot}$
		\FOR{$t = 1, 2, \ldots, T-1$}
			\STATE $s^t_i \leftarrow \tilde\bP_{\beta_t}[y_i\cond\ip{\hvw^t}{\vx^i}]$ \COMMENT{ $\beta_t$-var altered $\P{\cdot\cond\cdot}$}%
			\STATE $\tilde Q_{\beta_t}(\vw \cond \hvw^t) \deff -\sum_{i=1}^n s^t_i \cdot \log \P{y_i\cond\ip{\vw}{\vx^i}}$ 
			\STATE $\hvw^{t+1} = \arg\min_\vw\ \tilde Q_{\beta_t}(\vw \cond \hvw^t)$
			\STATE $\beta_{t+1} \leftarrow \step \cdot \beta_t$ \COMMENT{ Variance of $\tilde\bP_\beta[\cdot\cond\cdot] \downarrow$ as $\beta \uparrow$}
		\ENDFOR
		\STATE \textbf{return} {$\hvw^T$}
	}
	\end{algorithmic}
\end{algorithm}

We take a pause to assert that whereas the approach in \cite{MukhotyGJK2019}, although similar at first to Eq~\eqref{eq:wmle}, applies only to least-squares regression as it relies on notions of residuals missing from other GLMs. In contrast, \gem works for all GLMs e.g. least-squares/logistic/gamma regression and offers stronger theoretical guarantees.

Theorem~\ref{thm:gem-main} shows that \gem enjoys a linear rate of convergence. However, we first define notions of \emph{Local Weighted Strong Convexity and Lipschitz Continuity}. Let $\cB_2(\vv, r) := \bc{\vw: \norm{\vw - \vv}_2 \leq r}$ denote the $L_2$ ball of radius $r$ centered at the vector $\vv \in \bR^d$.

\begin{definition}[LWSC/LWLC]
\label{defn:lwsc-lwss}
Given data $\bc{(\vx^i,y_i)}_{i=1}^n$ and $\beta > 0$ an exponential family distribution $\P{\cdot\cond\cdot}$ is said to satisfy $\lambda_\beta$-Local Weighted Strongly Convexity and $\Lambda_\beta$-Local Weighted Lipschitz Continuity if for any \emph{true} model $\vwo$ and any $\vu, \vv \in \cB_2\br{\vwo, \sqrt\frac1\beta}$ the following hold,
\begin{enumerate}
    \item $\nabla^2 \tilde Q_\beta(\vv \cond \vu) \deff \left.\nabla^2 \tilde Q_\beta(\cdot \cond \vu)\right|_{\vv} \succeq \lambda_\beta\cdot I$
    \item $\norm{\nabla\tilde Q_\beta(\vwo\cond\vu)}_2 \deff \norm{\left.\nabla\tilde Q_\beta(\cdot\cond\vu)\right|_{\vwo}}_2 \leq \Lambda_\beta$
\end{enumerate}
\end{definition}
The above requires the $\tilde Q_\beta$-function to be strongly convex and Lipschitz continuous in a ball of radius $\frac1{\sqrt\beta}$ around the true model $\vwo$ i.e. as $\beta \uparrow$, the neighborhood in which these properties are desired also shrinks. We will show that likelihood functions corresponding to GLMs e.g., least squares and gamma regression satisfy these properties for appropriate ranges of $\beta$, even in the presence of corrupted samples.

\begin{theorem}[\gem convergence]
\label{thm:gem-main}
If the data and likelihood distribution satisfy the LWSC/LWLC properties for all $\beta \in (0,\beta_{\max}]$ and if \gem is initialized at $\hvw^1$ and scale $\beta_1 > 0$ s.t. $\beta_1\cdot\norm{\hvw^1 - \vwo}_2^2 \leq 1$, then for any $\epsilon > 1/{\beta_{\max}}$, for small-enough scale increment $\step > 1$, \gem ensures $\norm{\hvw^T - \vwo}_2^2 \leq \epsilon$ after $T = \bigO{\log\frac1\epsilon}$ iterations.
\end{theorem}

It is useful to take a moment to analyze this result. Note that if the LWSC/LWLC properties hold for larger values of $\beta$, \gem is able to offer smaller model recovery errors. Lets take least-squares regression with hybrid noise (see \S\ref{sec:me-rr-lr}) as an example. The proofs will show that LWSC/LWLC properties are assured for $\beta$ as large as $\beta_{\max} = \softO{\min\bc{\frac1{\alpha^{2/3}},\sqrt{\frac nd}}}$ (see \S\ref{sec:me-rr-lr}). Thus, with proper initialization of $\hvw^1, \step$ and $\beta_1$ (discussed below), \gem ensures $\norm{\hvw^T - \vwo}^2_2 \leq \softO{\max\bc{\alpha^{2/3},\sqrt{\frac dn}}}$ within $T = \bigO{\ln(n)}$ steps. This proof will hold so long as \gem is offered at least $n = \Omega(d\log d)$ training samples.

\begin{figure*}[t]%
	\centering
	\begin{subfigure}[b]{0.245\textwidth}
		\includegraphics[width=\textwidth]{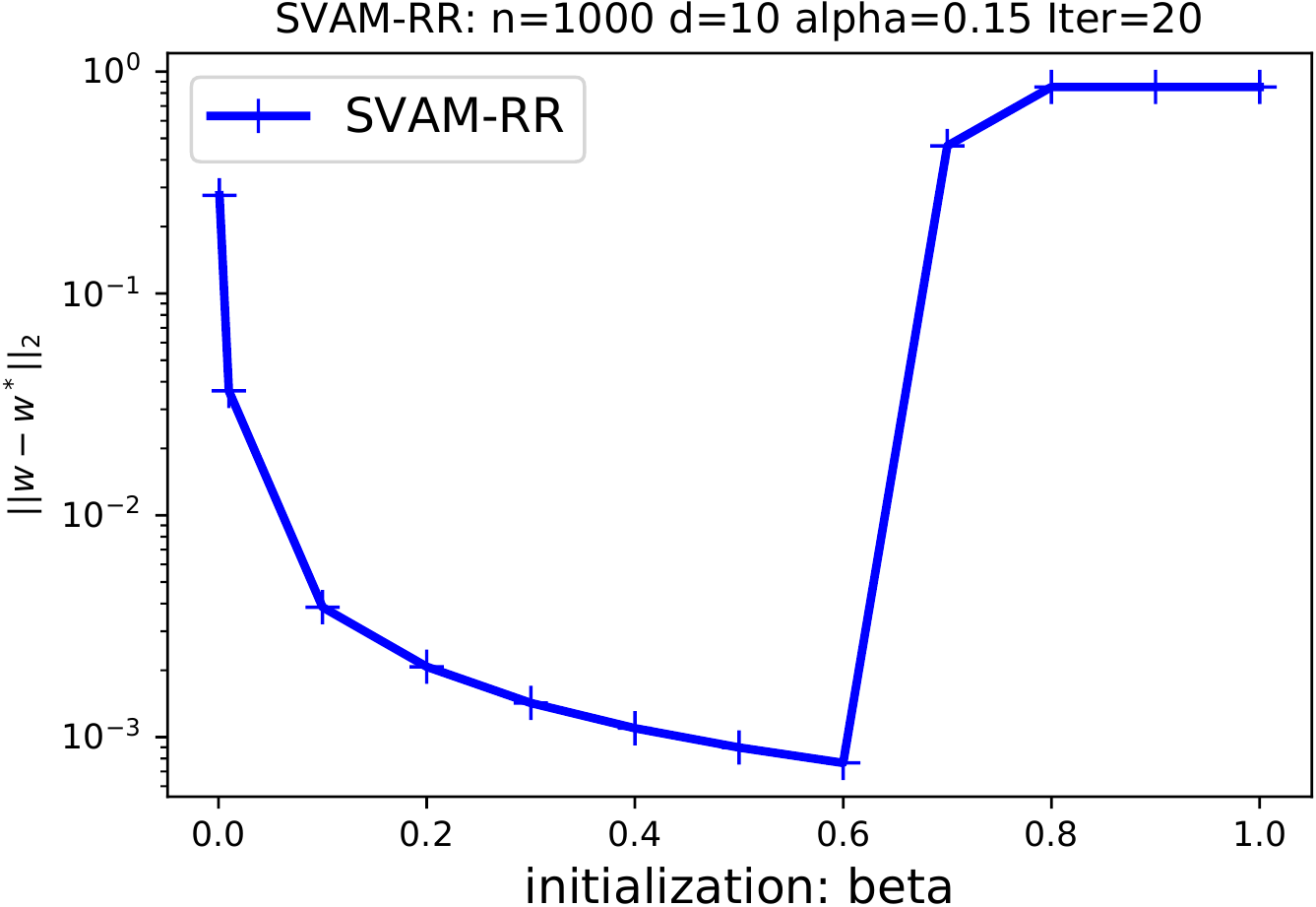}%
		\caption{\small $\beta_1$ Sensitivity}
	\end{subfigure}
	\begin{subfigure}[b]{0.245\textwidth}
		\includegraphics[width=\textwidth]{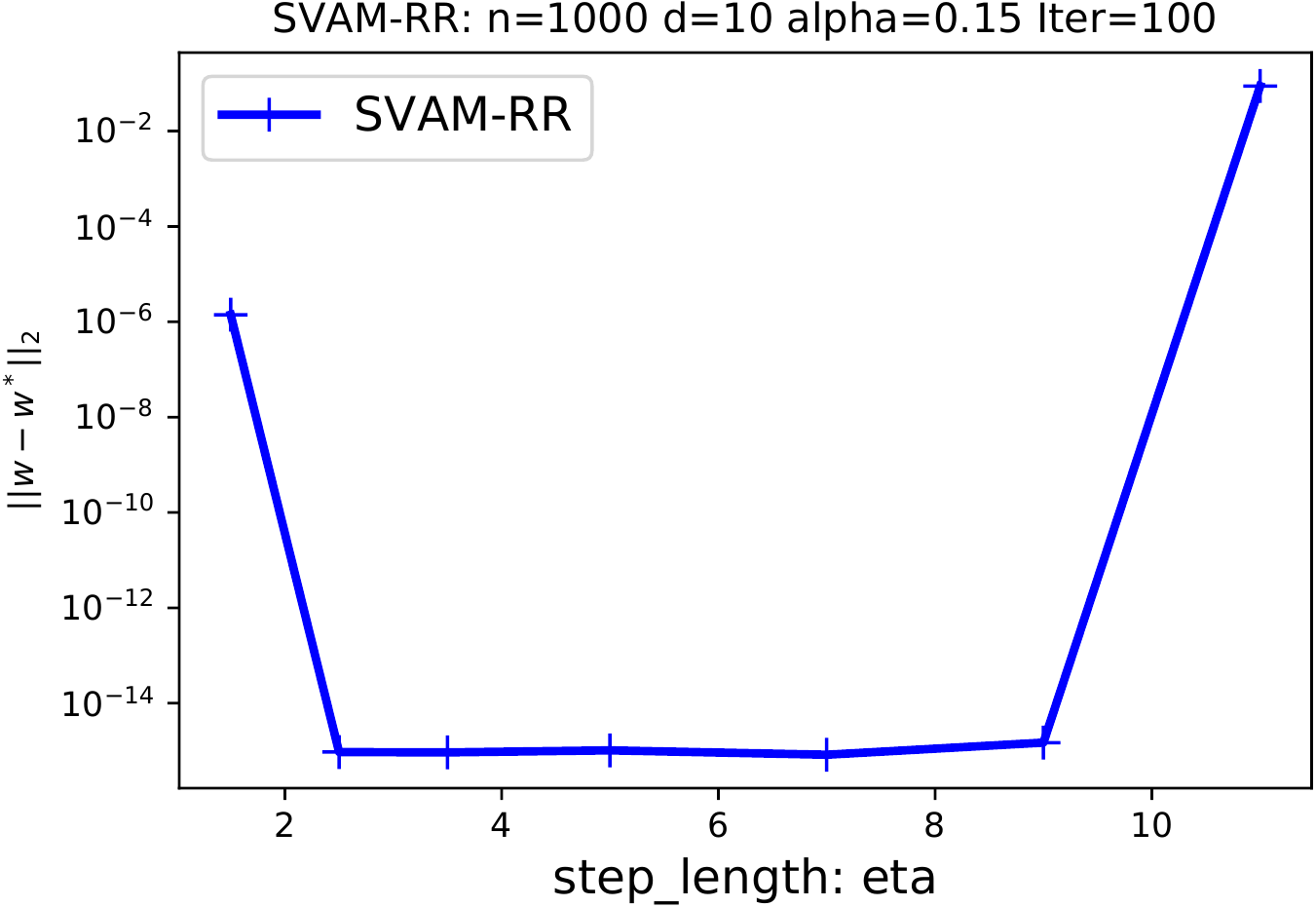}%
		\caption{\small $\step$ Sensitivity}
	\end{subfigure}
	\begin{subfigure}[b]{0.245\textwidth}
		\includegraphics[width=\textwidth]{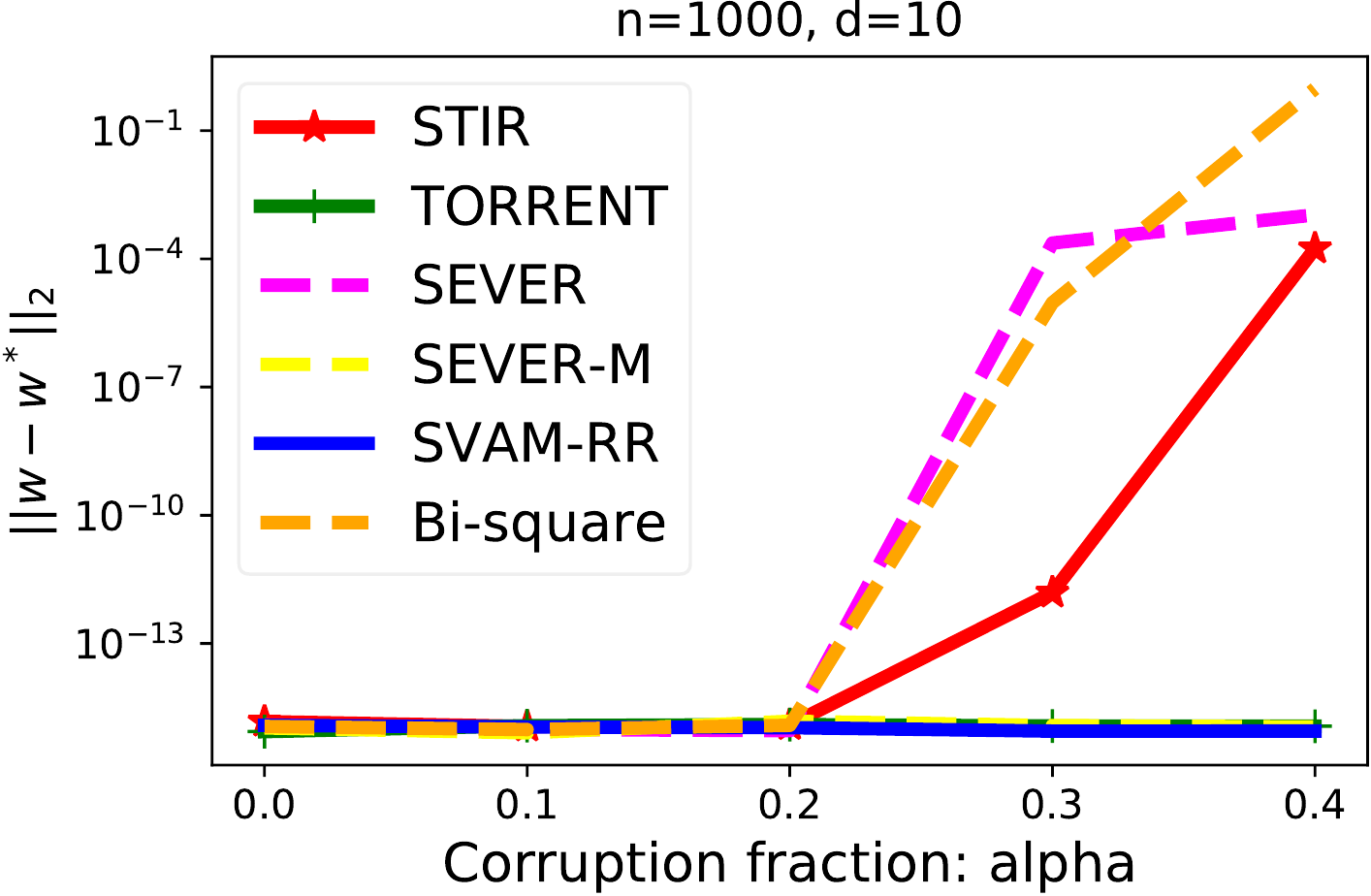}%
		\caption{\small $\alpha$ Sensitivity}
	\end{subfigure}
	\begin{subfigure}[b]{0.245\textwidth}
		\includegraphics[width=\textwidth]{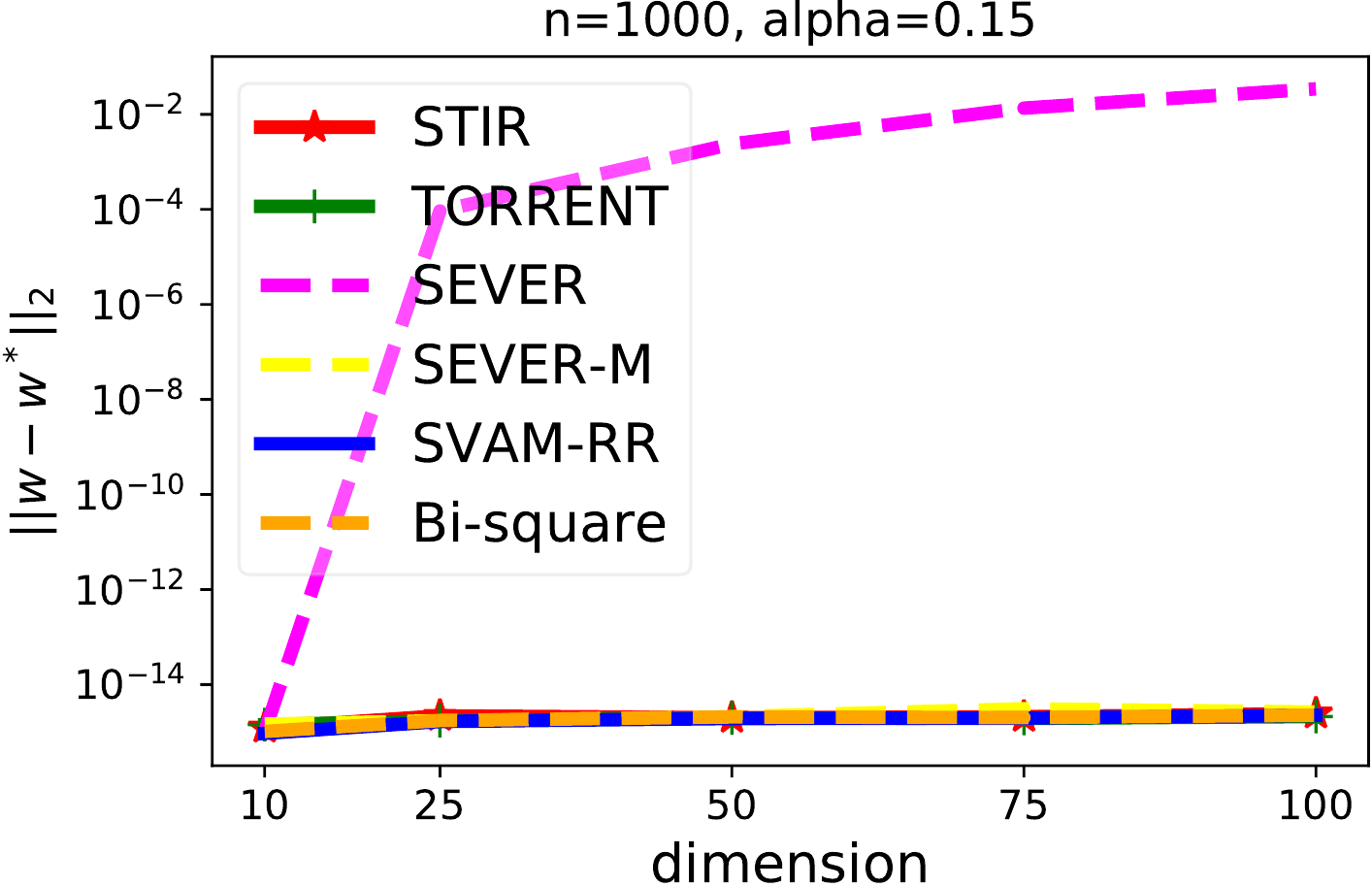}%
		\caption{\small $d$ Sensitivity}
	\end{subfigure}
	\caption{\small \gem offers stable convergence and recovery superior to competitor algorithms for a wide range of hyperparameters $\beta_1, \step$, corruption rates $\alpha = k/n$ and feature dimensionality $d$.}%
	\label{fig:hyperparam-sensitivity}%
\end{figure*}

\textbf{Initialization}: \gem needs to be invoked with $\hvw^1,\beta_1$ that satisfy the requirements of Thm~\ref{thm:gem-main} and small enough $\step$. If we initialize at the origin i.e. $\hvw^1 = \vzero$, then Theorem~\ref{thm:gem-main}'s requirement translates to $\beta_1 \leq \frac1{\norm{\hvw^1 - \vwo}_2^2}$ i.e. we need only find a small enough $\beta_1$. Thus, \gem needs to tune two scalars $\step,\beta_1$ to take small enough values which it does as described below. In practice, \gem offered stable performance for a wide range of $\beta_1, \step$ (see Fig~\ref{fig:hyperparam-sensitivity}).

\textbf{Hyperparameter Tuning}: \gem's two hyperparameters $\beta_1, \step$ were tuned using a held-out validation set. As the validation data could also contain corruptions, validation error was calculated by rejecting the top $\alpha$ fraction of validation points with the highest prediction error. The true value of $\alpha$ was provided to competitor algorithms as a handicap but not to \gem. Thus, $\alpha$ itself was treated as a (third) hyperparameter for \gem.

\section{Robust GLM Applications with \gem}
\label{sec:me-rr-lr}

This section adapts \gem to robust least-squares/gamma/logistic regression and robust mean estimation and establishes breakdown points and LWSC/LWLC guarantees for their respective $\tilde Q_\beta$ functions (see Defn~\ref{defn:lwsc-lwss}). We refer the reader to \S\ref{sec:intro} for definitions of \textbf{partially/fully adaptive} adversaries.

\textbf{Robust Least Squares Regression.} We have $n$ data points $(\vx^i, y_i)$, $\vx^i \in \bR^d$ sampled from a subGaussian distribution $\cD$ over $\bR^d$. We consider the \textbf{hybrid corruption} setting where on the $G = (1-\alpha)\cdot n$ ``good'' data points, we get labels $y_i = \ip{\vwo}{\vx^i} + \epsilon_i$ with Gaussian noise $\epsilon_i \sim \cN(0,\frac1\betao)$ with variance $\frac1\betao$ added. On the remaining $B = \alpha\cdot n$ ``bad'' points, we get adversarially corrupted labels $\tilde y_i = \ip{\vwo}{\vx^i} + b_i$ where $b_i \in \bR$ is chosen by the adversary. Note that $b_i$ can be unbounded. We also consider the \textbf{pure corruption} setting where clean points receive no Gaussian noise (note that this corresponds to $\betao = \infty$). \gemrr (Alg.~\ref{algo:gem-rr}) adapts \gem to the task of robust regression. 

\begin{theorem}[Partially Adaptive Adversary]
\label{thm:rr-main}
For \emph{hybrid corruptions} by a partially adaptive adversary with corruption rate $\alpha \leq 0.18$, there exist $\step > 1$ s.t. with probability at least $1 - \exp(-\Om d)$, LWSC/LWLC properties are satisfied for the $\tilde Q_\beta$ function for $\beta$ values as large as $\beta_{\max} = \bigO{\betao \min\bc{\frac1{\alpha^{2/3}}, \sqrt\frac n{d\log(n)}}}$. If initialized with $\hvw^1, \beta^1$ s.t. $\beta_1\cdot\norm{\hvw^1 - \vwo}_2^2 \leq 1$, \gemrr assures $\norm{\hvw^T - \vwo}_2^2 \leq \bigO{\frac1\betao\max\bc{\alpha^{2/3}, \sqrt\frac{d\log(n)}n}}$ within $T \leq \bigO{\log\frac n{\beta^1}}$ iterations. For \emph{pure corruptions} by a partially adaptive adversary, we have $\beta_{\max} = \infty$ and thus, for any $\epsilon > 0$, \gemrr assures $\norm{\hvw^T - \vwo}_2^2 \leq \epsilon$ within $T \leq \bigO{\log\frac1{\epsilon\beta^1}}$ iterations.
\end{theorem}

Note that in the pure corruption setting, \gem assures exact recovery of $\vwo$ simply by running the algorithm long enough. This is not a contradiction since in this case, the LWSC/LWSS properties can be shown to hold for all values of $\beta < \infty$ since we effectively have $\betao = \infty$ in this case. Thm~\ref{thm:rr-main} holds against a partially adaptive adversary but can be extended to a fully adaptive adversary as well but at the cost of a worse breakdown point (see Thm~\ref{thm:rr-fully-adaptive} below). Note that \gem continues to assure exact recovery of $\vwo$.

\begin{theorem}[Fully Adaptive Adversary]
\label{thm:rr-fully-adaptive}
For pure corruptions by a fully adaptive adversary with corruption rate $\alpha \leq 0.0036$, LWSC/LWLC are satisfied for all $\beta \in (0,\infty)$ i.e. $\beta_{\max} = \infty$ and for any $\epsilon > 0$, \gemrr assures $\norm{\hvw^T - \vwo}_2^2 \leq \epsilon$ within $T \leq \bigO{\log\frac 1{\epsilon\beta^1}}$ iterations if initialized as described in the statement of Theorem~\ref{thm:rr-main}.
\end{theorem}

\textbf{Establishing LWSC/LWLC:} In the appendices, Lemmata~\ref{lem:rr-lwsc} and \ref{lem:rr-LWLC} establish LWSC/LWLC properties for robust least squares regression while Theorems~\ref{repthm:rr-main} and \ref{thm:rr-fully-adaptive-restated} establish the breakdown points and existence of suitable increments $\step > 1$. Handling a fully adaptive adversary requires making mild modifications to the notions of LWSC/LWLC, details of which are present in Appendix~\ref{app:rr-adaptive}.

\textbf{Model Recovery and Breakdown Point:} For pure corruption, \gemrr offers exact model recovery against partially and fully adaptive adversaries as it assures $\norm{\hvw^T - \vwo}_2^2 \leq \epsilon$ for any $\epsilon > 0$ if \gemrr is executed long enough. For hybrid corruption where even ``clean'' points receive Gaussian noise with variance $\frac1\betao$, \gemrr assures $\norm{\hvw^T - \vwo}_2^2 \leq \bigO{\frac1\betao\sqrt\frac{d\log(n)}n}$ as $\alpha \rightarrow 0$ i.e. $\norm{\hvw^T - \vwo}_2^2 \rightarrow 0$ as $n \rightarrow \infty$ assuring consistent recovery. This significantly improves previous results by \cite{BhatiaJK2015,MukhotyGJK2019} which offer $\bigO{\frac1\betao}$ error even if $\alpha \rightarrow 0$ and $n \rightarrow \infty$. Note that \gemrr has a superior breakdown point (allowing upto 18\% corruption rate) against an oblivious adversary. The breakdown point deteriorates as expected (still allowing upto 0.36\% corruption rate) against a fully adaptive adversary. We now present analyses for other GLM problems.

\begin{figure}
\centering
\begin{adjustbox}{max width=0.9\linewidth}
\begin{minipage}{\linewidth}
\begin{algorithm}[H]
	\caption{\small{\gemrr: Robust Least Squares Regression}}
	\label{algo:gem-rr}
	\begin{algorithmic}[1]
		{\small
			\REQUIRE Data $\bcn{(\vx^i, y_i)}_{i = 1}^n$, initial scale $\beta_1$, initial model $\hvw^1$, $\step$
			\ENSURE A model estimate $\hvw \approx \vwo$
			\FOR{$t = 1, 2, \ldots, T-1$}
				\STATE $s_i \leftarrow \exp\br{-\frac{\beta_t}2(y_i - \ip{\vx^i}{\hvw^t})^2}$
				\STATE $S \leftarrow \diag(s_1,\ldots,s_n)$
				\STATE $\hvw^{t+1} \leftarrow (XSX^\top)^{-1}(XS\vy)$
				\STATE $\beta_{t+1} \leftarrow \step \cdot \beta_t$
			\ENDFOR
			\STATE \textbf{return} {$\hvw^T$}
		}
	\end{algorithmic}
\end{algorithm}
\vspace*{-5ex}
\begin{algorithm}[H]
	\caption{\small{\gemme: Robust Mean Estimation}}
	\label{algo:gem-me}
	\begin{algorithmic}[1]
		{\small
			\REQUIRE Data $\bcn{\vx^i}_{i = 1}^n$, initial scale $\beta_1$, initial model $\hvmu^1$, $\step$
			\ENSURE A mean estimate $\hvmu \approx \vmuo$
			\FOR{$t = 1, 2, \ldots, T-1$}
				\STATE $s_i \leftarrow \exp\br{-\frac{\beta_t}2\norm{\vx^i - \hvmu^t}_2^2}$
				\STATE $\hvmu^{t+1} \leftarrow \br{\sum_{i=1}^ns_i}^{-1}\br{\sum_{i=1}^ns_i\vx^i}$
				\STATE $\beta_{t+1} \leftarrow \step \cdot \beta_t$
			\ENDFOR
			\STATE \textbf{return} {$\hvmu^T$}
		}
	\end{algorithmic}
\end{algorithm}
\vspace*{-5ex}
\begin{algorithm}[H]
	\caption{\small{\gemgam: Robust Gamma Regression}}
	\label{algo:gem-gam}
	\begin{algorithmic}[1]
		{\small
			\REQUIRE Data $\bcn{(\vx^i, y_i)}_{i = 1}^n$, initial scale $\beta_1$, initial model $\hvw^1$, $\step$
			\ENSURE A model estimate $\hvw \approx \vwo$
			\FOR{$t = 1, 2, \ldots, T-1$}
				\STATE $s_i \leftarrow \cG(y_i \cond \tilde\eta_{\beta_t}, \tilde\phi_{\beta_t})$ \COMMENT{see Table~\ref{tab:variance}}
				\STATE {$\hvw^{t+1} \leftarrow \arg\min \sum_{i=1}^ns_i\cdot\ell(\vw,\vx^i,y_i)$} where\\
				$\ell(\vw,\vx,y) = (1-\phi)^{-1}y\exp(\ip\vw\vx) - \ip\vw\vx$%
				\STATE $\beta_{t+1} \leftarrow \step \cdot \beta_t$
			\ENDFOR
			\STATE \textbf{return} {$\hvw^T$}
		}
	\end{algorithmic}
\end{algorithm}
\vspace*{-5ex}
\begin{algorithm}[H]
	\caption{\small{\gemlr: Robust Classification}}
	\label{algo:gem-lr}
	\begin{algorithmic}[1]
		{\small
			\REQUIRE Data $\bcn{(\vx^i, y_i)}_{i = 1}^n$, initial scale $\beta_1$, initial model $\hvw^1$, $\step$
			\ENSURE A model estimate $\hvw \approx \vwo$
			\FOR{$t = 1, 2, \ldots, T-1$}
				\STATE $s_i \leftarrow (1 + \exp(-\beta_t y_i\ip{\vx^i}{\hvw^t}))^{-1}$
				\STATE $\hvw^{t+1} \leftarrow \arg\min \sum_{i=1}^ns_i\cdot\ell(\vw,\vx^i,y_i)$ where\\
				{$\ell(\vw,\vx,y) = \log(1 + \exp(- y\ip{\vx}{\vw}))$}%
				\STATE $\beta_{t+1} \leftarrow \step \cdot \beta_t$
			\ENDFOR
			\STATE \textbf{return} {$\hvw^T$}
		}
	\end{algorithmic}
\end{algorithm}
\end{minipage}
\end{adjustbox}
\end{figure}

\textbf{Robust Gamma Regression.} The data generation and corruption model for gamma regression are slightly different given that the gamma distribution has support only over positive reals. First, the canonical parameter is calculated as $\eta_i = \exp(\ip\vwo{\vx^i})$ using which a clean label $y^i$ is generated. To simplify the analysis, we assume that $\norm\vwo_2 = 1$, $\phi = 0.5$, $\vx^i \sim \cN(\vzero, I)$. For the $G = (1-\alpha)\cdot n$ ``good'' points, labels are generated as $y_i = \exp(\ip\vwo{\vx^i})(1-\phi)$ i.e. the \emph{no-noise} model. For the remaining $B = \alpha\cdot n$ ``bad'' points, the label is corrupted as $\tilde y^i = y^i\cdot b_i$ where $b_i > 0$ is a positive real number (but otherwise arbitrary and unbounded). A multiplicative corruption makes more sense since the final label must be positive. \gemgam (Algorithm~\ref{algo:gem-gam}) adapts \gem to robust gamma regression. Due to the alternate canonical parameter used in gamma regression, the initialization requirement also needs to be modified to $\beta_1\cdot\br{\exp\br{\norm{\hvw^1 - \vwo}_2} - 1}^2 \leq 1$. However, the hyperparameter tuning strategy discussed in \S\ref{sec:gem} continues to apply.

\begin{theorem}
\label{thm:gam-main}
For data corrupted by a partially adaptive adversary with $\alpha \leq \frac{0.002}{\sqrt{d}}$, there exist $\step > 1$ s.t. with probability at least $1 - \exp(-\Om d)$, LWSC/LWLC conditions are satisfied for the $\tilde Q_\beta$ function for $\beta$ values as large as $\beta_{\max} = \bigO{1/\br{\exp\br{\bigO{\alpha \sqrt{d}}} - 1}}$. If initialized at $\hvw^1,\beta_1$ s.t. $\beta_1\cdot\br{\exp\br{\norm{\hvw^1 - \vwo}_2} - 1}^2 \leq 1$ and $\beta \geq 1$, \gemgam assures $\norm{\hvw^T - \vwo}_2 \leq \epsilon$ for any $\epsilon \geq \bigO{\alpha \sqrt{d}}$ within $T \leq \bigO{\log\frac1\epsilon}$ steps.
\end{theorem}

\textbf{Model recovery, Consistency, Breakdown pt.} It is notable that prior results in literature do not offer any breakdown point results for gamma regression. We find that Thm~\ref{thm:gam-main} requires $\beta_1\cdot\br{\exp\br{\norm{\hvw^1 - \vwo}_2} - 1}^2 \leq 1$ and $\beta \geq 1$ which imply $\norm{\hvw^1 - \vwo}_2 \leq \ln 2$. This is in contrast to Thms~\ref{thm:rr-main} and \ref{thm:rr-fully-adaptive} that allow any initial $\hvw^1$ so long as $\beta_1, \step$ are sufficiently small. \gemgam guarantees convergence to a region of radius $\bigO{\alpha\sqrt d}$ around $\vwo$ whereas Thms~\ref{thm:rr-main} and \ref{thm:rr-fully-adaptive} assure exact recovery. However, these do not seem to be artifacts of the proof technique. In experiments, \gemgam did not offer vanishingly small recovery errors and did indeed struggle if initialized with $\beta_1 \ll 1$. It may be the case that there exist lower bounds preventing exact recovery for gamma regression similar to mean estimation.

\textbf{Robust Mean Estimation.} We have $n$ data points of which the set $G$ of $(1-\alpha)\cdot n$ ''good'' points are generated from a $d$-dimensional spherical Gaussian $\vx^i \sim \cN(\vmu, \Sigma)$ i.e. $\vx^i = \vmu + \vepsilon^i$ where $\vepsilon^i \sim \cN(\vzero, \Sigma)$ and $\Sigma = \frac1\betao\cdot I$ for some $\betao > 0$. The rest are the set $B$ of $\alpha\cdot n$ ''bad'' points that are corrupted by an adversary i.e. $\tilde\vx^i = \vmuo + \vb^i$ where $\vb^i \in \bR^d$ can be unbounded. \gemme (Algorithm~\ref{algo:gem-me}) adapts \gem to the robust mean estimation problem. For notational clarity we use, $\veta = \vmu$, in this problem.

\begin{theorem}
\label{thm:me-main}
For data corrupted by a partially adaptive adversary with corruption rate $\alpha \leq 0.26$, there exists $\step > 1$ s.t. with probability at least $1 - \exp(-\Om d)$, LWSC/LWLC conditions are satisfied for the $\tilde Q_\beta$ function for $\beta$ upto $\beta_{\max} = \bigO{\frac\betao d\min\bc{\log\frac1\alpha, \sqrt{nd}}}$. If initialized with $\hvmu^1, \beta^1$ s.t. $\beta_1\cdot\norm{\hvmu^1 - \vmuo}_2^2 \leq 1$, \gemme assures $\norm{\hvmu^T - \vmuo}_2^2 \leq \epsilon$ for any $\epsilon \geq \bigO{\trace^2(\Sigma)\cdot\max\bc{{\frac1{\ln(1/\alpha)}}, \frac1{\sqrt{nd}}}}$ within $T \leq \bigO{\log\frac n{\beta^1}}$ iterations.
\end{theorem}

\textbf{Model recovery, Consistency, Breakdown pt.} Note that for any constant $\alpha > 0$, the estimation error does not go to zero as $n \rightarrow \infty$. As mentioned in \S\ref{sec:related}, an error of $\Om{\alpha}$ is unavoidable no matter how large $n$ gets. Thus, the best hope we have is of the estimation error going to zero as $\alpha \rightarrow 0$ and $n \rightarrow \infty$. The error in Theorem~\ref{thm:me-main} does indeed go to zero in this setting. Also, note that the error depends only on the trace of the covariance matrix of the clean points, and thus for $\trace(\Sigma) = \bigO1$, the result offers an estimation error independent of dimension. \gemme offers a large breakdown point (allowing upto 26\% corruption rate).

\textbf{Establishing LWSC/LWLC for Gamma Regression and Mean Estimation}: In the appendices, Lemmata~\ref{lem:gam-lwsc}, \ref{lem:gam-LWLC}, Lemmata~\ref{lem:me-lwsc}, \ref{lem:me-LWLC} establish the LWSC/LWLC properties for the $\tilde Q_\beta$ function for gamma regression and mean estimation and Theorems~\ref{repthm:gam-main} and Theorem~\ref{repthm:me-main} establish the breakdown points and existence of increments $\step > 1$.

\textbf{Robust Classification.} In this case the labels are generated as $y_i = \sign(\ip{\vwo}{\vx^i})$ and the bad points in the set $B$ get their labels flipped $\tilde y_i = -\sign(\ip{\vwo}{\vx^i})$. \gemlr (Algorithm~\ref{algo:gem-lr}) adapts \gem to robust logistic regression.

\section{Experiments}
\label{sec:exps}

We used a 64-bit machine with Intel® Core™ i7-6500U CPU @ 2.50GHz, 4 cores, 16 GB RAM, Ubuntu 16.04 OS.

\textbf{Benchmarks.} \gem was benchmarked against several baselines \textbf{(a)} \textbf{VAM}: this is \gem executed with a fixed value of the scale $\beta$ by setting the scale increment to $\step = 1$ to investigate any benefits of varying the scale $\beta$, \textbf{(b)} \textbf{MLE}: likelihood maximization on all points (clean + corrupted) without any weights assigned to data points -- this checks for benefits of performing weighted MLE, \textbf{(c)} \textbf{Oracle}: an execution of the MLE only on the clean points -- this is the gold standard in robust learning and offers the best possible outcome. In addition, several \textbf{problem-specific competitors} were also considered. For robust regression, STIR \cite{MukhotyGJK2019}, TORRENT \cite{BhatiaJK2015}, SEVER \cite{DiakonikolasKKLSS2019}, RGD \cite{PrasadSBR2018}, and the classical robust M-estimator of Tukey's bisquare loss were included. Note that TORRENT already outperforms $L_1$ regularization methods while achieving better or competitive recovery errors (see \cite[Fig 2(b)]{BhatiaJK2015}). Since \gemrr was faster than TORRENT itself, $L_1$ regularized methods such as \cite{NguyenT2013,WrightM2010} were not considered. For robust mean estimation, popular baselines such as coordinate-wise median and geometric median were taken. For robust classification, the rank-pruning method RP-LR \cite{northcutt2017rankpruning} and the method from \cite{natarajan2013learning} were used.

\textbf{Experimental Setting and Reproducibility.} Due to lack of space, details of experimental setup, data generation, how adversaries were simulated etc are presented in Appendix~\ref{app:setup}. \gem also offered superior robustness than competitors against a wide range of ways to simulate adversarial corruption (see Appendix~\ref{app:exps} for details). Code for \gem is available at \code.

\begin{figure*}[t]%
	\centering
	\begin{subfigure}[b]{0.3\textwidth}
		\includegraphics[width=\textwidth]{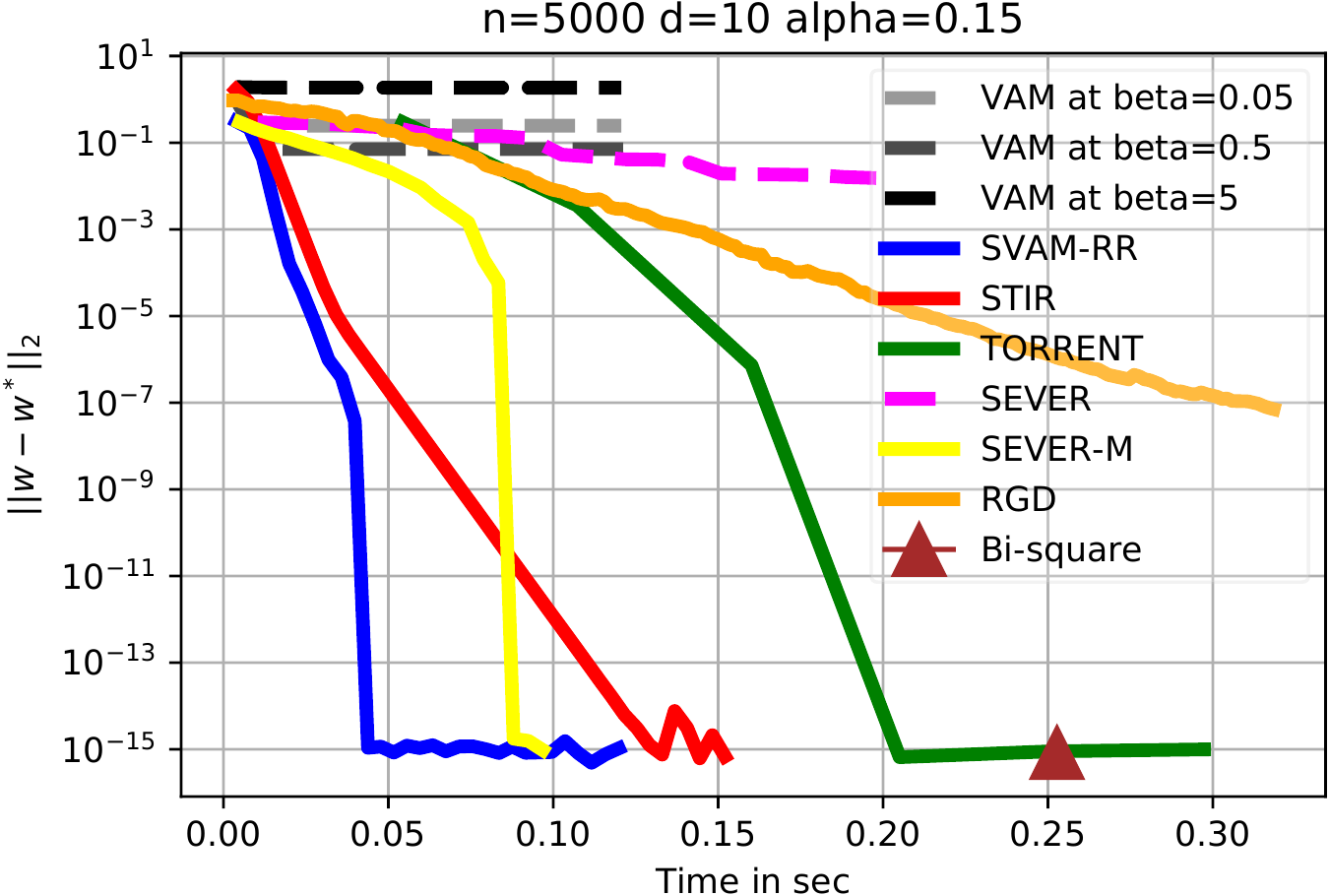}%
		\caption{\small Robust Regression}
	\end{subfigure}
	\begin{subfigure}[b]{0.3\textwidth}
		\includegraphics[width=\textwidth]{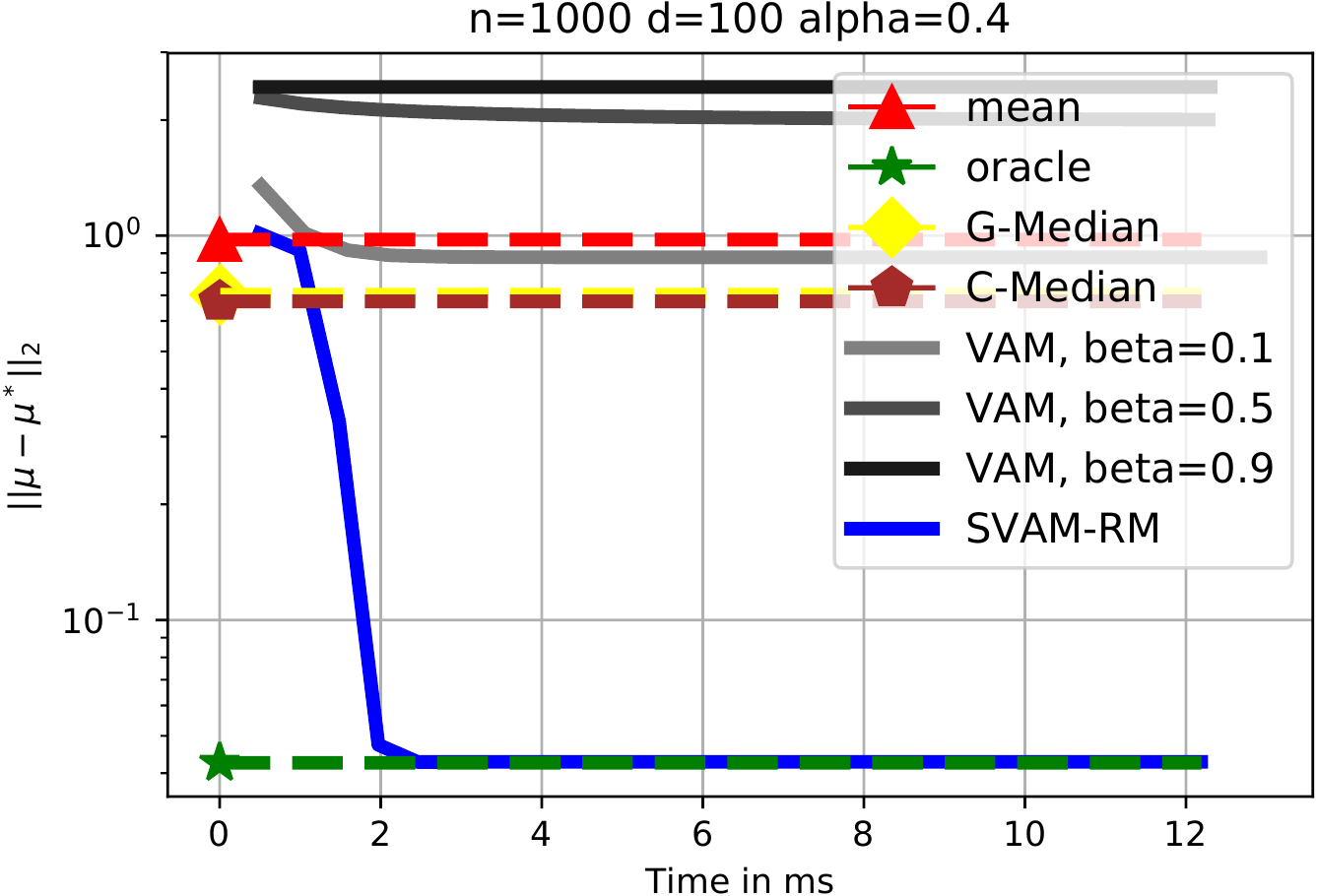}%
		\caption{\small Robust Mean Estimation}
	\end{subfigure}
	\begin{subfigure}[b]{0.3\textwidth}
		\includegraphics[width=\textwidth]{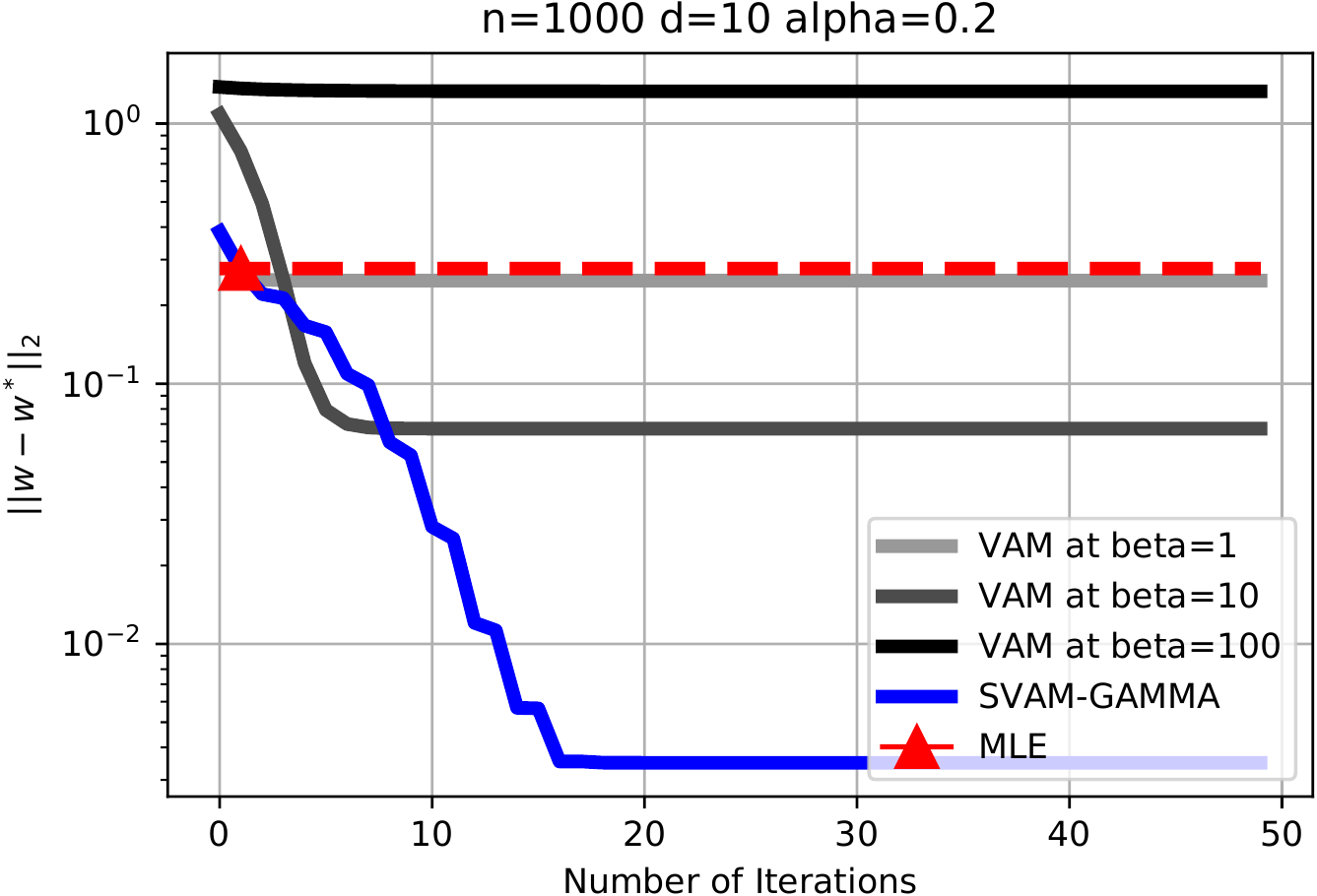}%
		\caption{\small Robust Gamma Regression}
	\end{subfigure}
	\begin{subfigure}[b]{0.3\textwidth}
		\includegraphics[width=\textwidth]{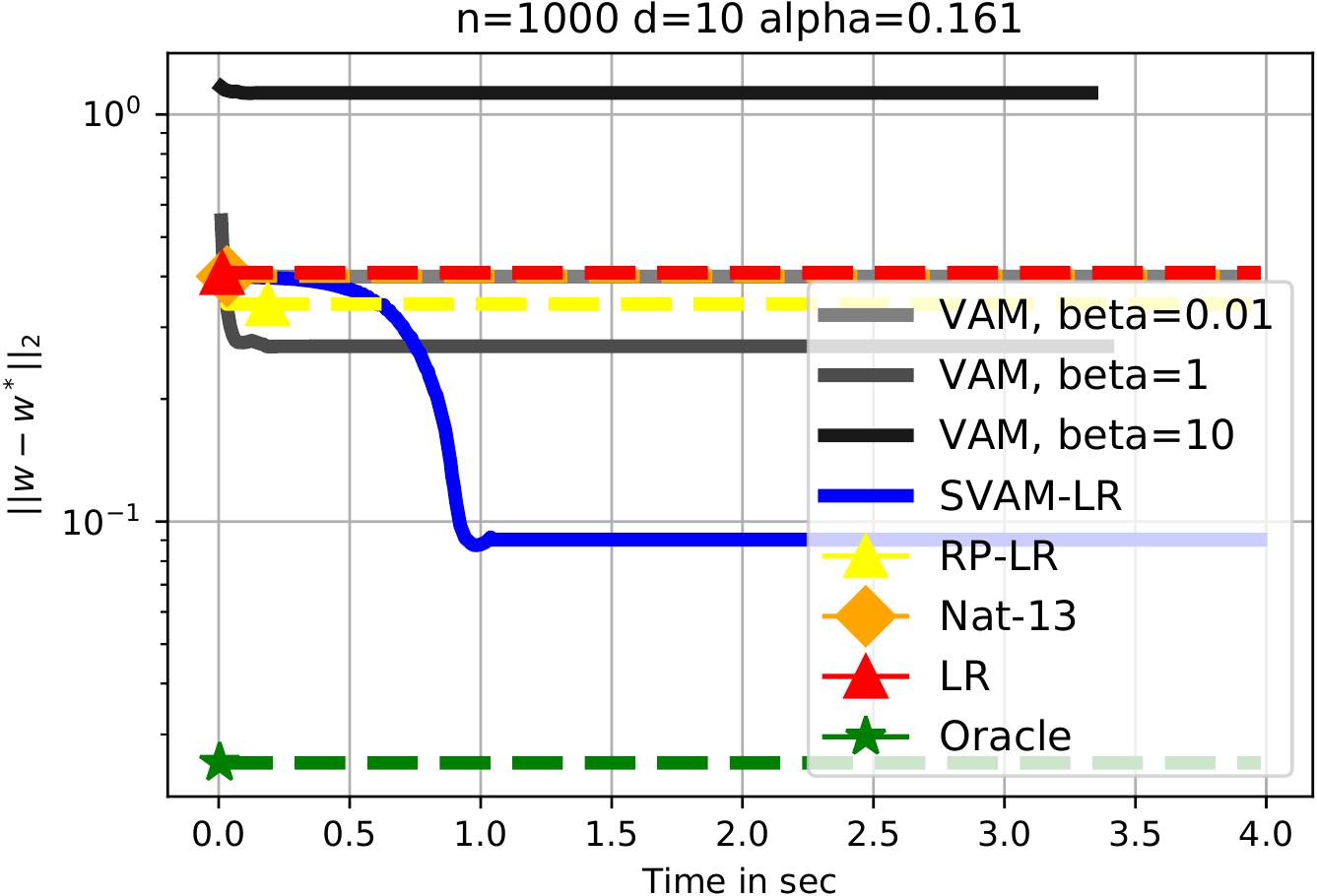}%
		\caption{\small Robust Classification}
	\end{subfigure}
	\begin{subfigure}[b]{0.3\textwidth}
		\includegraphics[width=\textwidth]{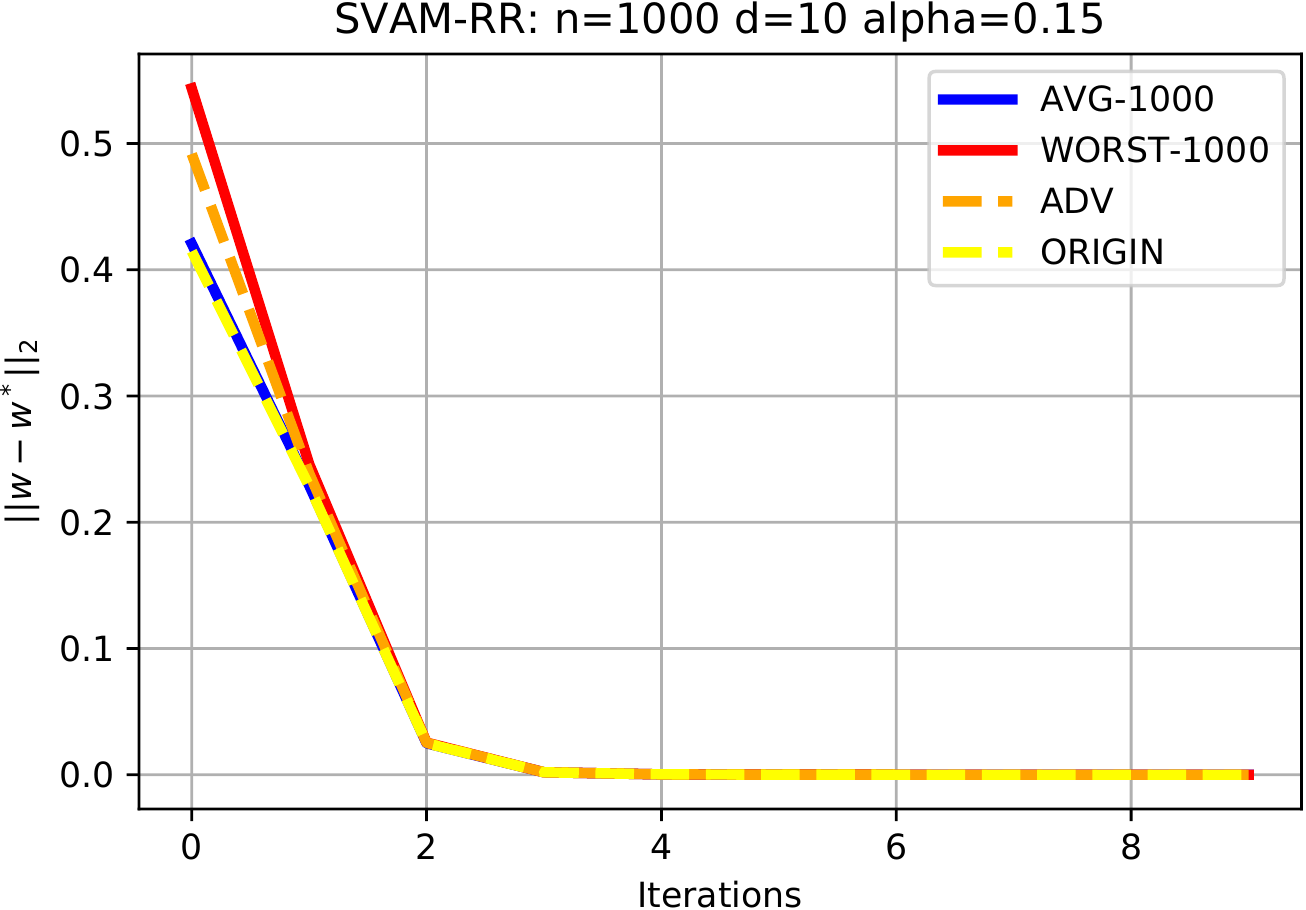}%
		\caption{\small Adversarial Initialization}
	\end{subfigure}
	\begin{subfigure}[b]{0.3\textwidth}
		\includegraphics[width=\textwidth]{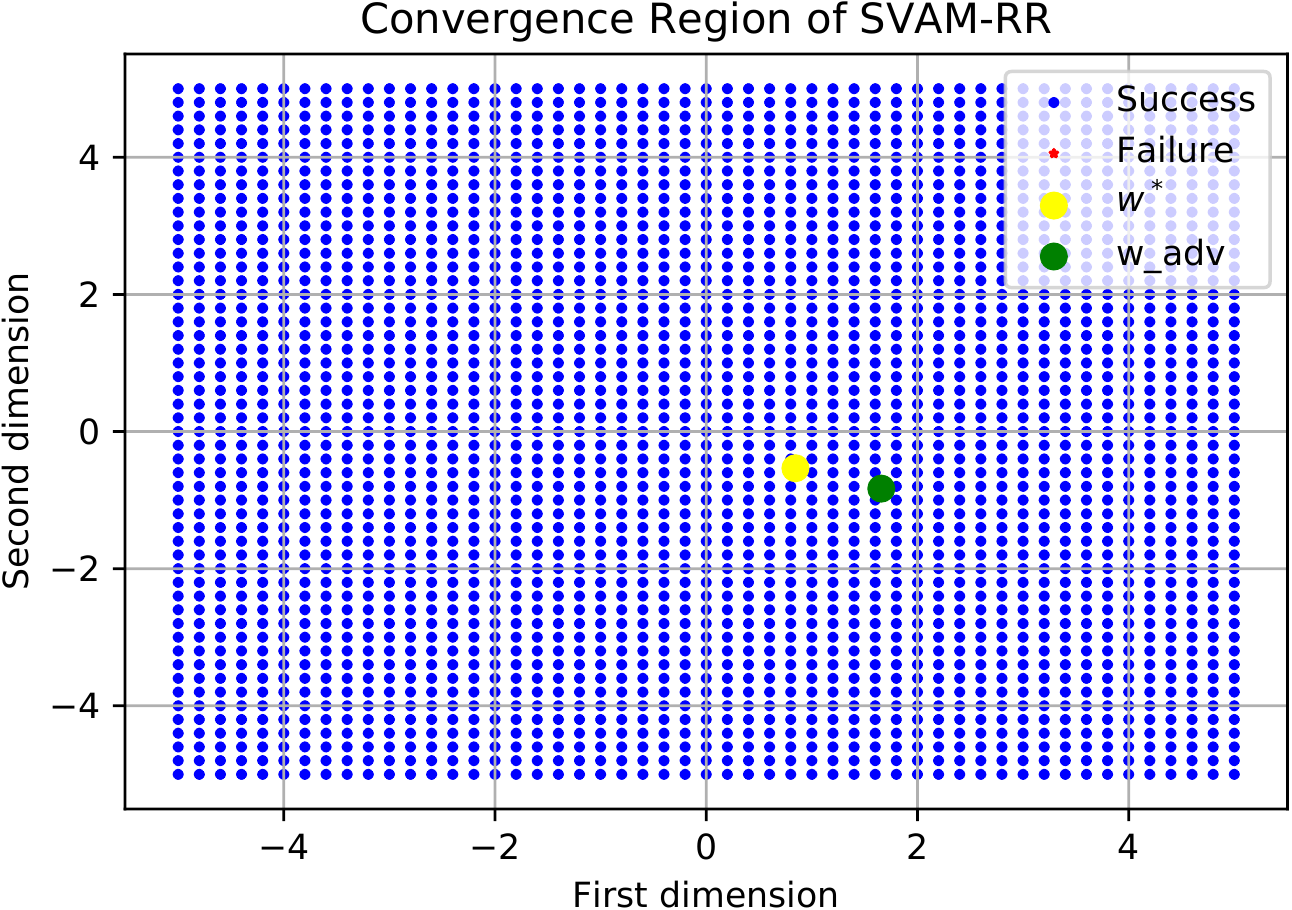}%
		\caption{\small Global Convergence}
	\end{subfigure}
	\caption{\small Figs~\ref{fig:conv}(a,b,c,d) compare \gem and various competitors on robust GLM problems. The number of data points $n$, dimensions $d$, and fraction of corruptions $\alpha = k/n$ are mentioned at the top of each figure. For algorithms for which iteration-wise performance was unavailable, their final performance level is plotted as a horizontal dashed line. A marker is placed on the line indicating the time it took for that algorithm to converge. The figures clarify that executing VAM with a single fixed value of $\beta$ cannot replace the gradual variations in $\beta_t$ done by \gem. Figs~\ref{fig:conv}(e,f) confirm that \gem offers convergence to $\vwo$ irrespective of the model with which it is initialized. In these figures, corruptions were introduced using an adversarial model $\tvw$ i.e. for corrupted points, the label was set to $\tilde y_i = \ip{\tvw}{\vx_i}$. \gem was then initialized at $\tvw$ itself to check if it gets misled by faulty initialization but was found to offer exact recovery regardless.}
	\label{fig:conv}%
\end{figure*}

\subsection{Experimental Observations}

\textbf{Robust Regression.} Fig~\ref{fig:conv}(a) shows that \gemrr, SEVER, RGD, STIR, and TORRENT are competitive and achieve oracle-level error. However, \gemrr can be twice as fast in terms of execution time. Since TORRENT itself outperforms $L_1$ regularization methods while achieving better or competitive recovery errors (see Fig 2(b) in \cite{BhatiaJK2015}), we do not compare against $L_1$ methods. \gemrr is several times faster than classical robust M-estimators such as Tukey's bisquare loss. Also, no single value of $\beta$ can offer the performance of \gem, as is indicated by the poor performance of VAM. Fig~\ref{fig:vam-beta} in the appendix shows that this is true even if very large or very small values of $\beta$ are used with VAM. We note that SEVER chooses a threshold in each iteration to eliminate specific points as corrupted. This threshold is chosen randomly (possibly for ease of proof) but causes SEVER to offer sluggish convergence. Thus, we also report the performance of a modification SEVER-M that was given an unfair advantage by revealing to it the actual number of corrupted points (\gem was not given this information). This sped-up SEVER but \gem continued to outperform SEVER-M. Fig~\ref{fig:avg-conv} in the appendix reports repeated runs of the experiment where \gem continues to lead.

\textbf{Robust Logistic and Gamma Regression.} Fig~\ref{fig:conv}(c,d) report results of \gem on robust gamma and logistic regression problems. The figures show that executing VAM with a fixed value of $\beta$ cannot replace the gradual variations in $\beta_t$ done by \gem. Additionally, for robust classification, \gemlr achieves error, an order of magnitude smaller than all competitors except the oracle. \gem also outperforms the RP-LR \cite{northcutt2017rankpruning} and \cite{natarajan2013learning} algorithms that were specifically designed for robust classification. A horizontal dashed line is used to indicate the final performance of algorithms for which iteration-wise performance was unavailable.

\textbf{Robust Mean Estimation.} Fig~\ref{fig:conv}(b) reports results on robust mean estimation problems. \gem outperforms VAM with any fixed value of $\beta$ as well as the naive sample mean (the MLE in this case). Popular approaches coordinate-wise median and geometric median were fast but offered poor results. \gem on the other hand achieved oracle error-level error by assigning proper scores to all data points.

\textbf{Sensitivity to Hyperparameter Tuning.} In Figs~\ref{fig:hyperparam-sensitivity}(a,b), \gemrr was offered hyperparameters in a wide range of values to study how it responded when provided mis-specified hyperparameters. \gem offered stable convergence for a wide range of $\beta_1, \step$ indicating that it is resilient to minor mis-specifications in hyperparameters.

\textbf{Sensitivity to Dimension and Corruption.} Figs~\ref{fig:hyperparam-sensitivity}(c,d) compare the error offered by various algorithms in recovering $\vwo$ for robust least-squares regression when the fraction of corrupted points $\alpha$ and feature dimension $d$ were varied. All values are averaged over $20$ experiments with each experiment using $1000$ data points. $\alpha$ was varied in the range $[0,0.4]$ and $d$ in the range $[10,100]$ with fixed hyper-parameters. STIR and Bi-square are sensitive to corruption while SEVER is sensitive to both corruption and dimension. RGD is not visible in the figures as its error exceeded the figure boundaries. Experiments for Fig~\ref{fig:hyperparam-sensitivity}(c) fixed $d = 10$ and vary $\alpha$ while Fig~\ref{fig:hyperparam-sensitivity}(d) fixed $\alpha = 0.15$ and vary $d$. Figs~\ref{fig:hyperparam-sensitivity}(c,d) show that \gemrr can tolerate large fractions of the data getting corrupted and is not sensitive to $d$.

\textbf{Testing \gem for Global Convergence.} To test the effect of initialization, in Fig~\ref{fig:conv}(e), corruptions were introduced using an adversarial model $\tilde{\mathbf{w}}$ i.e. for corrupted points, labels were set to $\tilde y_i = \ip{\tilde{\mathbf{w}}}{\vx_i}$. \gemrr was initialized at 1000 randomly chosen models, the origin, as well as at the adversarial model $\tilde\vw$ itself. WORST-1000 (resp. AVG-1000) indicate the worst (resp. average) performance \gem had at any of the 1000 initializations. Fig~\ref{fig:conv}(f) further emphasizes this using a toy 2D problem. \gem was initialized at all points on the grid. An initialization was called a success if \gem got error $< 10^{-6}$ within eight or fewer iterations. In all these experiments \gem rapidly converged to the true model irrespective of model initialization.

\section*{Acknowledgements}
The authors thank the anonymous reviewers of this paper for suggesting illustrative experiments and pointing to relevant literature. B.M. is supported by the Technology Innovation Institute and MBZUAI joint project (NO. TII/ARRC/2073/2021): Energy-based Probing for Spiking Neural Networks. D.D. is supported by the Research-I Foundation at IIT Kanpur and acknowledges support from the Visvesvaraya PhD Scheme for Electronics \& IT (FELLOW/2016-17/MLA/194). P.K. thanks Microsoft Research India and Tower Research for research grants.

\bibliographystyle{abbrvnat}
\bibliography{refs}

\appendix

\allowdisplaybreaks

\section{Summary of Assumptions}
\label{app:assumptions-limitations}
This paper presented \gem, a framework for robust GLM problems based on a novel variance reduced reweighted MLE technique that can be readily adapted to arbitrary GLM problems such as robust least squares/logistic/gamma regression and mean estimation. Here, we summarize the theoretical and empirical assumptions made by the \gem framework for easy inspection.

\emph{Experimental Assumptions.} \gem requires minimal assumptions to be executed in practice and only requires two scalar hyperparameters to be tuned properly. \gem is robust to minor misspecifications to its hyperparameters (see Figure~\ref{fig:hyperparam-sensitivity}) and the hyperparameter tuning described in \S\ref{sec:gem} works well in practice allowing \gem to offer superior or competitive empirical performance when compared to state of the art techniques for task-specific estimation techniques e.g. TORRENT for robust regression, classical and popular techniques such as Tukey's bisquare or geometric median for robust mean estimation, as well as recent advances in robust gradient-based techniques such as SEVER and RGD.

\emph{Theoretical Assumptions.} \gem establishes explicit breakdown points in several interesting scenarios against both partially and fully adaptive adversaries. To do so, \gem assumes a realizable setting e.g. in least-squares regression, labels for the $G$ clean points are assumed to be generated as $y_i = \ip\vwo{\vx_i} + \epsilon_i$ where $\vwo$ is the gold model and $\epsilon_i \sim \cN\br{0,\frac1\betao}$ is Gaussian noise. Of course, on the $B$ bad points, the (partially/fully adaptive) adversary is free to introduce corruptions jointly in any manner. For least-squares regression, the covariates/feature vectors i.e. $\vx_i$ are assumed to be sampled from some sub-Gaussian distribution that includes arbitrary bounded distributions as well as multivariate Gaussian distributions in $d$ dimensions (both standard and non-standard) -- see Appendix~\ref{app:rr} for details. For gamma regression and mean estimation settings, the covariates are assumed to be sampled from a spherical multivariate Gaussian in $d$ dimensions. Other assumptions are listed in the statements of the theorems in \S\ref{sec:me-rr-lr}.

\section{Adversary Models}
\label{app:adversary}

We explain the various adversary models in more detail here. Several models popular in literature give various degrees of control to the adversary. This section offers a more relaxed discussion of some prominent adversary models along with examples of applications in which they arise.

\textbf{(Oblivious) Huber Adversary.} Corruption locations $i_1,\ldots,i_k$ are chosen randomly for which corrupted labels are sampled i.i.d. from some pre-decided distribution $\cB$ i.e. $\tilde y_{i_j} \sim \cB$. Next, data features $\vx^i$ and the true model $\vwo$ are selected and \emph{clean} labels are generated according to the GLM for all non-corrupted points.

\textbf{Partially Adaptive Adversary.} The adversary first chooses the corruption locations $i_1,\ldots,i_k$ (e.g. some fixed choice or randomly). Then data features $\vx^i$ and the true model $\vwo$ are selected and \emph{clean} labels $y_1,\ldots,y_n$ are generated according to the GLM. Next, the adversary is presented with the collection $\bc{\vwo, \bc{(\vx^i,y_i)}_{i=1}^n, \bc{i_j}_{j = 1}^k}$ and is allowed to use this information to generate corrupted labels $\tilde y_{i_j}$ for the points marked for corruption.

\textbf{Fully Adaptive Adversary.} First data features $\vx^i$ and the true model $\vwo$ are selected and \emph{clean} labels $y_1,\ldots,y_n$ are generated according to the GLM. Then the adversary is presented with the collection $\bc{\vwo, \bc{(\vx^i,y_i)}_{i=1}^n}$ and is allowed to use this information to select which $k$ points to corrupt as well as generate corrupted labels $\tilde y_{i_j}$ for those points.

\textbf{Discussion.} The fully adaptive adversary can choose corruption locations $i_1,\ldots,i_k$ and the corrupted labels $\tilde y_{i_j}$ with complete information of the true model $\vwo$, the clean labels $\bc{y_i}$ and the feature vectors $\bc{\vx^i}$ and is the most powerful. The partially adaptive adversary can decide the corruptions $\tilde y_{i_j}$ after inspecting $\vwo,\bc{\vx^i, y_i}$ but cannot control the corruption locations. This can model e.g. an adversary that corrupts user data by installing malware on their systems. The adversary cannot force malware on a system of their choice but can manipulate data coming from already compromised systems at will. The Huber adversary is the least powerful with corruption locations that are random as well as corrupted labels that are sampled randomly and cannot depend on $\bc{\vx^i, y_i}$. Although weak, this adversary can nevertheless model sensor noise e.g. pixels in a CCD array that misfire with a certain probability. As noted earlier, \gem results are shown against both fully and partially adaptive adversaries.
\section{Experimental Setup}
\label{app:setup}

Experiments were carried out on a 64-bit machine with Intel® Core™ i7-6500U CPU @ 2.50GHz, 4 cores, 16 GB RAM and Ubuntu 16.04 OS. Statistics such as dataset size $n$, feature dimensions $d$ and corruption fraction $\alpha$ are mentioned above each figure. 20\% of train data was used as a held-out validation set using which ($\beta_1, \step$) were tuned using line search. Figure~\ref{fig:hyperparam-sensitivity} indicates that \gem is not sensitive to setting $\beta_1$ or $\step$ and a good value can be found using line search.

\textbf{Synthetic Data Generation.} Synthetic datasets were used in the experiments to demonstration recovery of the true model parameters. All regression co-variates/features were generated using a standard normal distribution $\cN(0, I_d)$. Clean responses in the least-squares regression settings were generated without additional Gaussian noise while corrupted responses were generated for an $\alpha$ fraction of the data points. For least-squares regression, clean responses were generated as $y_i = \ip{\vwo}{\vx_i}$, while for logistic regression, clean binary labels were generated as $y_i = \bI[\ip{\vwo}{\vx_i} > 0]$. For gamma regression, clean responses were generated using a likelihood distribution with vanishing variance. This was done by using a variance-altered likelihood distribution with the setting $\betao \rightarrow \infty$ (see Table~\ref{tab:variance} for the likelihood expressions). Clean data points for mean estimation were sampled from $\cN(\vmuo, \frac{1}{d}\cdot I_d)$. When not stated otherwise, corrupted labels were generated using an adversarial model. Specifically, to simulate the adversary, an \emph{adversarial} model $\tvw$ (for least-squares/logistic/gamma regression) or $\tilde\vmu$ (for mean estimation) was sampled and labels for data points chosen for corruption were generated using this adversarial model instead of the true model. For example, corrupted labels were generated for least-squares regression as $\tilde y_i = \ip{\tvw}{\vx_i}$ for all $k$ locations chosen for corruption. Please also refer to Appendix~\ref{app:exps} for an extensive study on several other ways of simulating the adversary and initialization schemes.

\textbf{Simulating the Adversary.} Corruptions were introduced using an adversarial model. Specifically, an adversarial model $\tvw$ was chosen, and for bad points, the adversary generated a label using $\tvw$, which overwrote the true label. For least-squares/logistic/gamma regression, both $\vwo,\tvw$ were independently chosen to be random unit vectors. For robust mean estimation Fig~\ref{fig:conv}(a), $\vmuo$ and $\tilde\vmu$ were chosen as random Gaussian vectors of length $2$ and $6$ respectively. Except Fig~\ref{fig:conv}(e,f) in which the setting is different, \gem variants were always initialized at the adversarial model itself i.e. $\hvw^1 = \tvw$ to test the ability of \gem to converge to the true model no matter what the initialization.

\textbf{Other Corruption Models.} \gem was found to offer superior robustness as compared to competitor algorithms against a wide range of ways to simulate adversarial corruption, including powerful ones that use expensive leverage score computations to decide corruptions. Fig~\ref{fig:corruption_model} in Appendix~\ref{app:exps} reports results of experiments with a variety of adversaries.

\begin{figure}[t]
	\centering
	\includegraphics[width=0.5\textwidth]{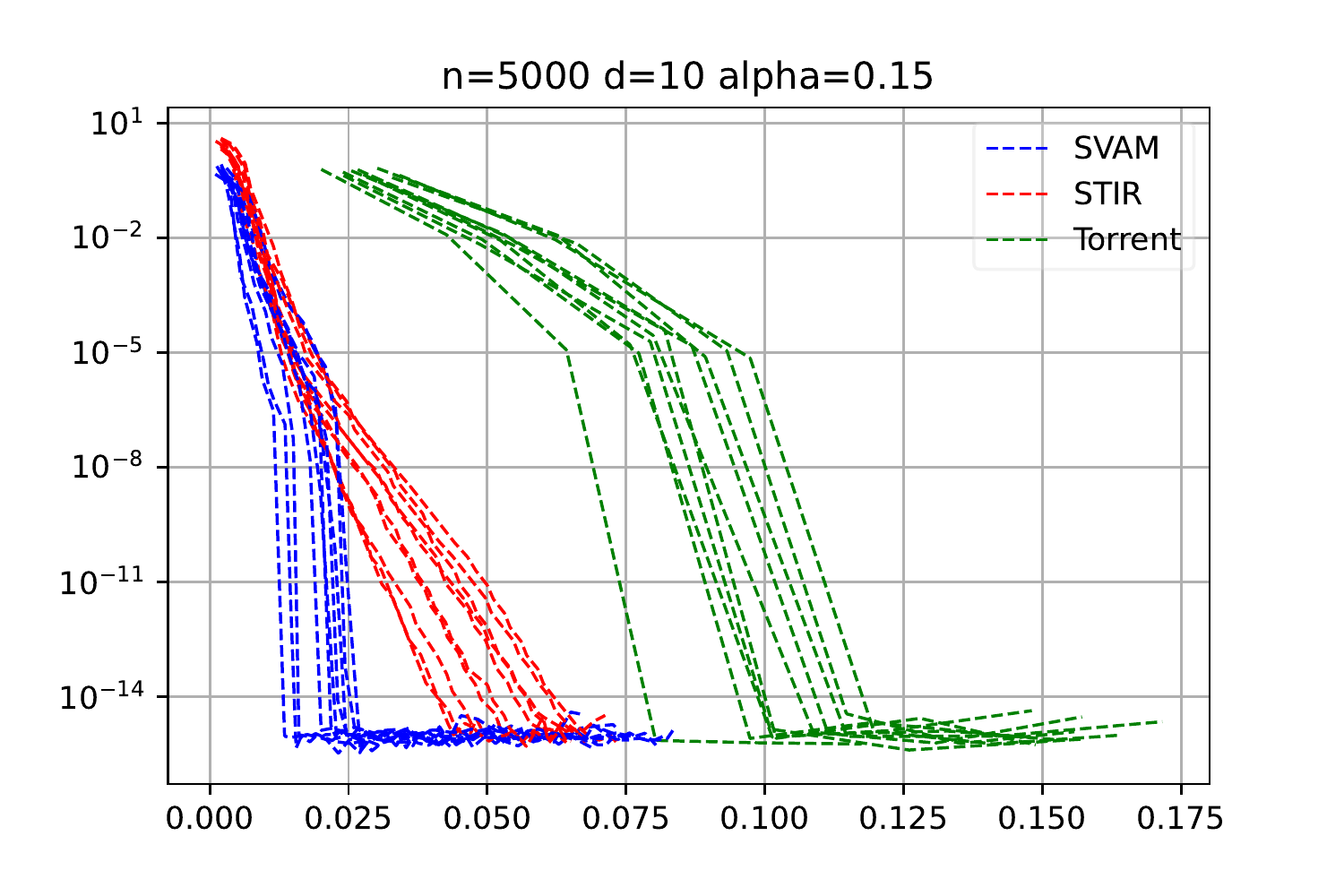}
 	\caption{Convergence results obtained after repeated experiments with the 3 leading methods for robust least squares regression. \gem continues to lead despite the natural variance in convergence plots.}
	\label{fig:avg-conv}%
\end{figure}

\begin{figure}[t]
	\centering
	\includegraphics[width=0.5\textwidth]{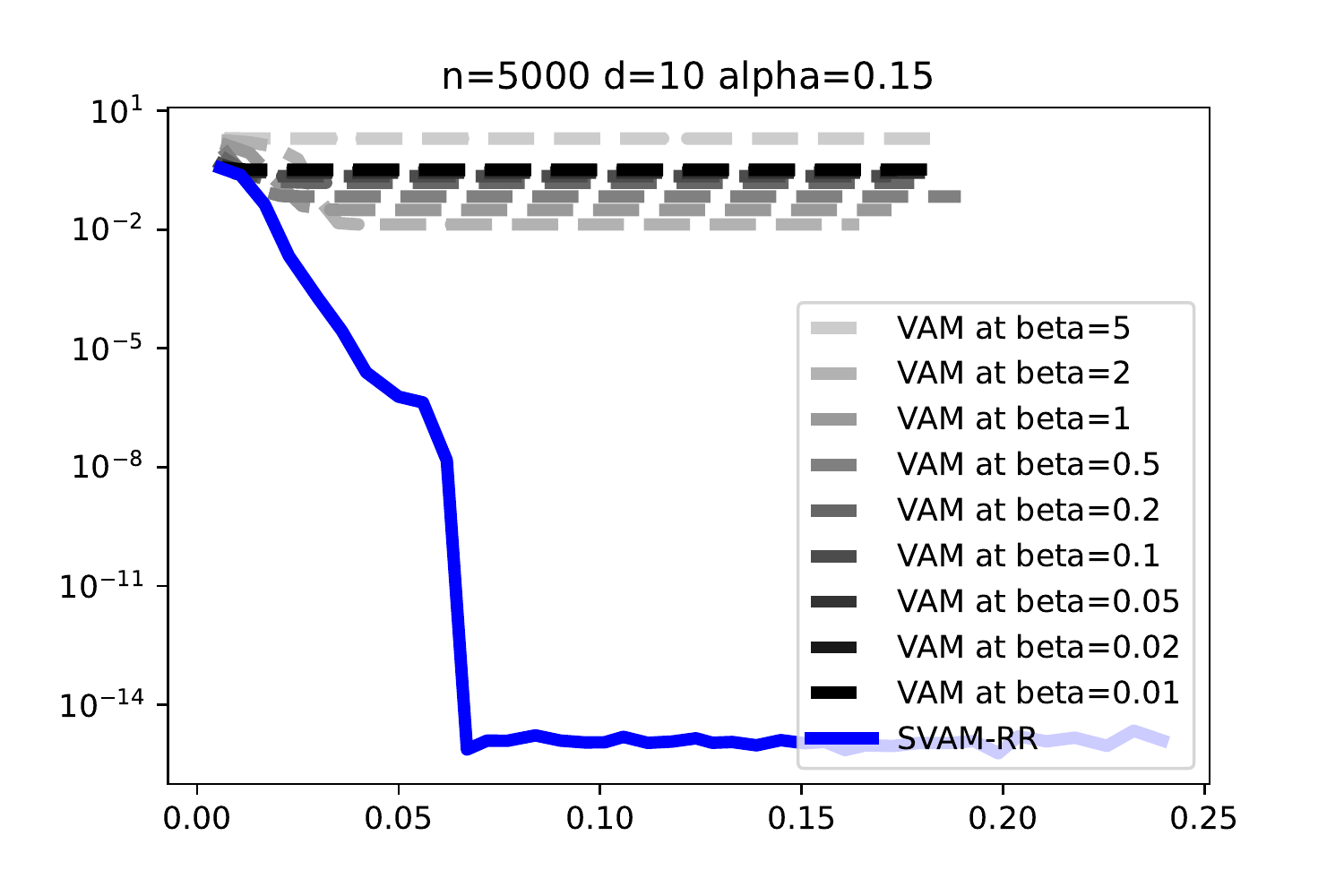}
 	\caption{Convergence results for VAM with various (fixed) values of $\beta$ for robust least squares regression. The plots indicate that the best performance achieved by VAM is with a moderately large value of $\beta$. Excessively large and excessively small values of $\beta$ are both ill-suited to model recovery. However, even if provided such an optimal value of $\beta$, VAM's performance is still inferior to that of \gem. We note that \gem does not use a fixed value of $\beta$ and instead dynamically updates it.}
	\label{fig:vam-beta}%
\end{figure}

\begin{figure}[t]
	\centering
	\includegraphics[width=0.32\textwidth]{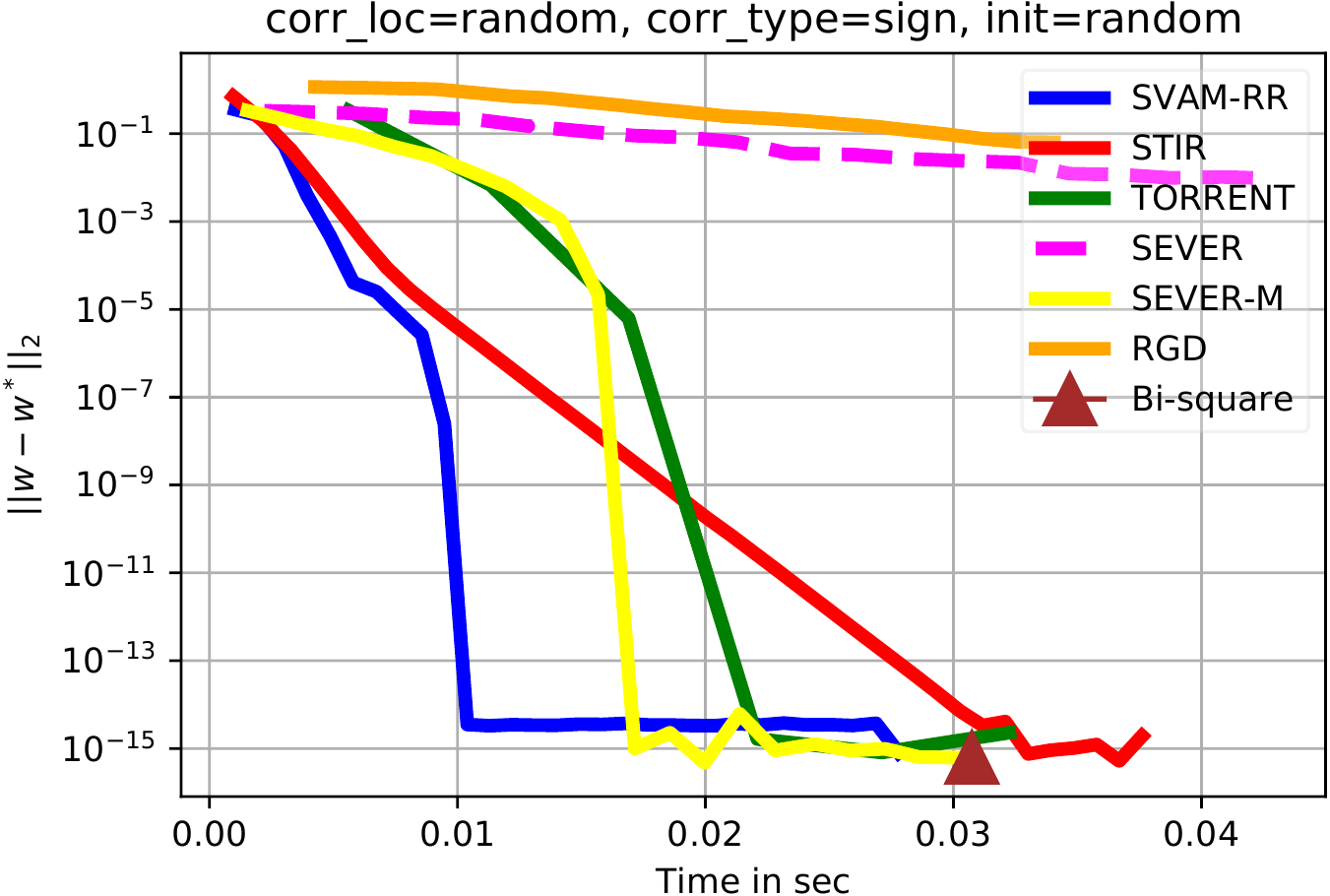}
	\includegraphics[width=0.32\textwidth]{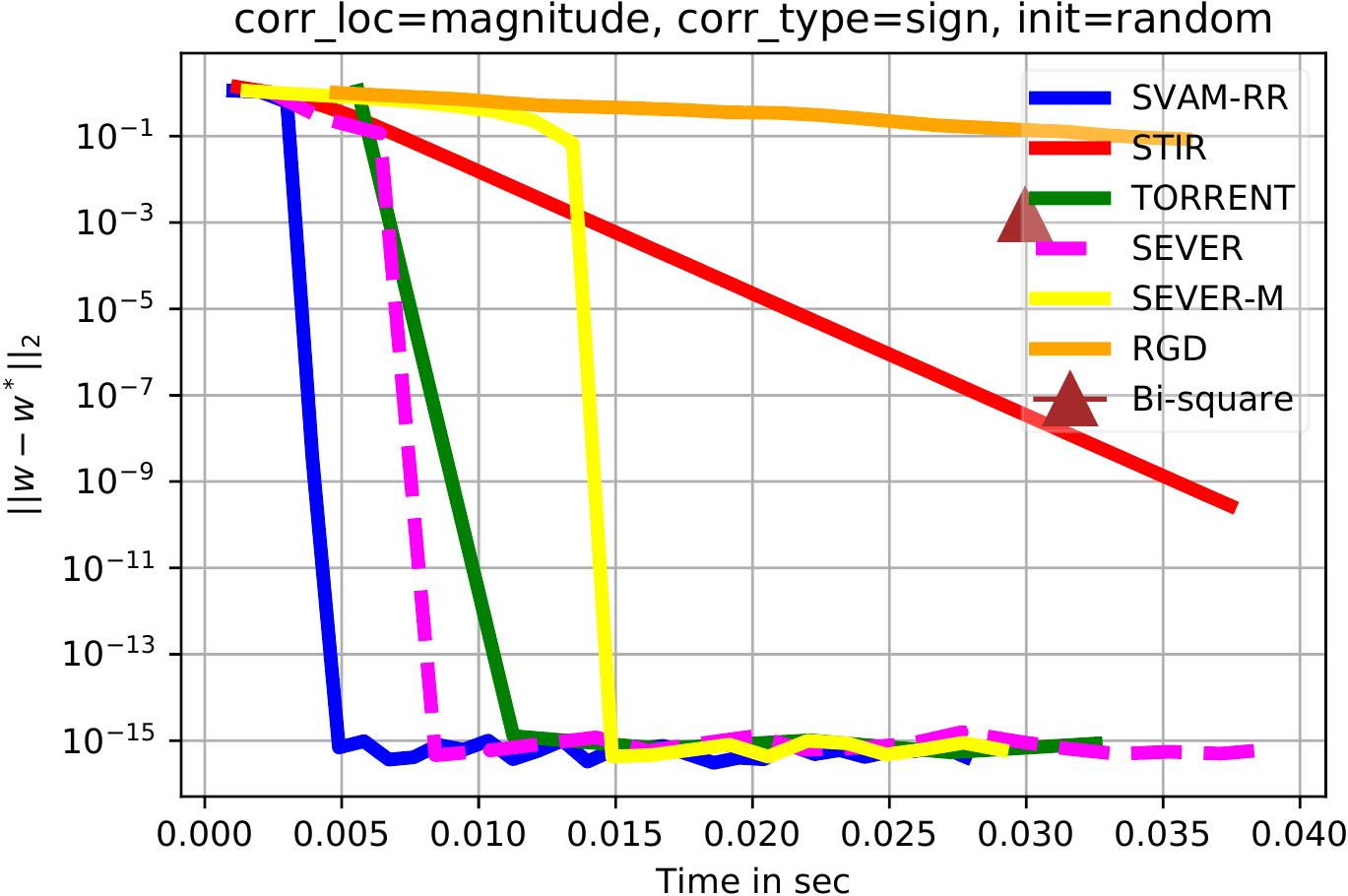}
	\includegraphics[width=0.32\textwidth]{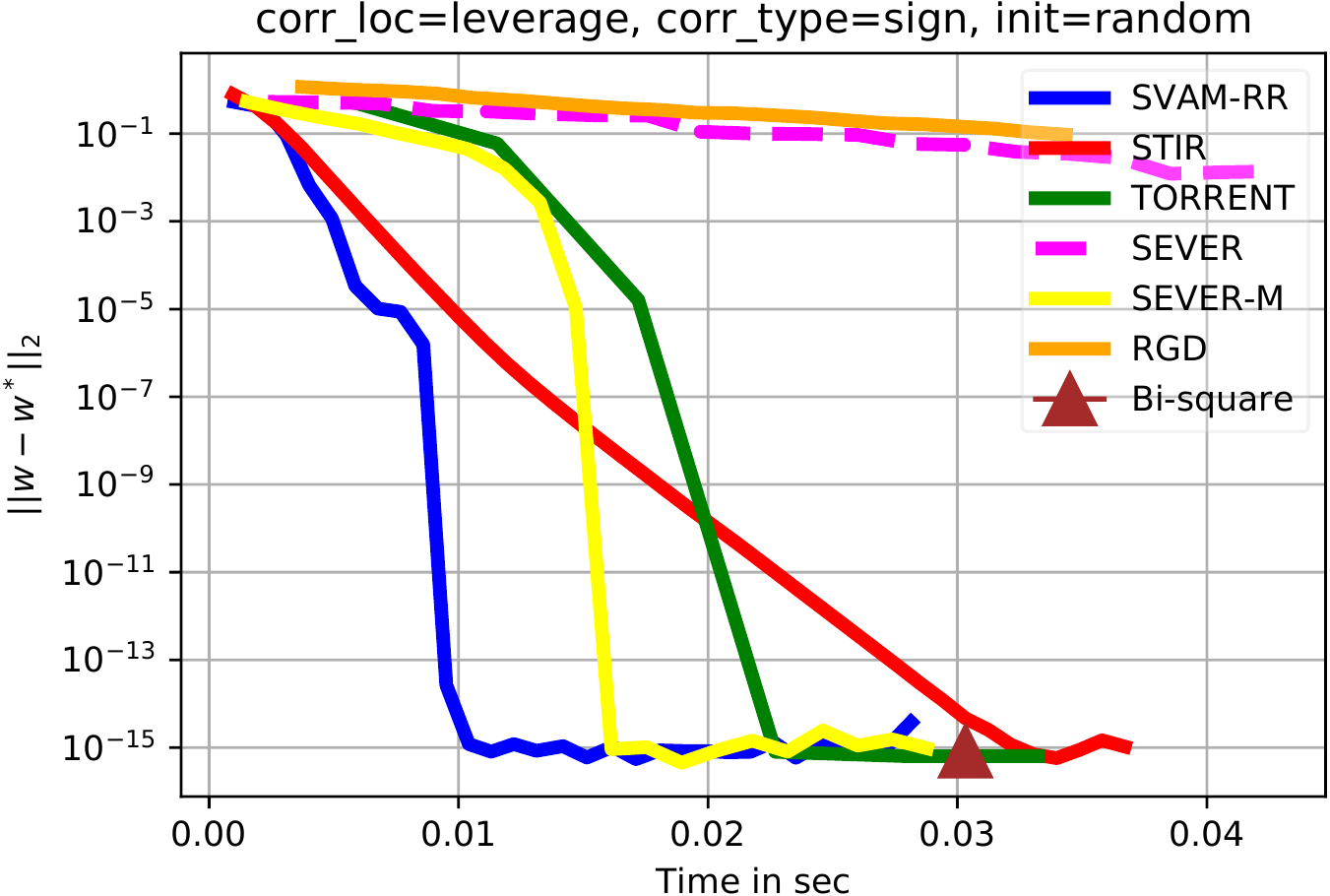}\\
	\includegraphics[width=0.32\textwidth]{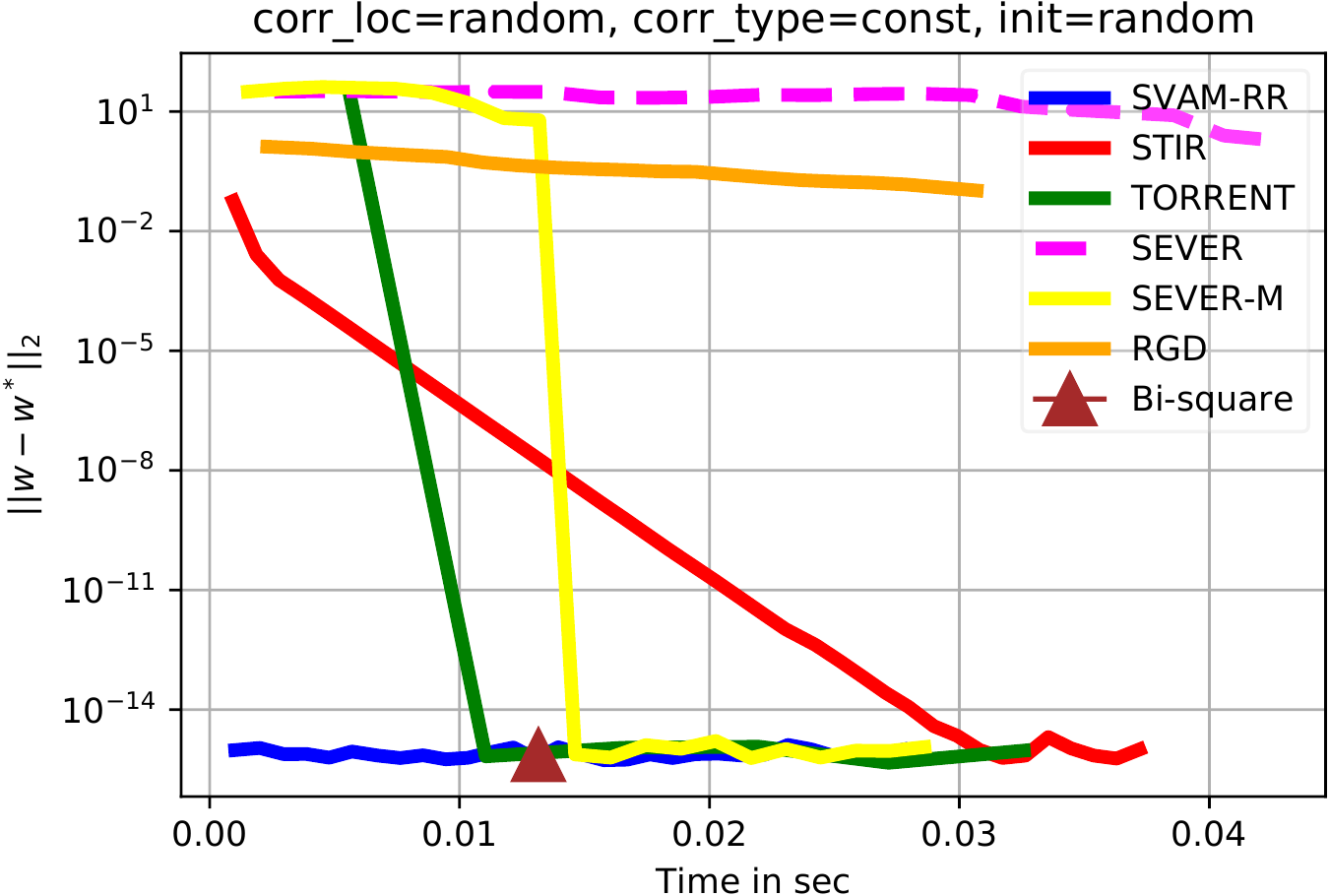}
	\includegraphics[width=0.32\textwidth]{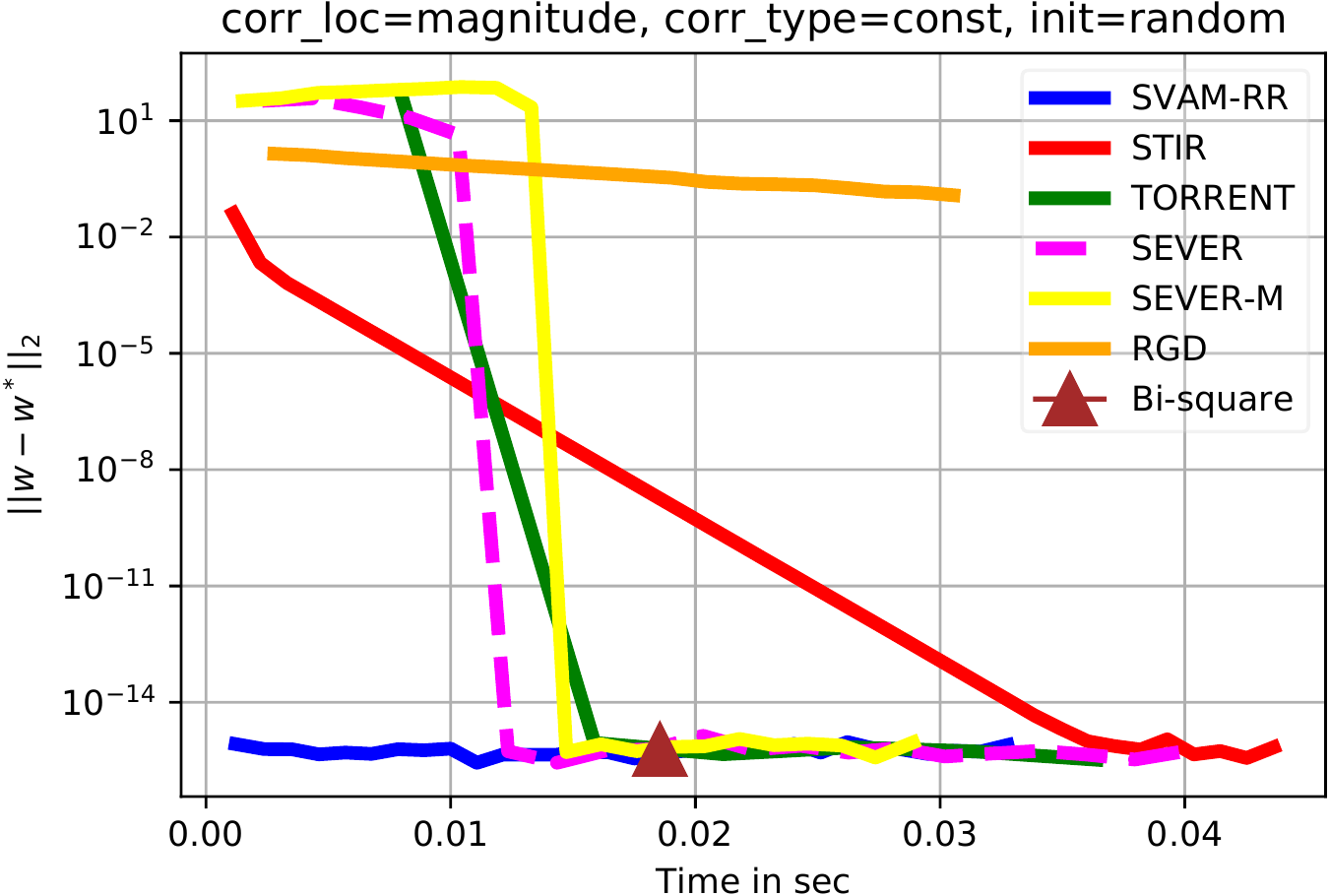}
	\includegraphics[width=0.32\textwidth]{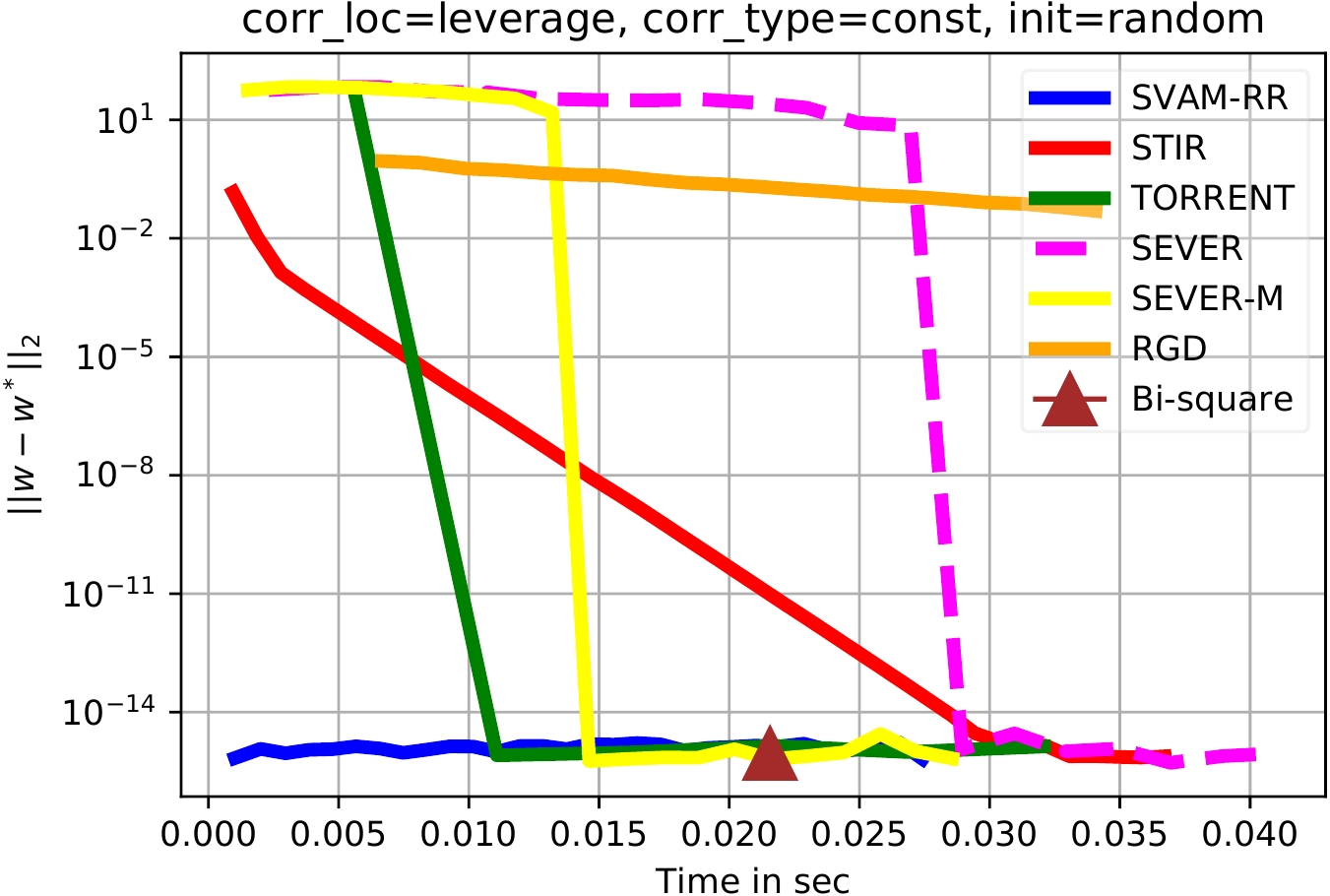}\\
	\includegraphics[width=0.32\textwidth]{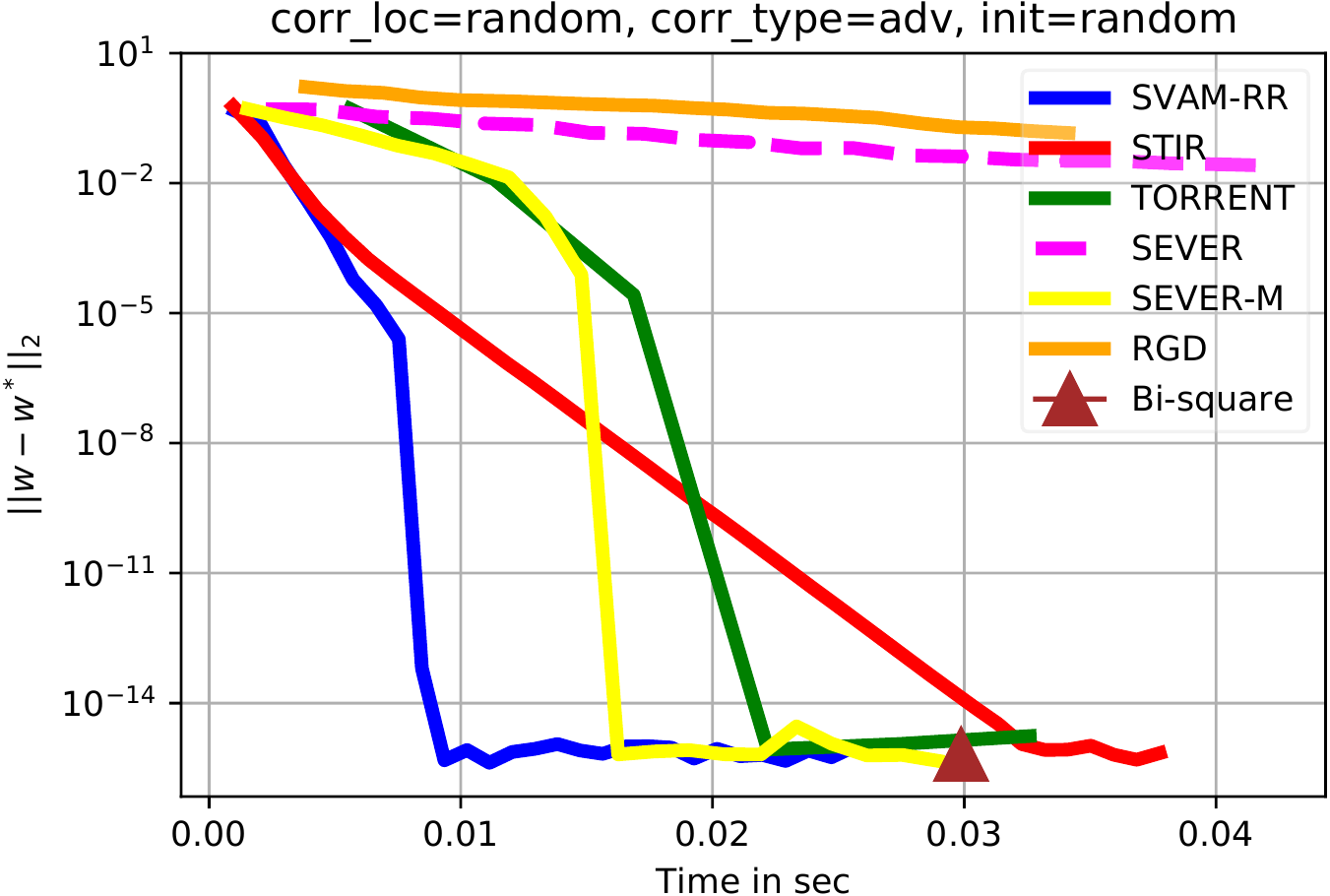}
	\includegraphics[width=0.32\textwidth]{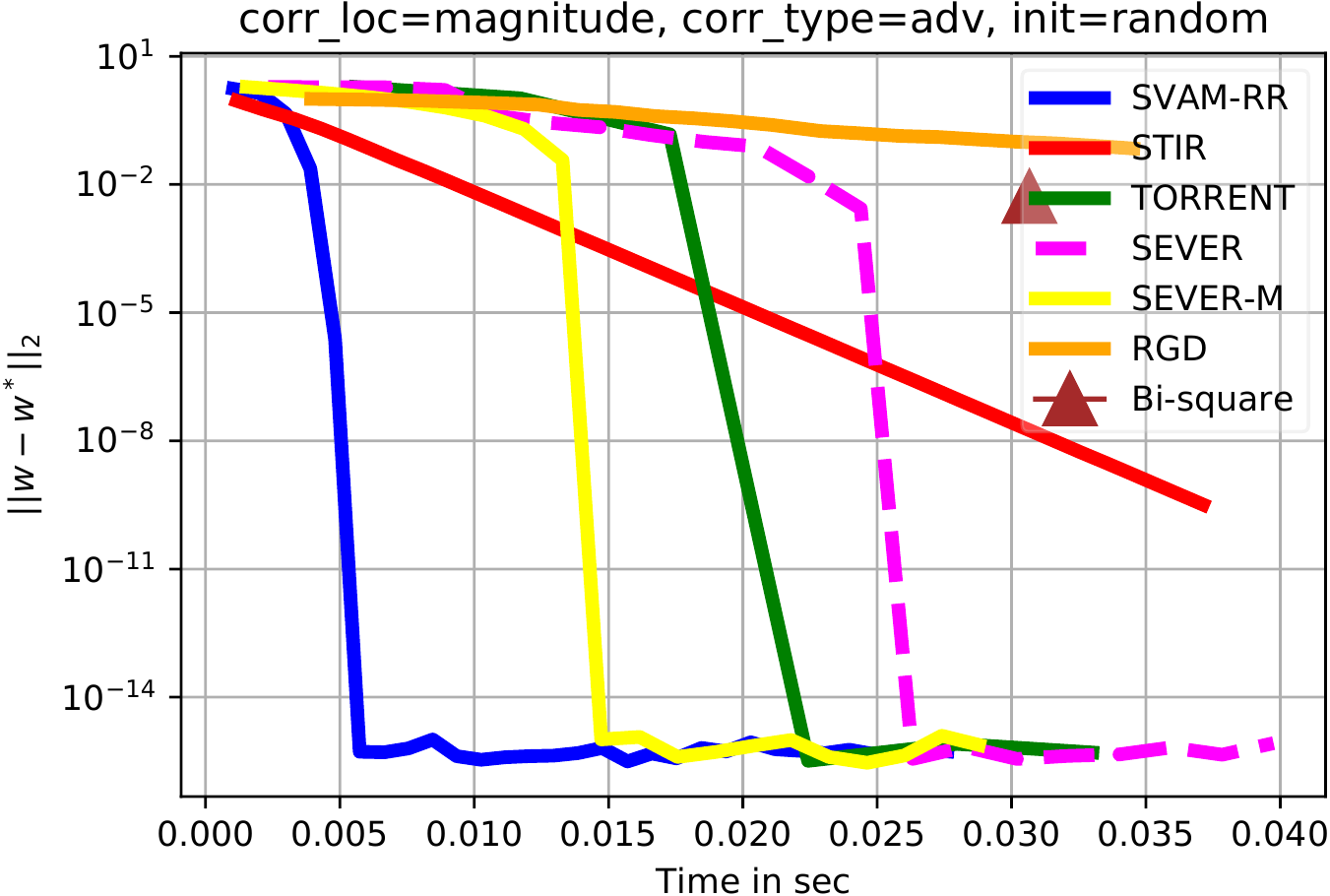}
	\includegraphics[width=0.32\textwidth]{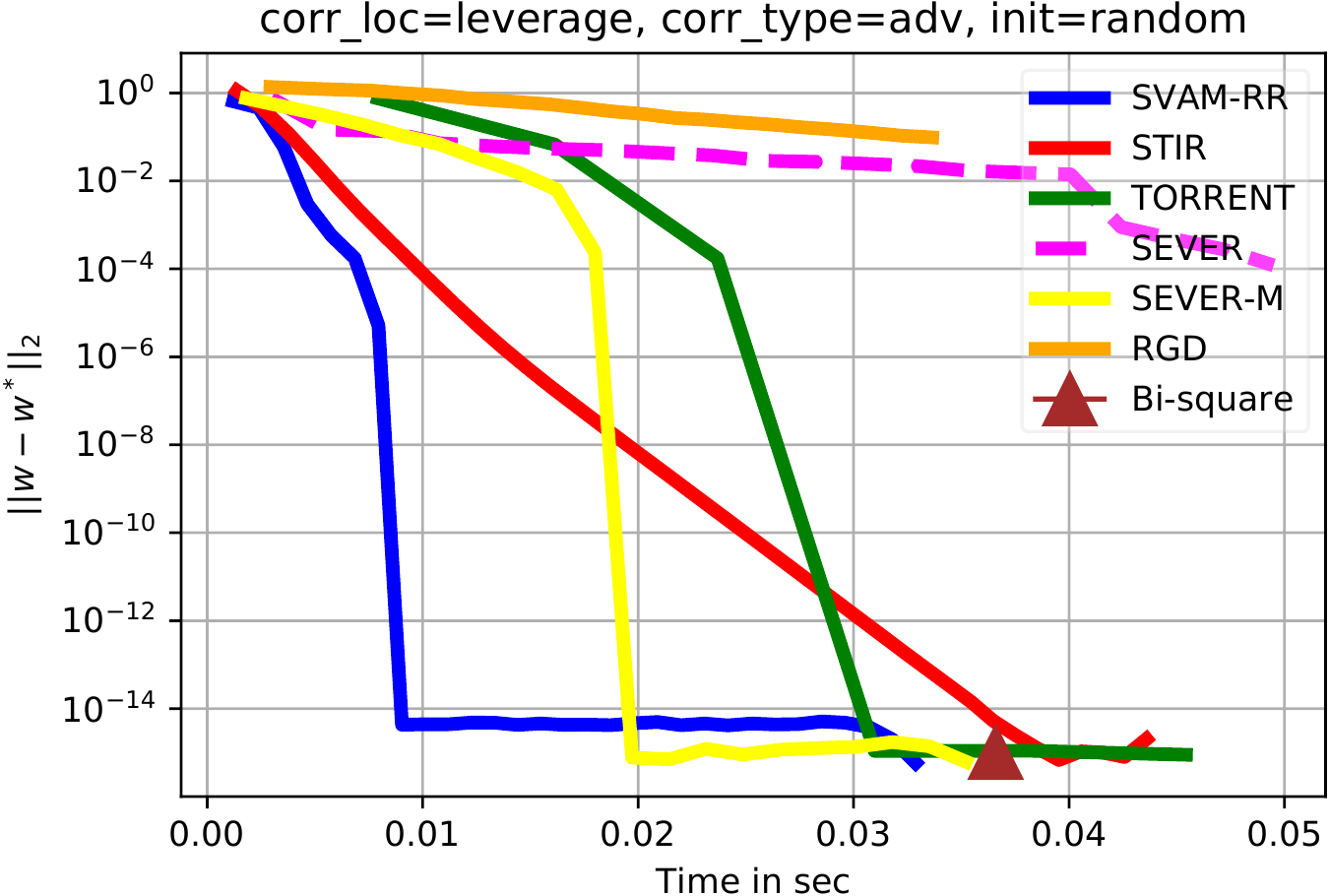}\\
	\includegraphics[width=0.32\textwidth]{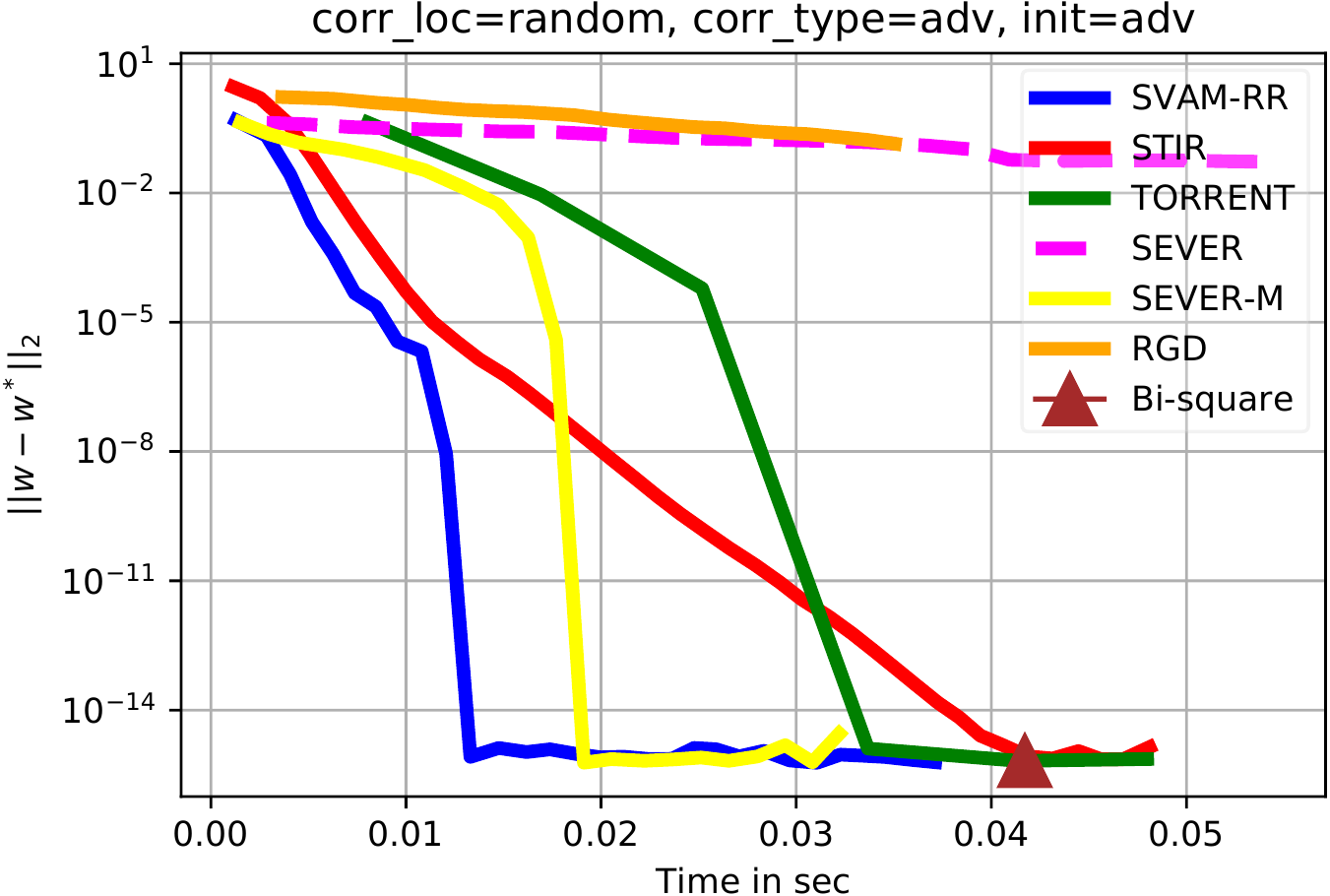}
	\includegraphics[width=0.32\textwidth]{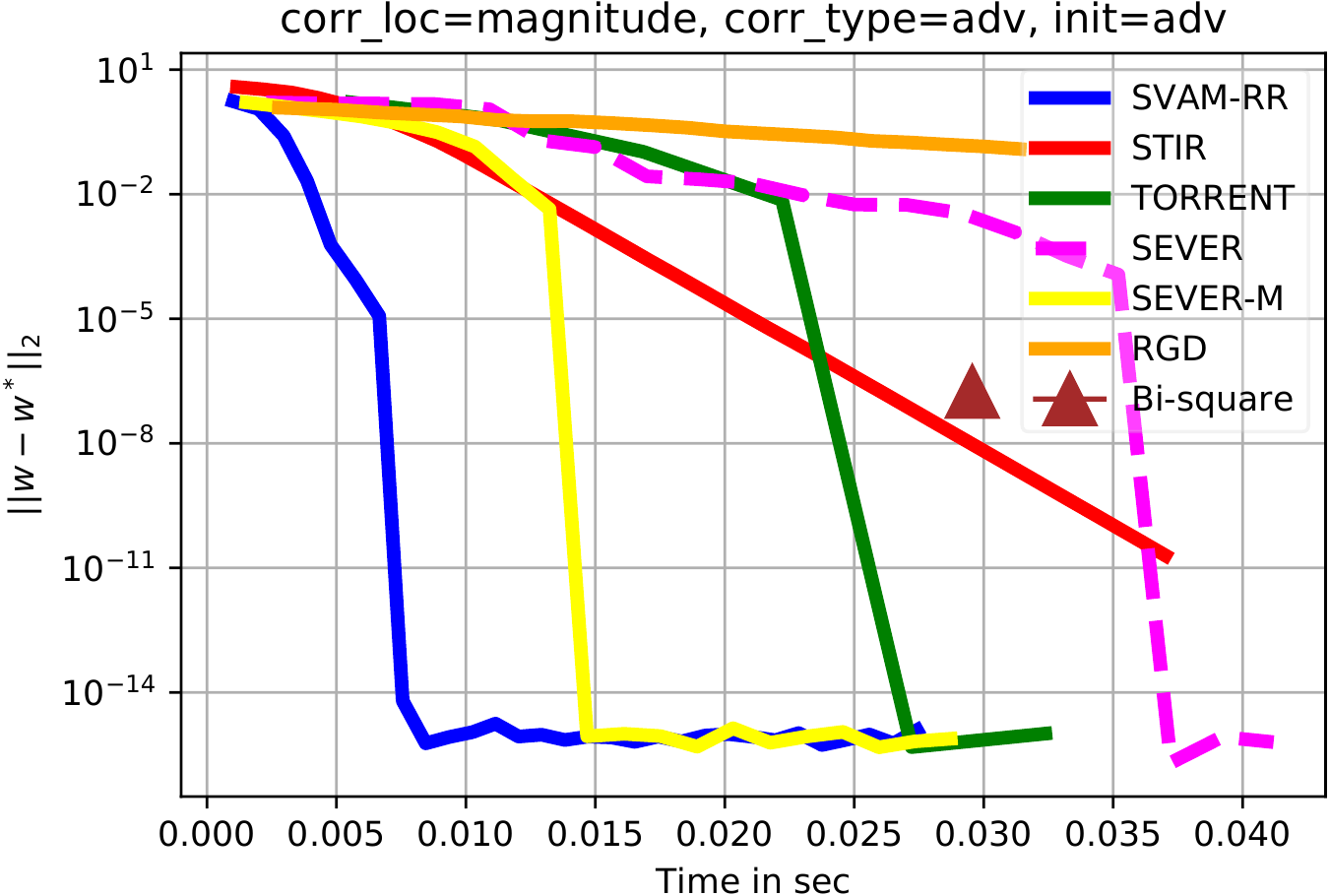}
	\includegraphics[width=0.32\textwidth]{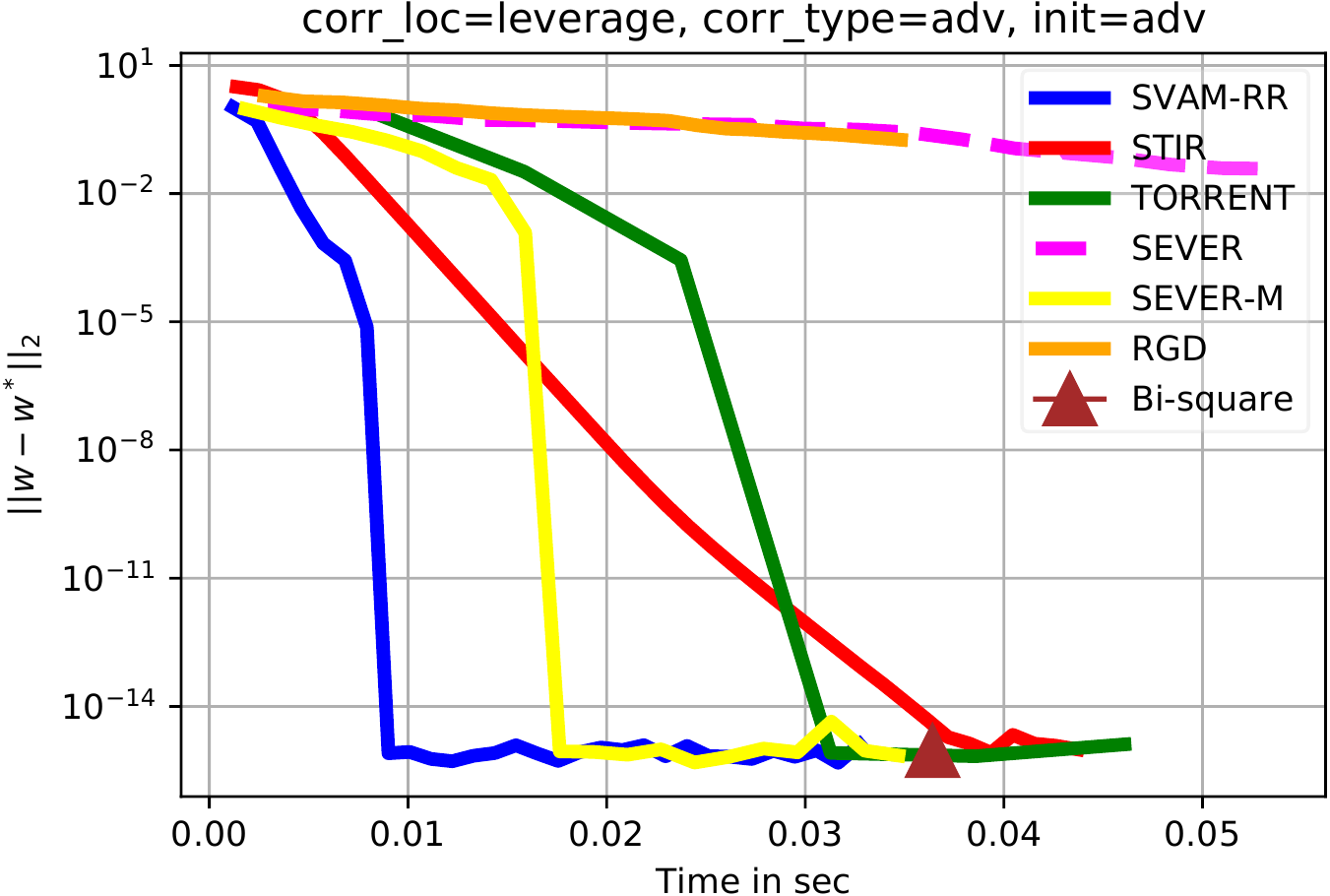}\\
	\caption{\gem exhibits a high degree of tolerance (superior to competitor algorithms) to a variety of corruption models such as choice of corruption location (e.g. choosing corruption locations randomly, based on leverage scores, based on magnitude of the clean label, etc), choice of type of corruption (e.g. sign-flip, constant shift, using an adversarial model, etc) and initialization schemes (e.g. random, adversarial, etc). Please see the text in Appendix~\ref{app:exps} for more details.}
	\label{fig:corruption_model}%
\end{figure}

\clearpage

\section{Tolerance to Different Adversaries and Adversarial Initialization Schemes}
\label{app:exps}

Figure \ref{fig:avg-conv} repeats the experiment of Fig. \ref{fig:conv}(a) 10 times, showing that the convergence results are consistent across several runs of the experiment. In this comparison, we put two closet competitors STIR and TORRENT, along with \gem, and omit other competitors to avoid clutter in the figure. The experiment shows that the performance of the methods does not vary very wildly across the runs and \gem's convergence continues to remain faster than its competitors.

Figure \ref{fig:vam-beta} considers the VAM method with several values of $\beta$ ranging from very small to very large. It is apparent that no single fixed value of $\beta$ is able to offer satisfactory result indicating that dynamically changing the value of $\beta$ as done by \gem is required for better convergence results.

Figure~\ref{fig:corruption_model} demonstrates recovery of $\vwo$, under different ways of simulating the adversary and initialization of algorithms. We consider corruption models, that differ in choosing locations for corruption and the way chosen points are corrupted. We demonstrate recovery, on choosing points for corruption in three ways: i) random, ii) using magnitude of response and iii) using leverage statistic score. The leverage score of a point increases as it lies farther from the mean of the data points. The diagonal elements of the projection matrix $P=X^\top(XX^\top)^{-1}X$, gives the respective leverage scores and $k$ data points having largest leverage score are choosen for corruption. After selecting the data points to be corrupted, we either i) flip the sign of the response, ii) set the responses to a constant value B or iii) use an adversarial model $\tilde{\mathbf{w}}$, to generate the corrupted response i.e. set the  response to $\tilde y_i = \ip{\tilde{\mathbf{w}}}{\vx_i}$. The initialization offered was also varied in two ways: i) random initialization, and ii) adversarial initialization where \gem was initialized at $\tilde\vw$, the same adversarial model using which corruptions were introduced. 

Experiments were performed by varying location of corruption in \{random, absolute magnitude, leverage score\}, corruption type in \{adversarial, sign change, set constant\} and initialization in \{random, adversarial\}. The experimental setting is given in the title of each figure, while all of them have $n=1000$, $d=10$ and $\alpha=0.15$. It can be observed that in general, \gem demonstrates superior performance irrespective of the adversarial models and initialization.

It can be also observed that SVAM converges after single iteration when corruptions are introduced by setting responses to a constant (second column) whereas Tukey's bisquare method does not converge well, when absolute magnitude is used to select location of corruption(second row), except for constant response corruptions.

\section{Proof of Theorem~\ref{thm:gem-main}}
\begin{theorem}[\gem convergence - Restated]
\label{thm:gem-main-restated}
Suppose the data and likelihood distribution satisfy the $\lambda_\beta$-LWSC and $\Lambda_\beta$-LWLC properties for all values of $\beta$ in the range $(0,\beta_{\max}]$. Then if \gem is initialized at a point $\hvw^1$ and initial scale $\beta_1 > 0$ such that $\beta_1\cdot\norm{\hvw^1 - \vwo}_2^2 \leq 1$ then for any $\epsilon > \frac1{\beta_{\max}}$, for small-enough scale increment $\step > 1$, \gem ensures $\norm{\hvw^T - \vwo}_2 \leq \epsilon$ within $T = \bigO{\log\frac1\epsilon}$ iterations.
\end{theorem}
\begin{proof}
The key to this proof is to maintain the invariant $\beta_t\cdot\norm{\hvw^t - \vwo}_2^2 \leq 1$. Note that initialization is done precisely to ensure this at the beginning of the execution of the algorithm which acts as the base case for an inductive argument. For the inductive case, consider an iteration $t$ and let $\sqrt{\beta_t}\cdot\norm{\hvw^t - \vwo}_2 \leq 1$. LWSC ensures strong convexity giving
\[
\tilde Q_{\beta_t}(\hvw^{t+1} \cond \hvw^t) - \tilde Q_{\beta_t}(\vwo \cond \hvw^t) \geq \ip{\nabla \tilde Q_{\beta_t}(\vwo\cond\hvw^t)}{\hvw^{t+1} - \vwo} + \frac{\lambda_{\beta_t}}2\norm{\hvw^{t+1} - \vwo}_2^2
\]
Since $\hvw^{t+1}$ minimizes $\tilde Q_{\beta_t}(\cdot \cond \hvw^t)$, we have $\tilde Q_{\beta_t}(\hvw^{t+1} \cond \hvw^t) \leq \tilde Q_{\beta_t}(\vwo \cond \hvw^t)$. Elementary manipulations and the Cauchy-Schwartz inequality now give us
\[
\norm{\hvw^{t+1} - \vwo}_2 \leq \frac{2\norm{\nabla \tilde Q_{\beta_t}(\vwo\cond\hvw^t)}_2}{\lambda_{\beta_t}} \leq \frac{2\Lambda_{\beta_t}}{\lambda_{\beta_t}}.
\]
Now, we will additionally ensure that we choose the scale increment $\step$ to be small enough (while still ensuring $\step > 1$) such that $\frac{2\Lambda_\beta}{\lambda_\beta} < \sqrt\frac1{\step\beta}$ for all $\beta \in (0,\beta_{\max}]$. Combining this with the above result gives us
\[
\norm{\hvw^{t+1} - \vwo}_2^2 \leq \frac1{\step\beta_t}
\]
Thus, if we now set $\beta_{t+1} = \step\beta_t$, then rearranging the terms in the above inequality tell us that $\beta_{t+1}\cdot\norm{\hvw^{t+1} - \vwo}_2^2 \leq 1$ which lets us continue the inductive argument. Note that this process can continue on till we have $\beta_t \leq \beta_{\max}$ since the LWSC/LWLC properties are assured till that point. Moreover, since $\beta_t$ goes up by a constant fraction at each step and $\norm{\hvw^t - \vwo}_2^2 \leq \frac1{\beta_t}$ due to the invariant, a linear rate of convergence is assured which finishes the proof. The existence of a suitable scale increment $\step$ satisfying the above requirements is established in a case-wise manner by Theorems~\ref{repthm:rr-main}, \ref{thm:rr-fully-adaptive-restated}, \ref{repthm:me-main} and \ref{repthm:gam-main}. We also note that, as discussed in \S\ref{sec:me-rr-lr} and elaborated in Appendix~\ref{app:gam}, since gamma regression requires an alternate parameterization owing to its need to support only non-negative labels, the invariant used for the convergence bound for \gemgam is also slightly altered as mentioned in Thm~\ref{thm:gam-main} to instead use $\beta_t\cdot(\exp(\norm{\hvw^t - \vwo}_2) - 1)^2 \leq 1$.
\end{proof}

\section{Some Helpful Results}
Below we present a few helpful results.

\begin{lemma}
\label{lem:rx-bound}
Suppose $\vepsilon^i \sim \cN(\vzero, I), i = 1, \ldots, n$ and denote $R_X := \max_{i \in n} \norm{\vepsilon^i}_2$. Then we have $R_X \leq \sqrt n$ with probability at least $1 - \exp(-\Om n)$.
\end{lemma}
\begin{proof}
Follows from standard arguments.
\end{proof}

\begin{lemma}
\label{lem:ssc-sss}
For covariate vectors $X = \bs{\vx_1,\ldots,\vx_n}$ generated from an isotropic sub-Gaussian distribution, for any fixed set $S \subset [n]$ and $n = \Om{d}$, with probability at least $1-\exp(-\Om d)$,
\[
0.99\abs{S} \leq \lambda_{\min}(X_SX_S^\top) \leq \lambda_{\max}(X_SX_S^\top) \leq 1.01\abs{S},
\]
where the constant inside $\Om{\cdot}$ depends only on the sub-Gaussian distribution and universal constants.
\end{lemma}
\begin{proof}
Taken from \cite{BhatiaJK2015}.
\end{proof}

\begin{lemma}
\label{lem:change-of-exp}
If $\vepsilon \sim \cN(\vzero, \frac1\betao\cdot I)$, then for any function $f: \bR^d \rightarrow \bR$, any $\beta > 0$ and any $\vDelta \in \bR^d$, we have
\[
\E{\exp\br{-\frac\beta2\norm{\vepsilon - \vDelta}_2^2}\cdot f(\vepsilon)} = \frac1\cons\cdot\E{f(\vx)},
\]
where $\vx \sim \cN\br{\frac{\beta}{\beta+\betao}\vDelta, \frac1{\beta+\betao}\cdot I}$ and $\cons = \br{\sqrt{\frac{\beta+\betao}\betao}}^d\exp\br{\frac{\beta\betao}{2(\beta+\betao)}\norm{\vDelta}_2^2}$.
\end{lemma}
\begin{proof}
We have
\begin{figure}[H]
\begin{adjustbox}{max width=\textwidth}
\parbox{\linewidth}{
\begin{align*}
	&\E{\exp\br{-\frac\beta2\norm{\vepsilon - \vDelta}_2^2}\cdot f(\vepsilon)} = \br{\sqrt\frac\betao{2\pi}}^d\idotsint_{\bR^d}\exp\br{-\frac\beta2\norm{\vepsilon - \vDelta}_2^2}\cdot f(\vepsilon)\cdot \exp\br{-\frac\betao2\norm{\vepsilon}_2^2}\ d\vepsilon\\
	&= \br{\sqrt\frac\betao{2\pi}}^d\exp\br{-\frac\beta2\norm{\vDelta}_2^2}\idotsint_{\bR^d}\exp\br{-\frac{\beta+\betao}2\norm{\vepsilon}_2^2 + \beta\vepsilon^\top\vDelta}\cdot f(\vepsilon)\ d\vepsilon\\
	&= \br{\sqrt\frac\betao{2\pi}}^d\exp\br{-\frac\beta2\norm{\vDelta}_2^2+\frac{\beta^2}{2(\beta+\betao)}\norm{\vDelta}_2^2}\idotsint_{\bR^d}\exp\br{-\norm{\sqrt\frac{\beta+\betao}2\vepsilon - \frac{\beta}{\sqrt{2(\beta+\betao)}}\vDelta}_2^2}\cdot f(\vepsilon)\ d\vepsilon\\
	&= \br{\sqrt\frac\betao{2\pi}}^d\exp\br{-\frac{\beta\betao}{2(\beta+\betao)}\norm{\vDelta}_2^2}\idotsint_{\bR^d}\exp\br{-\frac{\beta+\betao}2\norm{\vepsilon - \frac{\beta}{\beta+\betao}\vDelta}_2^2}\cdot f(\vepsilon)\ d\vepsilon\\
	&= \br{\sqrt\frac\betao{2\pi}}^d\br{\sqrt\frac{2\pi}{\beta+\betao}}^d\exp\br{-\frac{\beta\betao}{2(\beta+\betao)}\norm{\vDelta}_2^2}\E{f(\vx)}
\end{align*}
}
\end{adjustbox}
\end{figure}
which finishes the proof upon using $\cons\br{\sqrt\frac\betao{2\pi}}^d\br{\sqrt\frac{2\pi}{\beta+\betao}}^d\exp\br{-\frac{\beta\betao}{2(\beta+\betao)}\norm{\vDelta}_2^2} = 1$.
\end{proof}

\begin{lemma}
\label{lem:change-of-exp2}
If $\vepsilon \sim \cN(\vzero, \frac1\betao\cdot I)$, then for any function $f: \bR^d \rightarrow \bR$, any $\beta > 0$ and any $\vDelta \in \bR^d$, we have
\[
\E{\exp\br{\beta\vepsilon^\top\vDelta}\cdot f(\vepsilon)} = \ncons\cdot\E{f(\vx)},
\]
where $\vx \sim \cN\br{\frac\beta\betao\vDelta, \frac1\betao\cdot I}$ and $\ncons = \exp\br{\frac{\beta^2}{2\betao}\norm{\vDelta}_2^2}$.
\end{lemma}
\begin{proof}
We have
\begin{align*}
	\E{\exp\br{\beta\vepsilon^\top\vDelta}\cdot f(\vepsilon)} = \br{\sqrt\frac\betao{2\pi}}^d\idotsint_{\bR^d}\exp\br{\beta\vepsilon^\top\vDelta}\cdot f(\vepsilon)\cdot \exp\br{-\frac\betao2\norm{\vepsilon}_2^2}\ d\vepsilon\\
	= \br{\sqrt\frac\betao{2\pi}}^d\exp\br{\frac{\beta^2}{2\betao}\norm{\vDelta}_2^2}\idotsint_{\bR^d}\exp\br{-\frac12\norm{\sqrt{\betao}\vepsilon - \frac{\beta}{\sqrt{\betao}}\vDelta}_2^2}\cdot f(\vepsilon)\ d\vepsilon\\
	= \br{\sqrt\frac\betao{2\pi}}^d\exp\br{\frac{\beta^2}{2\betao}\norm{\vDelta}_2^2}\idotsint_{\bR^d}\exp\br{-\frac\betao2\norm{\vepsilon - \frac{\beta}{\betao}\vDelta}_2^2}\cdot f(\vepsilon)\ d\vepsilon
\end{align*}
which finishes the proof.
\end{proof}

\begin{lemma}
\label{lem:change-of-exp3}
If $\vepsilon \sim \cN(\vzero, \frac1\betao\cdot I)$, then for any constant $C > 0$ and fixed vectors $\vu, \vv$ such that $\norm\vu_2 = u$ and $\norm\vv_2 = v$, we have
\[
\E{\frac1{1+C\exp\br{\beta\vv^\top\vepsilon}}\cdot\vu^\top\vepsilon} \leq \min\bc{C,\frac1C}\exp\br{\frac{\beta^2v^2}{2\betao}}\frac{\beta v}{\betao}u
\]
\end{lemma}
\begin{proof}
We begin by analyzing the vector $\vnu = \E{\frac1{1+C\exp\br{\beta\vv^\top\vepsilon}}\cdot\vepsilon}$ itself. Note that due to the rotational symmetry of the Gaussian distribution, we can, w.l.o.g. assume that $\vv = (v, 0, 0, \ldots, 0)$. This means that $\vv^\top\vepsilon = v\cdot\vepsilon_1$. Thus, the $i\nth$ coordinate of the vector $\vnu$ i.e. $\vnu_i = \E{\frac1{1+C\exp\br{\beta v\cdot\vepsilon_1}}\cdot\vepsilon_i}$. Thus, by independence and unbiased-ness of the coordinates of a Gaussian vector, we have $\vnu_i = 0$ for all $i \neq 1$. So all we are left to analyze is $\vnu_1$. We have

\begin{align*}
		\abs{\vnu_1} &= \sqrt{\frac\betao{2\pi}}\abs{\int_\bR\frac{\exp\br{-\frac{\betao\epsilon^2}2}\epsilon}{1 + C\exp(\beta v\epsilon)}\ d\epsilon}\\
		&= \sqrt{\frac\betao{2\pi}}\int_0^\infty\exp\br{-\frac{\betao\epsilon^2}2}\epsilon\cdot\frac{C(\exp(\beta v\epsilon) - \exp(-\beta v\epsilon))}{1 + C^2 + C(\exp(\beta v\epsilon) + \exp(-\beta v\epsilon))}\ d\epsilon\\
		&\leq \sqrt{\frac\betao{2\pi}}\int_0^\infty\exp\br{-\frac{\betao\epsilon^2}2}\epsilon\cdot\frac{C\exp(\beta v\epsilon)}{1 + C^2 + C\exp(\beta v\epsilon)}\ d\epsilon\\
		&\leq \sqrt{\frac\betao{2\pi}}\int_0^\infty\exp\br{-\frac{\betao\epsilon^2}2}\epsilon\cdot\frac{C\exp(\beta v\epsilon)}{1 + C^2}\ d\epsilon\\
		&\leq \min\bc{C,\frac1C}\sqrt{\frac\betao{2\pi}}\int_0^\infty\exp\br{-\frac{\betao\epsilon^2}2}\epsilon\cdot\exp(\beta v\epsilon)\ d\epsilon\\
		&\leq \min\bc{C,\frac1C}\sqrt{\frac\betao{2\pi}}\int_\bR\exp\br{-\frac{\betao\epsilon^2}2}\epsilon\cdot\exp(\beta v\epsilon)\ d\epsilon\\
		&= \min\bc{C,\frac1C}\E{\exp(\beta v\epsilon)\cdot\epsilon}\\
		&= \min\bc{C,\frac1C}\exp\br{\frac{\beta^2v^2}{2\betao}}\frac{\beta v}{\betao}
\end{align*}
since $\frac C{1+C^2} \leq \min\bc{C,\frac1C}$ and we used Lemma~\ref{lem:change-of-exp2} in the last step since that lemma is independent of the dimensionality of the Gaussian vector. Applying the Cauchy-Schwartz inequality then gives us the result.
\end{proof}

\begin{lemma}
\label{lem:change-of-exp4}
Let $\beta, V > 2$, then we have
\[
\left.
\begin{array}{c}
	\int_\bR\frac{\exp\br{-\frac{x^2}2}}{1 + \exp\br{\beta V(V-x)}}\ dx\\
	\int_\bR\frac{\exp\br{-\frac{x^2}2}}{1 + \exp\br{\beta V(V-x)}}x\ dx
\end{array}
\right\} \leq \exp(-\Om{V^2})
\]
\end{lemma}
\begin{proof}
By completing squares we have
\[
\frac{\exp\br{-\frac{x^2}2}}{1 + \exp\br{\beta V(V-x)}} = \exp\br{-\frac{x^2}2}\frac{\exp\br{\frac{\beta x^2}{4}}}{\exp\br{\frac{\beta x^2}{4}} + \exp\br{\frac\beta4\br{x - 2V}^2}}
\]
Now, we consider two cases
\begin{enumerate}
	\item Case 1 $(x < \frac V2)$: In this case $\exp\br{\frac{\beta x^2}{4}} \leq \exp\br{\frac{\beta V^2}{16}}$ whereas $\exp\br{\frac\beta4\br{x - 2V}^2} \geq \exp\br{\frac{9\beta V^2}{16}}$. Thus, we have $\frac{\exp\br{\frac{\beta x^2}{4}}}{\exp\br{\frac{\beta x^2}{4}} + \exp\br{\frac\beta4\br{x - 2V}^2}} \leq \frac{\exp\br{\frac{\beta x^2}{4}}}{\exp\br{\frac\beta4\br{x - 2V}^2}} \leq \exp\br{-\frac{\beta V^2}2}$ in this region. Using standard Gaussian integrals we conclude that the region $\br{-\infty, \frac V2}$ contributes at most $\bigO{\exp\br{-\frac{\beta V^2}2}} \leq \exp(-\Om{V^2})$ (since $\beta > 2$) to both integrals.
	\item Case 2 $(x > \frac V2)$: In this case we simply bound $\frac{\exp\br{\frac{\beta x^2}{4}}}{\exp\br{\frac{\beta x^2}{4}} + \exp\br{\frac\beta4\br{x - 2V}^2}} \leq 1$ and use standard bounds on the complementary error function to conclude that the contribution of the region $\br{\frac V2, \infty}$ to both integrals is at most $\exp\br{-\Om{V^2}}$.
\end{enumerate}
\end{proof}

\begin{lemma}
\label{lem:subGaussian-cons}
Suppose $\vepsilon \sim \cN\br{\vzero, \frac1\betao\cdot I}$ and $\vv, \vDelta$ are fixed vectors such that $\sqrt\beta\norm\vDelta_2 \leq 1$. Then the random variable $X := \exp\br{-\frac\beta2\norm{\vepsilon - \vDelta}_2^2}\vepsilon^\top\vv$ has a subexponential constant at most
\[
\frac2{\sqrt{\beta + d}}\br{\sqrt{\frac\betao{\beta+\betao}}}^d\exp\br{-\frac{\beta\betao}{2(\beta+\betao)}\norm{\vDelta}_2^2}
\]
\end{lemma}
\begin{proof}
The subexponential constant of a random variable $X$ is defined as the value $\norm X_{\psi_1} := \sup_{p \geq 1}\frac1p\br{\E{\abs X^p}}^{1/p}$. In contrast, the subGaussian constant of a random variable $X$ is defined as the value $\norm X_{\psi_2} := \sup_{p \geq 1}\frac1{\sqrt p}\br{\E{\abs X^p}}^{1/p}$. Using Lemma~\ref{lem:change-of-exp} gives us
\begin{align*}
    \E{\abs X^p} &= \E{\exp\br{-\frac{p\beta}2\norm{\vepsilon - \vDelta}_2^2}(\vepsilon^\top\vv)^p}\\
    &= \br{\sqrt{\frac\betao{p\beta+\betao}}}^d\exp\br{-\frac{p\beta\betao}{2(p\beta+\betao)}\norm{\vDelta}_2^2}\E{(\vx^\top\vv)^p},
\end{align*}
where $\vx \sim \cN\br{\frac{p\beta}{p\beta+\betao}\vDelta, \frac1{p\beta+\betao}\cdot I}$. Now by virtue of $\vx$ being a Gaussian and using the triangle inequality for the subGaussian norm, we know that the random variable $\vx^\top\vv$ is $\br{\frac{p\beta}{p\beta+\betao}\norm\vDelta_2 + \frac1{\sqrt{p\beta+\betao}}}$-subGaussian. This, in turn implies that
\[
\br{\E{(\vx^\top\vv)^p}}^{1/p} \leq \sqrt p\cdot\br{\frac{p\beta}{p\beta+\betao}\norm\vDelta_2 + \frac1{\sqrt{p\beta+\betao}}} \leq \sqrt p\cdot\br{\frac{p\sqrt\beta}{p\beta+\betao} + \frac1{\sqrt{p\beta+\betao}}}
\]
Thus, we have
\begin{align*}
	\frac1p\br{\E{\abs X^p}}^{1/p} \leq \frac1{\sqrt p}\cdot\underbrace{\br{\sqrt{\frac\betao{p\beta+\betao}}}^{d/p}\exp\br{-\frac{\beta\betao}{2(p\beta+\betao)}\norm{\vDelta}_2^2}\br{\frac{p\sqrt\beta}{p\beta+\betao} + \frac1{\sqrt{p\beta+\betao}}}}_{(A)}
\end{align*}
Now, $(A)$ is a bounded function whereas $\frac1{\sqrt p}$ is a decreasing function. Thus, $\frac1p\br{\E{\abs X^p}}^{1/p}$ is a decreasing function of $p$ and hence achieves its maximum value in the range $p \geq 1$ at $p = 1$ itself. Noting that $\br{\frac{\sqrt\beta}{\beta+\betao} + \frac1{\sqrt{\beta+\betao}}} \leq \frac2{\sqrt{\beta + d}}$ finishes the proof.
\end{proof}
\section{Robust Regression}
\label{app:rr}

For this proof we use the notation $X = [\vx^1, \ldots, \vx^n] \in \bR^{d \times n}, \vy = [y_1, \ldots, y_n] \in \bR^n, \vb = [b_1,\ldots,b_n] \in \bR^n$. For any vector $\vv \in \bR^m$ and any set $T \subseteq [m]$, $\vv_T$ denotes the vector with all coordinates other than those in the set $T$ zeroed out. Similarly, for any matrix $A \in \bR^{k \times m}, A_T$ denotes the matrix with all columns other than those in the set $T$ zeroed out. We will let $G, B$ respectively denote the set of ``good'' uncorrupted points and ``bad'' corrupted points. We will abuse notation to let $G = (1-\alpha)\cdot n$ and $B = \alpha\cdot n$ respectively denote the number of good and bad points too.

\begin{theorem}[Theorem~\ref{thm:rr-main} restated -- Partially Adaptive Adversary]
\label{repthm:rr-main}
For data generated in the robust regression model as described in \S\ref{sec:me-rr-lr}, suppose corruptions are introduced by a partially adaptive adversary i.e. the locations of the corruptions (the set $B$) is not decided adversarially but the corruptions are decided jointly, adversarially and may be unbounded, then \gemrr enjoys a breakdown point of $0.1866$, i.e. it ensures a bounded $\bigO1$ error even if $k = \alpha\cdot n$ corruptions are introduced where the value of $\alpha$ can go upto at least $0.1866$. More generally, for corruption rates $\alpha \leq 0.1866$, there always exists values of scale increment $\step > 1$ s.t. with probability at least $1 - \exp(-\Om d)$, LWSC/LWLC conditions are satisfied for the $\tilde Q_\beta$ function corresponding to the robust least squares model for $\beta$ values at least as large as $\beta_{\max} = \bigO{\betao \min\bc{\frac1{\alpha^{2/3}}, \sqrt\frac n{d\log(n)}}}$.\\
\textbf{Hybrid Corruption Model}: If initialized with $\hvw^1, \beta^1$ s.t. $\beta_1\cdot\norm{\hvw^1 - \vwo}_2^2 \leq 1$, \gemrr assures
\[
\norm{\hvw^T - \vwo}_2^2 \leq \bigO{\frac1\betao\max\bc{\alpha^{2/3}, \sqrt\frac{d\log(n)}n}}
\]
within $T \leq \bigO{\log\frac n{\beta^1}}$ iterations for the \emph{hybrid corruption model} where even points uncorrupted by the adversary receive Gaussian noise with variance $\frac1\betao$ i.e. $y_i = \ip\vwo{\vx_i} + \epsilon_i$ for $i \in B$ where $\epsilon_i \sim \cN\br{0,\frac1\betao}$.\\
\textbf{Pure Corruption Model}:  For the \emph{pure corruption model} where uncorrupted points receive no Gaussian noise i.e. $y_i = \ip\vwo{\vx_i}$ for $i \in G$, \gemrr assures exact model recovery. Specifically, for any $\epsilon > 0$, \gemrr assures
\[
\norm{\hvw^T - \vwo}_2^2 \leq \epsilon
\]
within $T \leq \bigO{\log\frac1{\epsilon\beta^1}}$ iterations.
\end{theorem}
\begin{proof}
For any two models $\vv,\vw$, the $\tilde Q_\beta$ function for robust least squares has the following form
\[
\tilde Q_\beta(\vv\cond\vw) = \sum_{i=1}^n s_i\cdot\br{\ip\vw{\vx^i} - y_i}^2,
\]
where $s_i \leftarrow \exp\br{-\frac{\beta}2(y_i - \ip{\vx^i}{\vw})^2}$. We first outline the proof below.

\textit{Proof Outline.} This proof has three key elements
\begin{enumerate}
	\item We will establish the LWSC and LWLC properties for any fixed value of $\beta > 0$ with probability $1 - \exp(-\Om d)$. As promised in the statement of Theorem~\ref{thm:rr-main}, we will execute \gemrr for no more than $\bigO{\log n}$ iterations, taking a naive union bound would offer a confidence level of $1 - \log n\exp(-\Om d)$. However, this can be improved by noticing that the confidence levels offered by the LWSC/LWLC results are actually of the form $1 - \exp(-\Om{n\zeta^2 - d\log n})$. Thus, a union over $\bigO{\log n}$ such events will at best deteriorate the confidence bounds to $1 - \log n\exp(-\Om{n\zeta^2 - d\log n}) = 1 - \exp(-\Om{n\zeta^2 - d\log n - \log\log n})$ which is still $1 - \exp(-\Om d)$ for the values of $\zeta$ we shall set.
	\item The key to this proof is to maintain the invariant $\beta_t\cdot\norm{\hvw^t - \vwo}_2^2 \leq 1$. Recall that initialization ensures $\beta_1\cdot{\norm{\hvw^1-\vwo}_2^2}\leq 1$ to start things off. \S\ref{sec:gem} gives details on how to initialize in practice. This establishes the base case of an inductive argument. Next, iductively assuming that $\beta_t\cdot\norm{\hvw^t-\vwo}_2^2\leq 1$ for an iteration $t$, we will establish that $\norm{\hvw^{t+1} - \vwo}_2 \leq \frac{2\Lambda_{\beta_t}}{\lambda_{\beta_t}} \leq \frac{(A)}{\sqrt\beta_t}$ where $(A)$ will be an application-specific expression derived below.
	\item We will then ensure that $(A) < 1$, say $(A) = 1/\sqrt\step$ for some $\step > 1$, whenever the number of corruptions are below the breakdown point. This ensures $\norm{\hvw^{t+1} - \vwo}_2^2 \leq \frac1{{\step\beta_t}}$, in other words, ${\beta_{t+1}}\cdot{\norm{\hvw^{t+1}-\vwo}_2^2}\leq 1$ for $\beta_{t+1} = \step\cdot\beta_t$ so that the invariant is preserved. However, notice that the above step simultaneously ensures that $\frac{2\Lambda_{\beta_t}}{\lambda_{\beta_t}} \leq \frac1{\sqrt{\step\beta_t}}$. This ensures that a valid value of scale increment $\step$ can always be found till $\beta_t \leq \beta_{\max}$. Specifically, we will be able assure the existence of a scale increment $\step > 1$ satisfying the conditions of Theorem~\ref{thm:gem-main} w.r.t the LWSC/LWLC results only till $\beta < \bigO{\betao \min\bc{\frac1{\alpha^{2/3}}, \sqrt\frac n{d\log(n)}}}$.
\end{enumerate}

We now present the proof. Lemmata~\ref{lem:rr-lwsc},\ref{lem:rr-LWLC} establish the LWSC/LWLC properties for the $\tilde Q_\beta$ function for robust least squares regression. Let $\vDelta := \hvw^t - \vwo$. By Lemma~\ref{lem:ssc-sss}, with probability at least $1 - \exp(-\Om{n - d})$, we have $\norm{X_B}_2 = \sqrt{\lambda_{\max}(X_BX_B^\top)} \leq \sqrt{1.01B}$. The proof of Lemma~\ref{lem:rr-LWLC} tells us that with the same probability, we have
	\[
	\norm{S\vb}_2 \leq \sqrt{\frac{B}{2\pi}}[(\beta\norm{\vDelta}^21.01)^{1/3}+(\frac{1}{e})^{1/3}]^{\frac{3}{2}} \leq \sqrt{\frac{B}{2\pi}}[(\kappa^21.01)^{1/3}+(\frac{1}{e})^{1/3}]^{\frac{3}{2}}.
	\]
   On the other hand, the proof of Lemma~\ref{lem:rr-LWLC} also tells us us that with probability at least $1 - \exp\br{-\Om{n\nu^2 - d\log\frac1\nu - d\log(n)}}$ we have,
	\[
	\norm{XS\vepsilon}_2 = \norm{X_GS_G\vepsilon_G} \leq G(1+\nu)\sqrt\frac{\kappa^2}{2\pi}\frac\beta{\beta+\betao}\frac1\rcons.
	\]
	By Lemma~\ref{lem:rr-lwsc}, with probability at least $1 - \exp\br{-\Om{n\zeta^2 - d\log\frac1\zeta - d\log(n)}}$, we have $\lambda_{\min}(XSX^\top) \geq \lambda_{\min}(X_GS_GX_G^\top) \geq \sqrt\frac\beta{2\pi}(1-\zeta)\frac1\rcons\cdot G$. This gives us,
\begin{figure}[H]
\begin{adjustbox}{max width=\textwidth}
\parbox{\linewidth}{
\begin{align*}
		&\norm{\hvw^{t+1} - \vwo}_2	\leq \frac{B\sqrt{\frac{1.01}{2\pi}}[(1.01\kappa^2)^{1/3}+(\frac{1}{e})^{1/3}]^{\frac{3}{2}} + G(1+\nu)\sqrt\frac{\kappa^2}{2\pi}\frac\beta{\beta+\betao}\frac1\rcons}{\sqrt\frac\beta{2\pi}(1-\zeta)\frac1\rcons\cdot G}\\
		&= \frac\kappa{\sqrt\beta}\cdot\frac1{1-\zeta}\br{\frac\alpha{1-\alpha}\rcons\sqrt{1.01}\bs{(1.01)^{1/3}+\br{e\kappa^2}^{-1/3}}^{\frac32} + (1+\nu)\frac\beta{\beta+\betao}}\\
		&= \frac\kappa{\sqrt\beta}\cdot\frac1{1-\zeta}\br{\frac\alpha{1-\alpha}\sqrt\frac{\beta+\betao}\betao{\br{1+\frac{\beta\betao}{\beta+\betao}\norm{\vDelta}_2^2}^{3/2}}\sqrt{1.01}\bs{(1.01)^{1/3}+\br{e\kappa^2}^{-1/3}}^{\frac32} + (1+\nu)\frac\beta{\beta+\betao}}\\
		&\leq \frac\kappa{\sqrt\beta}\cdot\frac1{1-\zeta}\br{\frac\alpha{1-\alpha}\sqrt{1 + \frac\beta\betao}{\br{1+\frac{\betao\kappa^2}{\beta+\betao}}^{3/2}}\sqrt{1.01}\bs{(1.01)^{1/3}+\br{e\kappa^2}^{-1/3}}^{\frac32} + (1+\nu)\frac\beta{\beta+\betao}}\\
		&\leq \frac\kappa{\sqrt\beta}\cdot\underbrace{\frac1{1-\zeta}\br{\frac\alpha{1-\alpha}\sqrt{1 + \frac\beta\betao}{\br{1+\kappa^2}^{3/2}}\sqrt{1.01}\bs{(1.01)^{1/3}+\br{e\kappa^2}^{-1/3}}^{\frac32} + (1+\nu)\frac\beta{\beta+\betao}}}_{(A)}
	\end{align*}
}
\end{adjustbox}
\end{figure}
	where in the last two steps, we used that $\sqrt\beta\norm\vDelta_2 \leq \kappa$ and $\frac{\betao}{\beta+\betao} \leq 1$. We shall set $\kappa = 0.47$ below. Now, in order to satisfy the conditions of Theorem~\ref{thm:gem-main}, we need to assure the existence of a scale increment $\step > 1$ that assures a linear rate of convergence. Since \gemrr always maintains the invariant ${\beta_t}\cdot\norm{\hvw^t - \vwo}_2^2 \leq 1$ and $\kappa < 1$, assuring the existence of a scale increment $\step > 1$ is easily seen to be equivalent to showing that $(A) < 1$. This is done below and gives us our breakdown point.
	
\begin{enumerate}
		\item \textbf{Breakdown Point}: In order to ensure a linear rate of convergence, we need only ensure $(A) < 1$. If we set $\nu, \zeta$ to small constants (which we can always do for large enough $n$) and also set $\frac\beta\betao$ to a small constant (which still allows us to offer an error $\norm\vDelta_2 \leq \bigO{\frac1{\sqrt{\betao}}}$), then we can get $(A) < 1$ if
	\[
	\frac\alpha{1-\alpha}{\br{1+\kappa^2}^{3/2}}\sqrt{1.01}\bs{(1.01)^{1/3}+\br{e\kappa^2}^{-1/3}}^{\frac32} < 1
	\]
	Setting $\kappa = 0.47$ gives us a breakdown point of $\alpha \leq 0.1866$.
		\item \textbf{Consistency (Hybrid Corruption)}: For vanishing corruption i.e. $\alpha \rightarrow 0$, we can instead show a much stronger, consistent estimation guarantee. Suppose we promise that we would always set $\zeta, \nu \geq \frac1n$. Then, the results in Lemmata~\ref{lem:rr-lwsc} and \ref{lem:rr-LWLC} hold with probability at least $1 - \exp(-\Om d)$ even if we set $\zeta, \nu = \Om{\sqrt\frac{d\log(n)}n}$. Now, recall that we need to set $(A) < 1$ to expect a linear rate of convergence. Setting $\zeta = \nu$ to simplify notation (and also since both can be set to similar values without sacrificing $1 - \exp(-\Om d)$ confidence guarantees), using $\alpha \leq 0.5$, and using the shorthand $\rho := 1 + \frac\beta\betao$ gives us the requirement
\begin{align*}
			&(A) \leq 1 \Leftrightarrow\\
		   &\frac1{1-\zeta}\br{\frac\alpha{1-\alpha}\sqrt{1 + \frac\beta\betao}{\br{1+\kappa^2}^{3/2}}\sqrt{1.01}\bs{(1.01)^{1/3}+\br{e\kappa^2}^{-1/3}}^{\frac32} + (1+\nu)\frac\beta{\beta+\betao}} < 1\\
			&\Leftrightarrow \frac1{1-\nu}\br{9\alpha\sqrt\rho + (1+\nu)\frac{\rho-1}\rho} < 1
\end{align*}
The last requirement can be fulfilled for $\beta \leq \betao\cdot\min\bc{\Om{\frac1{\alpha^{2/3}}}, \Om{\sqrt\frac n{d\log(n)}}}$ which assures us of an error guarantee of $\norm{\vDelta}_2^2 \leq \bigO{\frac1\betao\max\bc{\alpha^{2/3}, \sqrt\frac{d\log(n)}n}}$ and finishes the proof for hybrid corruption case. Note that for no corruptions i.e. $\alpha = 0$, we do recover $\norm{\vDelta}_2^2 \leq \bigO{\sqrt\frac{d\log(n)}n}$.
	\item \textbf{Consistency (Pure Corruption)}: The pure corruption case corresponds to $\betao = \infty$. In this case, the above analysis shows that LWSC/LWLC properties continue to hold for arbitrarily large values of $\beta$ that assures arbitrarily accurate recovery of $\vwo$. Given the linear rate of convergence and the fact that \gemrr maintains the invariant $\sqrt{\beta_t}\cdot\norm{\hvw^t - \vwo}_2 \leq 1$, it is clear that for any $\epsilon > 0$, a model recovery error of $\norm{\hvw^T - \vwo}_2^2 \leq \epsilon$ is assured within $T \leq \bigO{\log\frac1{\epsilon\beta^1}}$ iterations.
\end{enumerate}
\end{proof}

\begin{lemma}[LWSC for Robust Least Squares Regression]
\label{lem:rr-lwsc}
For any $0 \leq \beta \leq n$, the $\tilde Q_\beta$-function for robust regression satisfies the LWSC property with constant $\lambda_\beta \geq Gc_\varepsilon(1-\zeta)$ with probability at least $1 - \exp(-d)$ for any $\zeta \geq \Om{\sqrt\frac{d\log(n)}n}$. In particular, for standard Gaussian covariates and Gaussian noise with variance $\frac1\betao$, we can take $c_\varepsilon \geq \frac1\rcons$ where $\rcons = \sqrt\frac{\beta+\betao}\betao{\br{1+\frac{\beta\betao}{\beta+\betao}\norm{\vDelta}_2^2}^{3/2}}$.
\end{lemma}
\begin{proof}
It is easy to see that $\nabla^2\tilde Q_\beta(\hvw\cond\vw) = XSX^\top$ for any $\hvw \in \cB_2\br{\vwo,\sqrt\frac1\beta}$.
Let $\vx \sim \cD, \epsilon \sim \cD_\varepsilon$ and let $y = \ip\vwo\vx + \epsilon$ be the response of an uncorrupted data point and $\vw \in \cB_2\br{\vwo,\frac{\kappa}{\sqrt{\beta}}}$ be any fixed model. Then if we let $\vDelta := \vwo-\vw$, then weight $\vs_i=\sqrt{\frac{\beta}{2\pi}}\exp(-\frac{\beta}{2}(\langle \vDelta,\vx_i\rangle+\epsilon_i)^2)$. 
	
	For any fixed $\vv \in S^{d-1}$, we have:

		\begin{align*}
		\E{\vv^\top X_GS_GX_G^\top\vv} = G\cdot\E{s_i \ip{\vx_i}\vv^2} 
		&= G\sqrt{\frac{\beta}{2\pi}}\cdot\Ee{\vx_i \sim \cD, \epsilon_i \sim {\cD}_\epsilon }{\ip{\vx_i}\vv^2 \exp(-\frac{\beta}{2}(\ip{\vDelta}{\vx_i}+\epsilon_i)^2)}\\ 
		&\geq  G\sqrt{\frac{\beta}{2\pi}} \cdot c(\beta, \vDelta, \sigma)
		\end{align*}
		where, 
		
		\[
		c_\varepsilon: = c(\beta, \vDelta, \sigma) := \inf_{\vv \in S^{d-1}}\bc{\Ee{\vx\sim \cD , \epsilon \sim {\cD}_\epsilon }{\ip{\vx}\vv^2\exp(-\frac{\beta}{2}(\ip{\vDelta}{\vx}+\epsilon)^2)}}
		\]
		
Similar to Lemma~\ref{lem:point-conv} we have: 
	\[
	\P{\lambda_{\min}(X_GS_GX_G^\top) < \br{1-\frac\zeta2} \, c_\varepsilon \, G\sqrt{\frac{\beta}{2 \pi}}} \leq 2\cdot9^d\exp\bs{-\frac{mn\zeta^2c_\varepsilon^2}{128R^4}}
	\]
	
Note that Lemma~\ref{lem:approx} continues to hold in this setting.	Proceeding as in the proof of Lemma~\ref{lem:wsc-wss} to set up a $\tau$-net over $\cB_2\br{\vwo,\frac{\kappa}{\sqrt{\beta}}}$ and taking a union bound over this net finishes the proof.

We now simplify $c_\varepsilon: = c(\beta, \vDelta, \sigma)$ for various distributions:
\begin{description}
	\item[Centered Isotropic Gaussian] For the special case of $\cD = \cN(\vzero,I_d)$, using rotational symmetry, we can w.l.o.g. take $\vDelta = (\Delta_1,0,0,\ldots,0)$ and $\vv = (v_1,v_2,0,0,\ldots,0), \, v_1^2 + v_2^2=1$. Thus, if $x_1,x_2 \sim \cN(0,1)$ i.i.d. then 
	\begin{align*}
	&\Ee{\vx\sim \cD , \epsilon \sim {\cD}_\epsilon }{\ip{\vx}\vv^2\exp(-\frac{\beta}{2}(\ip{\vDelta}{\vx}+\epsilon)^2)}\\
	&= \Ee{x_1,x_2\sim\cN(0,1), \epsilon \sim\cN(0,\sigma^2)}{(v_1^2x_1^2 + v_2^2x_2^2 + 2v_1v_2x_1x_2)\exp(-\frac{\beta}{2} (\Delta_1x_1+\epsilon)^2)}\\ 
	&= \Ee{x_1,x_2,\epsilon }{(v_1^2x_1^2 + v_2^2x_2^2)\exp(-\frac{\beta}{2} (\Delta_1x_1+\epsilon)^2)} \quad [ \text{as}, \Ee{}{x_2}=0] \\
	&= \Ee{x_1, \epsilon}{(v_1^2x_1^2 + v_2^2)\exp(-\frac{\beta}{2} (\Delta_1x_1+\epsilon)^2)}\quad [\text{as}, \Ee{}{x_2^2}=1]\\
	&= v_1^2\Ee{x_1, \epsilon}{x_1^2\exp(-\frac{\beta}{2} (\Delta_1x_1+\epsilon)^2)}+v_2^2\Ee{x_1, \epsilon}{\exp(-\frac{\beta}{2} (\Delta_1x_1+\epsilon)^2)}\\
	\end{align*}
	Using,
	\begin{align*}
	\frac{1}{\sqrt{2\pi\sigma^2}}\int\limits_{-\infty}^{\infty}\exp(-\frac{\beta}{2} (\Delta_1x_1+\epsilon)^2-\frac{\epsilon^2}{2\sigma^2})d\epsilon \quad
	= \frac{1}{\sqrt{1+\beta\sigma^2}}\exp(-\frac{\beta \Delta_1^2x_1^2 }{2(1+\beta \sigma^2)})
	\end{align*}
	
	We have, 
	\begin{figure}[H]
\begin{adjustbox}{max width=\textwidth}
\parbox{\linewidth}{
	\begin{align*}
	\Ee{x_1, \epsilon}{x_1^2\exp(-\frac{\beta}{2} (\Delta_1x_1+\epsilon)^2)}
	&= \frac{1}{\sqrt{2\pi}}\int\limits_{-\infty}^{\infty}\left[x_1^2 \exp(-\frac{x_1^2}{2})\frac{1}{\sqrt{2\pi\sigma^2}}\int\limits_{-\infty}^{\infty}\exp(-\frac{\beta}{2} (\Delta_1x_1+\epsilon)^2-\frac{\epsilon^2}{2\sigma^2})d\epsilon\right] \,dx_1\\
	&= \frac{1}{\sqrt{1+\beta\sigma^2}}\frac{1}{\sqrt{2\pi}}\int\limits_{-\infty}^{\infty}x_1^2 \exp(-\frac{x_1^2}{2}-\frac{\beta \Delta_1^2x_1^2 }{2(1+\beta \sigma^2)}) \,dx_1\\
	&=\frac{1}{\sqrt{1+\beta\sigma^2}}\left(\frac{1+\beta \sigma ^2 }{1+\beta(\sigma^2+\Delta_1^2)}\right)^\frac{3}{2}
	\end{align*}			
	}
\end{adjustbox}
\end{figure}
	and,
	
	\begin{align*}
	\Ee{x_1, \epsilon}{\exp(-\frac{\beta}{2} (\Delta_1x_1+\epsilon)^2)}
	&=\frac{1}{\sqrt{1+\beta\sigma^2}}\left(\frac{1+\beta \sigma ^2 }{1+\beta(\sigma^2+\Delta_1^2)}\right)^\frac{1}{2}
	\end{align*}

	This gives us 
	\begin{align*}
	c_\varepsilon &= \inf_{(v_1,v_2) \in S^1} \bc{\frac{v_1^2}{\sqrt{1+\beta\sigma^2}}\left(\frac{1+\beta \sigma ^2 }{1+\beta(\sigma^2+\Delta_1^2)}\right)^\frac{3}{2}+\frac{v_2^2}{\sqrt{1+\beta\sigma^2}}\left(\frac{1+\beta \sigma ^2 }{1+\beta(\sigma^2+\Delta_1^2)}\right)^\frac{1}{2}}\\
	&=\frac{1+\beta \sigma ^2 }{(1+\beta\sigma^2+\beta\norm{\vDelta}^2)^\frac{3}{2}}
	\qquad[\text{using, } v_1^2 +v_2^2=1 \text{ and } \norm{\vDelta} =\Delta_1^2]\\
	&= \frac{\frac{\betao +\beta}\betao}{\br{\frac{\betao +\beta}\betao+\beta\norm{\vDelta}^2}^\frac{3}{2}} = \frac1{\sqrt\frac{\betao +\beta}\betao\br{1+\frac{\beta\betao}{\beta+\betao}\norm{\vDelta}_2^2}^{3/2}} = \frac1\rcons
	\end{align*}
	which finishes the proof.
\end{description}
\end{proof}

\begin{lemma}[LWLC for Robust Least Squares Regression]
\label{lem:rr-LWLC}
For any $0 \leq \beta \leq n$, the $\tilde Q_\beta$-function for robust regression satisfies the LWLC property with constant $\Lambda_\beta \leq G(1 + \nu)V + 1.01B[(\frac{\beta\norm{\vDelta}^21.01}{2\pi})^{1/3}+(\frac{1}{2\pi e})^{1/3}]^{3/2}$ with probability at least $1 - \exp(-d)$ for any $\nu \geq \Om{\sqrt\frac{d\beta\log(n)}{n(\beta + d)}}$, where $V = \sqrt\frac{\kappa^2}{2\pi}\frac\beta{\beta+\betao}\frac1\rcons$ and $\rcons = \sqrt\frac{\beta+\betao}\betao{\br{1+\frac{\beta\betao}{\beta+\betao}\norm{\vDelta}_2^2}^{3/2}}$.
\end{lemma}

\begin{proof}
It is easy to see that $\nabla \tilde Q_\beta(\vwo\cond\vw) = X_GS_G\vepsilon_G + X_BS_B\vb$. We bound these separately below.

\paragraph{Weights on Bad Points.} Suppose we denote $\vDelta:= \vw - \vwo$ and let $S = \diag(\vs^t)$ be the weights assigned by the algorithm, then the analysis below shows that we must have
	\[
	\norm{XS\vb}_2 \leq 1.01B[(\frac{\beta\norm{\vDelta}^21.01}{2\pi})^{1/3}+(\frac{1}{2\pi e})^{1/3}]^{3/2}
	\]
	We will bound $\norm{S\vb}_2$ below and use Lemma~\ref{lem:ssc-sss} to get the above bound. Let $b_i$ denote the corruption on the data point $\vx_i$, i.e. $y_i=\langle\vwo,\vx_i\rangle+b_i$. The proof proceeds via a simple case analysis:
	\begin{description}
		\item[Case 1: $\abs{b_i} \leq k\abs{\langle \vDelta,\vx_i \rangle}$] In this case we simply bound $(s_ib_i)^2 \leq \frac{\beta}{2\pi}b_i^2 \leq \frac{\beta}{2\pi}k^2\langle \vDelta,\vx_i \rangle^2$.
		\item[Case 2: $\abs{b_i} > k\abs{\langle \vDelta,\vx_i \rangle}$] 
				or, $-\abs{\langle \vDelta,\vx_i \rangle}>-\frac{\abs{b_i}}{k}$ or,$\abs{b_i}-\abs{\langle \vDelta,\vx_i \rangle}>\abs{b_i}-\frac{\abs{b_i}}{k}$

		\begin{equation*}
		\abs{b_i-\langle \vDelta,\vx_i \rangle}\geq \abs{b_i}- \abs{\langle \vDelta,\vx_i \rangle}> \abs{b_i}\frac{k-1}{k}
		\end{equation*}
		
		So that,
				\begin{align*}
				b_i^2s_i^2&=b_i^2\mathcal{N}(y_i|\langle \vw,\mathbf{x}_i \rangle, \frac{1}{\beta})^2\\
				&=b_i^2\frac{\beta}{2\pi}\exp(-\frac{\beta}{2}(y_i-\langle \vw,\vx_i \rangle)^2)^2\\
				&=b_i^2\frac{\beta}{2\pi}\exp(-\beta(b_i-\langle \vDelta,\vx_i \rangle)^2)\\
				&\leq b_i^2\frac{\beta}{2\pi}\exp(-\beta b_i^2\frac{(k-1)^2}{k^2}) \qquad \text{for, } k\geq 1\\
				&=\frac{k^2}{2\pi(k-1)^2}z\exp(-z) \qquad \text{for, } z=	\beta b_i^2\frac{(k-1)^2}{k^2}\\			
				&\leq \frac{k^2}{2\pi e(k-1)^2} \qquad \text{as, } \max_{z}\{z\exp(-z)\}=\frac{1}{e}
				\end{align*}
		
		Combining case 1 and 2,
		
		\begin{align*}
		\norm{S\vb}_2^2&=\sum_{i\in B}(s_ib_i)^2\leq \sum_{i\in B}\max\{\frac{\beta}{2\pi}k^2\langle \vDelta,\vx_i \rangle^2, \frac{k^2}{2\pi e(k-1)^2}\} \qquad \text{where, }k\geq 1\\
		&\leq \frac{\beta}{2\pi}k^2\sum_{i\in B}\langle \vDelta,\vx_i \rangle^2+ \frac{B}{2\pi e}\frac{k^2}{(k-1)^2}\\
		&\leq \frac{\beta\norm{\vDelta}^2\lambda_{max}(X_BX_B^T)}{2\pi}k^2+\frac{B}{2\pi e}\frac{k^2}{(k-1)^2}\\
		&=qk^2+p\frac{k^2}{(k-1)^2}\qquad \text{where, } p = \frac{B}{2 \pi e},\quad q=\frac{\beta\norm{\vDelta}^2\lambda_{max}(X_BX_B^T)}{2\pi}\\
		\end{align*} 
		
		Let, 
		\begin{align*}
		&g(k)=qk^2+\frac{pk^2}{(k-1)^2}; g'(k)=2qk + \frac{2pk}{(k-1)^2}-\frac{2pk}{(k-1)^3}=0\\
		&\implies k=1+(\frac{p}{q})^{1/3} \implies \min_{k}g(k)=(q^{1/3}+p^{1/3})^3
		\end{align*}
This gives us
		\begin{align*}
		\norm{S\vb}_2^2&\leq[(\frac{\beta\norm{\vDelta}^2\lambda_{max}(X_BX_B^T)}{2\pi})^{1/3}+(\frac{B}{2\pi e})^{1/3}]^3\\
		&\leq B[(\frac{\beta\norm{\vDelta}^21.01}{2\pi})^{1/3}+(\frac{1}{2\pi e})^{1/3}]^3\\
		\end{align*}
	\end{description}

\paragraph{Bounding the Weights on Good Points.} Given noise values sampled from $\cD_\varepsilon = \cN(0,\frac1\betao)$ and corrupted points $B$ are choosen independent of $\vx_i$, then for any $\beta > 0$ and $S= \text{diag}(\vs)$,  $\vs$ computed w.r.t. model $\vw$ at variance $\frac{1}{\beta}$, then the following analysis shows that
	\[
	\P{\norm{X_GS_G\vepsilon_G}_2 > G(1 + \nu)V} \leq \br{\frac{12R_X}{V\nu}}^{3d}\exp\br{-\frac{m\betao\nu^2V^2\cdot n}{32}},
	\]
	where $V = \sqrt\frac{\kappa^2}{2\pi}\frac\beta{\beta+\betao}\frac1\rcons$ and $\rcons = \sqrt\frac{\beta+\betao}\betao{\br{1+\frac{\beta\betao}{\beta+\betao}\norm{\vDelta}_2^2}^{3/2}}$.
Suppose $\epsilon \sim \cN(0,\frac1\betao)$ and $\vx \sim \cN(\vzero, I)$. Then, for any fixed error vector $\vDelta \in \cB_2\br{\frac\kappa{\sqrt\beta}}$, if we set $s = \sqrt\frac\beta{2\pi}\exp\br{-\frac\beta2(\epsilon - \vDelta\cdot\vx)^2}$, then we can analyze the vector $\E{\epsilon s\cdot\vx}$ as
\begin{align*}
\E{\epsilon s\cdot\vx} = \Ee{\vx}{\Ee{\epsilon}{\epsilon s}\cdot\vx} = \sqrt\frac\beta{2\pi}\Ee{\vx}{\Ee{\epsilon}{\exp\br{-\frac\beta2(\epsilon - \vDelta^\top\vx)^2}\epsilon}\cdot\vx}\\
= \sqrt\frac\beta{2\pi}\sqrt\frac\betao{\beta+\betao}\frac\beta{\beta+\betao}\underbrace{\Ee{\vx}{\br{\exp\br{-\frac{\beta\betao}{2(\beta+\betao)}(\vDelta^\top\vx)^2}\vDelta^\top\vx}\cdot\vx}}_\vp
\end{align*}
where in the last step, we used Lemma~\ref{lem:change-of-exp} in one dimensions. Now, by rotational symmetry of the Gaussian distribution and unbiased and independent nature of its coordinates, we can assume w.l.o.g. that the error vector is of the form $\vDelta = (\delta, 0, \ldots, 0)$ and conclude, as we did in the Gaussian mixture model analysis, that the vector $\vp$ in this situation also must have only its first coordinate nonzero. Thus, we have, for 
\[
	\abs{\vp_1} = \abs{\E{\exp\br{-\frac{\beta\betao}{2(\beta+\betao)}\delta^2x^2}\delta x^2}},
\]
where $x \sim \cN(0,1)$ since we had $\vx \sim \cN(\vzero, I)$. Applying Lemma~\ref{lem:change-of-exp} in single dimensions yet again and the Cauchy-Schwartz inequality gives us, for any unit vector $\vv$,
\[
\E{\epsilon s\cdot\vx^\top\vv} \leq \sqrt\frac\beta{2\pi}\sqrt\frac\betao{\beta+\betao}\frac\beta{\beta+\betao}\frac1{\br{1+\frac{\beta\betao}{\beta+\betao}\norm{\vDelta}_2^2}^{3/2}}\norm\vDelta_2 \leq \sqrt\frac{\kappa^2}{2\pi}\frac\beta{\beta+\betao}\frac1\rcons,
\]
where $\rcons = \sqrt\frac{\beta+\betao}\betao{\br{1+\frac{\beta\betao}{\beta+\betao}\norm{\vDelta}_2^2}^{3/2}}$ and in the last step we used $\beta\norm{\vDelta}_2^2 \leq \kappa^2$. The above, when combined with uniform convergence arguments over appropriate nets over the unit vectors $\vv$ and the error vector $\vDelta$ gives us the claimed result.

\end{proof}

\begin{lemma}
	\label{lem:point-conv}
	With the same preconditions as in Lemma~\ref{lem:point-exp}, we have
	\[
	\left.
	\begin{array}{r}
	\P{\lambda_{\min}(X_GS_GX_G^\top) < (1-\zeta) c\cdot G\sqrt{\frac{\beta}{2\pi}}}\\
	\vspace*{-2ex}\\
	\P{\lambda_{\max}(X_GS_GX_G^\top) > (1+\zeta) \cdot G\sqrt{\frac{\beta}{2\pi}}}
	\end{array}
	\right\}
	\leq 2\cdot9^d\exp\bs{-\frac{mnc^2\zeta^2}{32R^4}},
	\]
	where $R$ is the subGaussian constant of the distribution $\cD$ that generated the covariate vectors $\vx^i$. When $\cD$ is the standard Gaussian i.e. $\cN(\vzero, I)$, we have $R \leq 1$. For $\cN(\vzero,\frac1\betao\cdot I)$, we have $R \leq \sqrt\frac1\betao$. In the above, $c$ is a constant that depends only on the distribution $\cD$ and is bounded for various distributions in Lemmata~\ref{lem:point-exp}.
\end{lemma}
\begin{proof}
	Let $A \in \bR^{d \times d}$ be a square symmetric matrix, for $\delta > 0$, we have:
	\begin{align*}
	\norm{A - c\cdot I}_2 \leq \delta \iff \forall \vv \in S^{d-1} , \abs{\vv^\top A\vv - c} \leq \delta   \iff c - \delta \leq \lambda_{\min}(A) \leq \lambda_{\max}(A) \leq c + \delta
	\end{align*}
	
	Also, for any square symmetric matrix $F \in \bR^{d\times d}$ and $\cN_\epsilon$ being $\epsilon$ net over $S^{d-1}$
	\begin{align*}
	\norm{F}_2 \leq (1-2\epsilon)^{-1}\sup_{\vv \in \cN_\epsilon}\abs{\vv^\top F \vv}
	\end{align*}
	
	Taking $F = A - c\cdot I$ and  $\epsilon = 1/4$, we have 
	\begin{align}
	\label{net}
		\norm{A - c\cdot I}_2 \leq 2\sup_{\vv \in \cN_{1/4}}\abs{\vv^\top A\vv - c}
	\end{align}

	let $Z_i := \sqrt{s_i} \cdot\ip{\vx_i}{\vv}$ and $\vx \sim \cD$ then for any fixed $\vv \in S^{d-1}$, we have
	\[
	\norm{Z_i}_{\psi_2} = \sup_{p \geq 1} \frac{\br{\E{\abs{Z_i}^p}}^{1/p}}{\sqrt{p}} \leq (\frac{\beta}{2\pi})^{1/4}\cdot\sup_{p \geq 1} \frac{\br{\E{\abs{\ip{\vx_i}{\vv}}^p}}^{1/p}}{\sqrt{p}} = (\frac{\beta}{2\pi})^{1/4}R
	\]
	where we use the fact that $\norm{\ip{\vx_i}{\vv}}_{\Psi_2} \leq R$ since $\cD$ is $R$-sub-Gaussian and $\sqrt{s_i} \leq (\frac{\beta}{2 \pi})^{1/4}$.  
	
	Also, $Z$ is $R(\frac{\beta}{2\pi})^{1/4}$-sub-Gaussian, implies $Z^2$ is $R^2\sqrt{\frac{\beta}{2\pi}}$ sub-exponential, as well as $Z^2 - \bE Z^2$ is $2R^2\sqrt{\frac{\beta}{2\pi}}$-sub-exponential, using centering. And for a single good data point Lemma~\ref{lem:point-exp} gives  $\mu := \bE Z_i^2 \in [c\sqrt{\frac{\beta}{2\pi}},\sqrt{\frac{\beta}{2\pi}}]$.
	\begin{align*}
	&\P{\abs{\vv^\top X_GS_GX_G^\top\vv - G\mu} \geq \varepsilon\cdot G\sqrt{\frac{\beta}{2\pi}}}\\
	&= \P{\abs{\sum_{i \in G}(Z^2_i - \mu)} \geq \varepsilon\cdot G\sqrt{\frac{\beta}{2\pi}}}\\
	&\leq 2\exp\bs{-m\cdot\min\bc{\frac{ \left( \varepsilon\cdot G\sqrt{\frac{\beta}{2\pi}} \right) ^2}{G\left( 2R^2\sqrt{\frac{\beta}{2\pi}} \right)^2},\frac{\varepsilon\cdot G\sqrt{\frac{\beta}{2\pi}}}{2R^2\sqrt{\frac{\beta}{2\pi}}}}} \text{[Theorem 2.8.2]}\\ 
	&\leq 2\exp\bs{-\frac{mn\varepsilon^2}{8R^4}}
	\end{align*}
	where $m > 0$ is a universal constant and in the last step we used $G \geq n/2$ and w.l.o.g. we assumed that $\varepsilon \leq 2R^2$. Taking a union bound over all $9^d$ elements of $\cN_{1/4}$, we get
	\begin{align*}
	\P{\norm{X_GS_GX_G^\top - G\mu\cdot I}_2 \geq \varepsilon\cdot G\sqrt{\frac{\beta}{2\pi}}} &\leq \P{2\sup_{\vv\in\cN_{1/4}}\abs{\vv^\top X_GS_GX_G^\top\vv - G\mu} \geq \varepsilon\cdot G\sqrt{\frac{\beta}{2\pi}}} \text{using \ref{net}}\\
	&\leq 2\cdot9^d\exp\bs{-\frac{mn\varepsilon^2}{32R^4}}
	\end{align*}
	Setting $\varepsilon = \zeta c$ and noticing that $\mu \in [c\sqrt{\frac{\beta}{2\pi}},\sqrt{\frac{\beta}{2\pi}}]$ by Lemma~\ref{lem:point-exp} finishes the proof.
\end{proof}

\begin{lemma}
	\label{lem:approx}
	Consider two models $\vw^1,\vw^2 \in \bR^d$ such that $\norm{\vw^1 - \vw^2}_2 \leq \tau$ and let $\vs^1,\vs^2$ denote the corresponding weight vectors, i.e. $s_i^j=\sqrt{\frac{\beta}{2\pi}}\exp(-\frac{\beta}{2}(y_i-\langle \vw^j,\vx_i\rangle)^2),\, j = 1,2$. Also let $S^1 = \diag(\vs^1)$ and $S^2 = \diag(\vs^2)$. Then for any $X = [\vx_1,\ldots,\vx_n] \in \bR^{d\times n}$ such that $\norm{\vx_i}_2 \leq R_X$ for all $i$,
	\[
	\abs{\lambda_{\min}(XS^1X^\top) - \lambda_{\min}(XS^2X^\top)} \leq \frac{n\tau \beta R_X^3}{\sqrt{2\pi e}},
	\]
	where $R_X$ is the maximum length in a set of $n$ vectors, each sampled from a $d$-dimensional Gaussian (see Lemma~\ref{lem:rx-bound}).
\end{lemma}
\begin{proof}
	
	Let, $s_i^j=f(r_i^j)=\sqrt{\frac{\beta}{2\pi}}\exp(-\frac{\beta}{2}(y_i-r_i^j)^2)$ where, $r_i^j=\langle \vw^j,\vx_i\rangle, \, j = 1,2$
	
	Since, $f:\bR \rightarrow \bR$, is a everywhere differentiable function, it is $L$-lipschitz continuous where $L=\sup\limits_{r}\abs{f'(r)}$ 
	
	\begin{align*}
	f'(r)&=\sqrt{\frac{\beta}{2\pi}}\exp(-\frac{\beta}{2}(y_i-r)^2)\beta(y_i-r)\\
	&=\sqrt{\frac{\beta}{2\pi}}t\exp(-t^2)\sqrt{2\beta} && where, t=\sqrt{\frac{\beta}{2}}(y_i-r)\\
	&\leq \frac{\beta}{\sqrt{\pi}}\frac{1}{\sqrt{2e}} &&  t\exp(-t^2)\leq \frac{1}{\sqrt{2e}}\\
	&=\frac{\beta}{\sqrt{2\pi e}}
	\end{align*}
	
	Hence,	$\frac{\abs{f(r_i^1)-f(r_i^2)}}{\abs{r_i^1-r_i^2}}\leq \frac{\beta}{\sqrt{2\pi e}}$ or $\abs{s_i^1-s_i^2}\leq \frac{\beta}{\sqrt{2\pi e}}\abs{\langle \vw^1-\vw^2,\vx_i\rangle}=\frac{\beta\tau R_X}{\sqrt{2\pi e}}$. This gives us $\norm{\vs^1 - \vs^2}_1 \leq \frac{n\tau \beta R_X}{\sqrt{2\pi e}}$.	Now, if we let $S^1 = \diag(\vs^1)$ and $S^2 = \diag(\vs^2)$, then for any unit vector $\vv \in S^{d-1}$, denoting $R_X := \max_{i \in [n]}\ \norm{\vx_i}_2$ we have
	\begin{align*}
	    \abs{\vv^\top XS^1X^\top\vv - \vv^\top XS^2X^\top\vv} &= \abs{\sum_{i=1}^n\br{\vs^1_i-\vs^2_i}\ip{\vx_i}{\vv}^2}\\
	    &\leq \norm{\vs^1 - \vs^2}_1\cdot\max_{i\in[n]}\ \ip{\vx_i}{\vv}^2\\
	    &\leq \norm{\vs^1 - \vs^2}_1\cdot R_X^2\\
	    &\leq \frac{n\tau \beta R_X^3}{\sqrt{2\pi e}}.
	\end{align*}
	This proves that $\norm{XS^1X^\top - XS^2X^\top}_2 \leq \frac{n\tau \beta R_X^3}{\sqrt{2\pi e}}$.
	
\end{proof}

\begin{lemma}
	\label{lem:point-exp}
	Let $X = [\vx_1,\ldots,\vx_n] \in \bR^{d\times n}$  generated from an isotropic $R$-sub-Gaussian distribution $\cD$, for any fixed model $\vw$ and $\beta > 0$, let $s_i=\sqrt{\frac{\beta}{2\pi}}\exp(-\frac{\beta}{2}(y_i-\langle \vw,\vx_i\rangle)^2)$ be the weight of the data point $\vx_i$, $S_G = \diag(\vs_G)$, then there exists a constant $c > 0$ that depends only on $\cD$ such that for any fixed vector unit $\vv \in S^{d-1}$,
	\[
	c\cdot G\sqrt{\frac{\beta}{2\pi}} \leq \E{\vv^\top X_GS_GX_G^\top\vv} \leq G\sqrt{\frac{\beta}{2\pi}}.
	\]
\end{lemma}
\begin{proof}
	Let $\vx_i \sim \cD$ and for good points $y_i = \ip{\vwo}{\vx_i}$. Let $\vDelta := \vwo-\vw $, as $s_i \leq \sqrt{\frac{\beta}{2\pi}}$, we have by linearity of expectation,

	\[
	\E{\vv^\top X_GS_GX_G^\top\vv} = \E{\sum_{i\in G}s_i\ip{\vx_i}{\vv}^2} = G\cdot\E{s_i\ip{\vx_i}\vv^2} \leq G\sqrt{\frac{\beta}{2\pi}}\cdot\E{\ip{\vx_i}\vv^2} = G\sqrt{\frac{\beta}{2\pi}},
	\]
	since $\cD$ is isotropic. 
	
	For the lower bound, we may write,
	
	\begin{align*}
	\E{\vv^\top X_GS_GX_G^\top\vv} = G\cdot\E{\ip{\vx_i}\vv^2 \cdot s_i} 
	&= G\sqrt{\frac{\beta}{2\pi}}\cdot\E{\ip{\vx_i}\vv^2 \exp(-\frac{\beta}{2}\ip{\vDelta}{\vx_i}^2)}\\ 
	&\geq  G\sqrt{\frac{\beta}{2\pi}} \cdot c(\beta, \vDelta)
	\end{align*}
	
	where, for any distribution $\cD$ over $\bR^d$, we define the constant $c$ as
	\[
	c(\beta, \vDelta) := \inf_{\vv \in S^{d-1}}\bc{\Ee{\vx\sim \cD}{\ip{\vx}\vv^2\exp(-\frac{\beta}{2}\ip{\vDelta}{\vx}^2)}}
	\]
	\begin{description}

		\item[Centered Isotropic Gaussian] For the special case of $\cD = \cN(\vzero,I_d)$, using rotational symmetry, we can w.l.o.g. take $\vDelta = (\Delta_1,0,0,\ldots,0)$ and $\vv = (v_1,v_2,0,0,\ldots,0)$. Thus, if $x_1,x_2 \sim \cN(0,1)$ i.i.d. then 
		\begin{figure}[H]
\begin{adjustbox}{max width=\textwidth}
\parbox{\linewidth}{
		\begin{align*}
		\Ee{\vx\sim \cD}{\ip{\vx}\vv^2\exp(-\frac{\beta}{2}\ip{\vDelta}{\vx}^2)} &= \Ee{x_1,x_2\sim\cN(0,1)}{(v_1^2x_1^2 + v_2^2x_2^2 + 2v_1v_2x_1x_2)\exp(-\frac{\beta}{2} \Delta_1^2x_1^2)}\\ 
		&= \Ee{x_1,x_2\sim\cN(0,1)}{(v_1^2x_1^2 + v_2^2x_2^2)\exp(-\frac{\beta}{2} \Delta_1^2x_1^2)} \quad [ \text{as}, \Ee{}{x_2}=0] \\
		&= \Ee{x_1\sim\cN(0,1)}{(v_1^2x_1^2 + v_2^2)\exp(-\frac{\beta}{2} \Delta_1^2x_1^2)}\quad [\text{as}, \Ee{}{x_2^2}=1]\\
		&= v_1^2\Ee{x_1\sim\cN(0,1)}{x_1^2\exp(-\frac{\beta}{2} \Delta_1^2x_1^2)}+v_2^2\Ee{x_1\sim\cN(0,1)}{\exp(-\frac{\beta}{2} \Delta_1^2x_1^2)}\\
		&= \frac{v_1^2}{(1+\beta \Delta_1^2)^{3/2}}+\frac{v_2^2}{(1+\beta \Delta_1^2)^{1/2}}\\
		\end{align*}
		}
\end{adjustbox}
\end{figure}
		This gives us 
		\begin{align*}
		c(\beta, \vDelta) &= \inf_{(v_1,v_2) \in S^1} \bc{\frac{v_1^2}{(1+\beta \norm{\vDelta}^2)^{3/2}}+\frac{v_2^2}{(1+\beta \norm{\vDelta}^2)^{1/2}}}\geq \frac{1}{(1+\beta \norm{\vDelta}^2)^{3/2}}\\ 
		& \geq \frac{1}{(1+ \kappa^2)^{3/2}} \qquad[\text{using} \sqrt{\beta} \norm{\vDelta} < \kappa]
		\end{align*}
		\item[Centered Non-isotropic Gaussian] For the case $\cD=\cN(\vzero,\Sigma)$, we have $\vx\sim \cD=\Sigma^{1/2}.\cN(\vzero,I_d)$. Thus for any fixed unit vector $\vv$, we have $\ip{\vv}{\vx}\sim\ip{\tilde{\vv}}{\vz}$ where $\tilde{\vv}=\Sigma^{-1/2}\vv$ and $\vz\sim \cN(\vzero,I_d)$. We also have $\norm{\tilde{\vv}}_2\in\bs{\frac{1}{\sqrt{\Lambda}},\frac{1}{\sqrt{\lambda}}}$, where $\Lambda=\lambda_{max}(\Sigma)$ and $\lambda=\lambda_{min}(\Sigma)$. Now for any fixed vectors $\vDelta,\vv$ we first perform rotations so that we have $\tilde{\vDelta}=(\Delta,0,0,\cdots,0)$ and $\tilde{\vv}=(v_1,v_2,0,0,\cdots,0)\in\mathbb{B}(\vzero,r)$ where $\Delta\in \bs{\frac{\norm{\vDelta}}{\sqrt{\Lambda}},\frac{\norm{\vDelta}}{\sqrt{\lambda}}}$ and $r\in\bs{\frac{1}{\sqrt{\Lambda}},\frac{1}{\sqrt{\lambda}}}$. This gives us $c(\beta,\vDelta)\geq \inf_{v_1,v_2}f(v_1,v_2)$ where,
	\begin{align*}
	f(v_1,v_2)
	&= \Ee{x_1,x_2\sim\cN(0,1)}{(v_1^2x_1^2 + v_2^2x_2^2 + 2v_1v_2x_1x_2)\exp(-\frac{\beta}{2} \Delta^2x_1^2)}\\
	&= \frac{v_1^2}{(1+\beta \Delta^2)^{3/2}}+\frac{v_2^2}{(1+\beta \Delta^2)^{1/2}}\\
	\end{align*}
	similar to that of isotropic counterpart; giving the following,
	\begin{align*}
		c(\beta, \vDelta) &= \inf_{(v_1,v_2)\in \mathbb{B}(\vzero,r)} \bc{\frac{v_1^2}{(1+\beta \Delta^2)^{3/2}}+\frac{v_2^2}{(1+\beta \Delta^2)^{1/2}}}\geq \frac{1}{\Lambda(1+\frac{\beta}{\lambda} \norm{\vDelta}^2)^{3/2}}
		\end{align*}
	As the above term needs to be bounded away from 0, we require $\lambda=\lambda_{min}(\Sigma)>0$ i.e. reasonably away from 0.

		\item[Non-centered isotropic Gaussian] Suppose the covariates are generated from a distribution $\cD=\cN(\vmu,I_d)$. As earlier, by rotational symmetry, we can take $\vDelta=(\norm{\vDelta},0,0,\cdots,0)$, $\vv=(v_1,v_2,0,\ldots,0)$, $\vmu=(\mu_1,\mu_2,\mu_3,0,0,\ldots,0)$. Assume $\norm{\vmu}_2=\rho$. Letting $\ip{\vmu}{\vv}=:p\leq \rho$ and $x_1,x_2,x_3\sim \cN(0,1)$ i.i.d. gives $c(\beta,\vDelta)\geq \inf_{v_1,v_2}f(v_1,v_2)$ independence of $x_1,x_2,x_3$ and the fact that $\E{x_2}=0$ and $\E{x_2^2}=1$ gives us,
		\begin{figure}[H]
\begin{adjustbox}{max width=\textwidth}
\parbox{\linewidth}{
		\begin{align*}
		f(v_1,v_2)
		&=\Ee{x_1,x_2\sim \cN(0,1)}{(p+v_1x_1+v_2x_2)^2\exp\br{-\frac{\beta\norm{\vDelta}^2}{2}(x_1+\mu_1)^2}}\\
		&=\Ee{x_1,x_2\sim \cN(0,1)}{((p+v_1x_1)^2+v_2^2x_2^2+2(p+v_1x_1)v_2x_2)\exp\br{-\frac{\beta\norm{\vDelta}^2}{2}(x_1+\mu_1)^2}}\\
		&=\Ee{x_1\sim \cN(0,1)}{((p+v_1x_1)^2+v_2^2)\exp\br{-\frac{\beta\norm{\vDelta}^2}{2}(x_1+\mu_1)^2}}
		\end{align*}
		}
\end{adjustbox}
\end{figure}
		Now, since $(v_1,v_2)\in S^1$ we have the following two cases:\\
		
		\textbf{Case 1:} $v_2^2\geq \frac{1}{2}$. In this case 
		\begin{align*}
			f(v_1,v_2) &\geq \frac{1}{2}\Ee{x_1\sim \cN(0,1)}{\exp\br{-\frac{\beta\norm{\vDelta}^2}{2}(x_1+\mu_1)^2}}\\
			& =\frac{1}{2\sqrt{2\pi}}\int\limits_{-\infty}^{\infty}\exp\br{-\frac{c}{2}(x+\mu_1)^2-\frac{x^2}{2}}dx \text{ ,where }c=\beta\norm{\vDelta}^2\\
			\text{if }\mu_1>0\\
			&\geq \frac{1}{2\sqrt{2\pi}}\int\limits_{-\infty}^{\infty}\exp\br{-\frac{c+1}{2}(x+\mu_1)^2}dx
			=\frac{1}{2\sqrt{c+1}}\geq \frac{1}{2\sqrt{\kappa^2+1}}=0.45\\
			\text{else if }\mu_1<0\\
			&\geq \frac{1}{2\sqrt{2\pi}}\int\limits_{-\infty}^{\infty}\exp\br{-\frac{c+1}{2}x^2}dx
			=\frac{1}{2\sqrt{c+1}}\geq \frac{1}{2\sqrt{\kappa^2+1}}=0.45
		\end{align*}
		for $c=\beta\norm{\vDelta}^2\leq \kappa^2$ and $\kappa=0.47$.\\
		
		\textbf{Case 2:} $v_1^2\geq \frac{1}{2}$. In this case, if $x_1\geq 2\sqrt{2}\rho$, then $\abs{v_1x_1+p}\geq \rho$ and also $\abs{v_1x_1+p}\geq \frac{x_1}{2\sqrt{2}}$. So we can write $(v_1x_1+p)^2\geq \frac{\rho x_1}{2\sqrt{2}}$. Also $\abs{x_1+\mu_1}\leq 2x_1$. Hence,
		\begin{align*}
		f(v_1,v_2) 
		&\geq \frac{\rho}{2\sqrt{2}}\Ee{x_1\sim \cN(0,1)}{x_1\exp\br{-\frac{c}{2}(2x_1)^2}\ind{x_1\geq 2\sqrt{2}\rho}}\\
		&=\frac{\rho}{4\sqrt{\pi}}\int\limits_{2\sqrt{2}\rho}^{\infty}x\exp\br{-\frac{c}{2}(2x)^2-\frac{x^2}{2}}dx\\
		&=\frac{\rho}{4\sqrt{\pi}}\int\limits_{2\sqrt{2}\rho}^{\infty}x\exp\br{-(\frac{1}{2}+2c)x^2}dx\\
		&=\frac{\rho}{4\sqrt{\pi}(1+4c)}\int\limits_{2\sqrt{2}\rho}^{\infty}(1+4c)x .\exp\br{-\frac{1+4c}{2}x^2}dx\\
		&=\frac{\rho}{4\sqrt{\pi}(1+4c)}\int\limits_{4(1+4c)\rho^2}^{\infty}exp(-z)dz\\
		&=\frac{\rho}{4\sqrt{\pi}(1+4c)} e^{-4(1+4c)\rho^2}\\
		&\geq \frac{\rho}{4\sqrt{\pi}(1+4\kappa^2)} e^{-4(1+4\kappa^2)\rho^2}
		\geq \frac{\rho}{14} e^{-8\rho^2}
		\end{align*}
		using the fact $c=\beta\norm{\vDelta}^2\leq \kappa^2$ and $\kappa=0.47$. We see from above that we need to avoid large $\rho$ value. One way to avoid this is to center the covariates i.e. use $\tilde{x_i}=x_i-\hat{\mu}$ where $\hat{\mu}:=\frac{1}{n}\sum_{i=1}^n x_i$. This would approximately center the covariates and ensure an effective value of $\rho\approx \bigO{\sqrt{\frac{d}{n}}}$.

	\item[Bounded Distribution] Suppose our covariate distribution has bounded support that is, we have $supp(\cD)\subset \mathcal{B}_2(\rho)$ for some $
	\rho>0$. Assume $\rho>1$ w.l.o.g. Also using the centering trick above, assume that $\Ee{\vx\sim\cD}{\vx}=\vzero$. Then we have $\ip{\vDelta}{\vx}\leq \norm{\vDelta}\rho$ which implies $\exp\br{-\frac{\beta}{2}\ip{\vDelta}{\vx}^2}\geq \exp\br{-\frac{\beta}{2}\norm{\vDelta}^2\rho^2}$. Let $\Sigma$ denote the covariance of distribution $\cD$ and let $\lambda:=\lambda_{min}(\Sigma)$ denote its smallest eigenvalue. This gives us $c\geq e^{-\frac{\beta}{2}\norm{\vDelta}^2\rho^2}\Ee{\vx\sim\cD}{\ip{\vx}{\vv}^2}\geq e^{-\kappa^2\rho/2}\cdot\lambda=e^{-0.11\rho}\cdot \lambda$ using value of $\kappa=0.47$
		\end{description}
	This finishes the proof.
\end{proof}

\begin{lemma}
	\label{lem:wsc-wss}
	Let $X = \bs{\vx_1,\ldots,\vx_n}$ are generated from an isotropic $R$-sub-Gaussian distribution $\cD$ and $G$ denotes the set of uncorrupted points, then there exists distribution specific constant $c$, such that
	\[
	\left.
	\begin{array}{r}
	\P{\lambda_{\min}(X_GS_GX_G^\top) < (1-\zeta)c\cdot G\sqrt{\frac{\beta}{2 \pi}}}\\
	\vspace*{-2ex}\\ 
	\P{\lambda_{\max}(X_GS_GX_G^\top) > (1+\zeta)\cdot G\sqrt{\frac{\beta}{2 \pi}}}
	\end{array}
	\right\} \leq \exp\br{-\Om{n\zeta^2 - d\log\frac1\zeta - d\log(n)}}.
	\]
\end{lemma}
\begin{proof}
	The bound for the largest eigenvalue follows directly due to the fact that all weights are upper bounded by $\sqrt{\frac{\beta}{2 \pi}}$ and hence $X_GS_GX_G^\top \preceq \sqrt{\frac{\beta}{2 \pi}} \cdot X_GX_G^\top$ and applying Lemma~\ref{lem:ssc-sss}. For the bound on the smallest eigenvalue, notice that Lemma~\ref{lem:point-conv} shows us that for any fixed variance $\frac1\beta$, we have
	\[
	\P{\lambda_{\min}(X_GS_GX_G^\top) < \br{1 - \frac\zeta2} c\cdot G\sqrt{\frac{\beta}{2 \pi}}} \leq 2\cdot9^d\exp\bs{-\frac{mn\zeta^2c^2}{128R^4}}
	\]
	Given $R_X := \max_{i \in [n]}\ \norm{\vx_i}_2$, lemma~\ref{lem:approx} shows us that if $\vw^1,\vw^2$ are two models at stage $\frac{1}{\beta}$ variance, such that $\norm{\vw^1 - \vw^2}_2 \leq \tau$ then, the following holds \emph{almost surely}.
	\[
	\abs{\lambda_{\min}(X_GS^1_GX_G^\top) - \lambda_{\min}(X_GS^2_GX_G^\top)} \leq \frac{G\tau \beta R_X^3}{\sqrt{2\pi e}}
	\]
	This prompts us to initiate a uniform convergence argument by setting up a $\tau$-net over $\cB_2\br{\vwo,\sqrt{\frac{2\pi}{\beta}}}$ for $\tau = \frac{\zeta c}{2R_X^3}\sqrt{\frac{e}{\beta}}$. Note that such a net has at most $\br{\frac{6R_X^3}{\zeta c}\sqrt{\frac{2 \pi}{e}}}^d$ elements by applying standard covering number bounds for the Euclidean ball \cite[Corollary 4.2.13]{Vershynin2018}. Taking a union bound over this net gives us
	\begin{align*}
	\P{\lambda_{\min}(X_GS_GX_G^\top) < 0.99c\cdot G\sqrt{\frac{\beta}{2 \pi}}} &\leq 2\cdot\br{\frac{54R_X^3}{\zeta c}\sqrt{\frac{2 \pi}{e}}}^d\exp\bs{-\frac{mn\zeta^2c^2}{128R^4}}\\
	&\leq \exp\br{-\Om{n\zeta^2 - d\log\frac1\zeta - d\log(n)}},
	\end{align*}
	where in the last step we used Lemma~\ref{lem:rx-bound} to bound $R_X = \bigO{R\sqrt{n}}$ with probability at least $1 - \exp(-\Om n)$.
\end{proof}
\subsection{Robust Least-squares Regression with a Fully Adaptive Adversary}
\label{app:rr-adaptive}

To handle a fully adaptive adversary, we need mild modifications to the notions of LWSC and LWLC given in Definition~\ref{defn:lwsc-lwss}, so that the adversary is now allowed to choose the location of corruption arbitrarily. For sake of simplicity, we present these re-definitions and subsequent arguments in the context of robust least-squares regression but similar extensions hold for the other GLM tasks as well. Let us denote the useful shorthands $\bar{\alpha}= 1- \alpha$, and $\cT_{\bar\alpha}= \{T \subset [n]: \abs{T}= \bar\alpha\cdot n\}$ denote the collection of all possible subsets of $\bar\alpha\cdot n$ data points.

\begin{definition}[LWSC/LWLC against Fully Adaptive Adversaries]
\label{defn:lwsc-LWLC-adaptive}
Suppose we are given an exponential family likelihood distribution $\P{\cdot\cond\cdot}$ and data points $\bc{(\vx^i,y_i)}_{i=1}^n$ of which an $\alpha > 0$ fraction has been corrupted by a fully adaptive adversary and $\beta > 0$ is any positive real value. Then, we say that the adaptive $\tilde\lambda_\beta$-local weighted strong convexity property is satisfied if for any model $\vw$, we have
\[
\min_{G \in \cT_{\bar\alpha}} \lambda_{\min}(X_GS_GX_G^\top) \geq \tilde{\lambda}_{\beta},
\]
where $S$ is a diagonal matrix containing the data point scores $s_i$ assigned by the model $\vw$ (see Definition~\ref{defn:lwsc-lwss} for a definition of the scores). Similarly, we say that the adaptive $\tilde\Lambda_\beta$-weighted strong smoothness properties are satisfied if for any \emph{true} model $\vwo$ and any model $\vw \in \cB_2\br{\vwo, \sqrt\frac1\beta}$, we have
\[
\max_{G \in \cT_{\bar\alpha}} \norm{X_BS_B\vb}_2 \leq \tilde{\Lambda}_{\beta},
\]
where we used the shorthand $B = [n] \backslash G$ and $S$ continues to be the diagonal matrix containing the data point scores $s_i$ assigned by the model $\vw$.
\end{definition}

Note that for the setting of robust least-squares regression, for any two models $\hvw,\vw$ we have $\lambda_{\min}(\nabla^2\tilde Q_\beta(\hvw|\vw)) = \lambda_{\min}(XSX^\top) \geq \lambda_{\min}(X_GS_GX_G^\top)$ (since the scores $s_i$ are non-negative) which motivates the above re-definition of adaptive LWSC. For the same setting we also have $\nabla \tilde Q_\beta(\vwo\cond\vw) = X_GS_G\vepsilon_G + X_BS_B\vb$ (with the shorthand $B = [n] \backslash G$) i.e. $\norm{\nabla \tilde Q_\beta(\vwo\cond\vw)}_2 \leq \norm{X_GS_G\vepsilon_G}_2 + \norm{X_BS_B\vb}_2$ by triangle inequality. However, for sake of simplicity we will be analyzing the noiseless setting i.e. $\vepsilon = \vzero$ which motivates the above re-definition of adaptive LWLC. These re-definitions can be readily adapted to settings with hybrid noise i.e. when $\betao < \infty$ as well.

We now show that for the same settings as considered for the proof of Theorem~\ref{thm:rr-main}, the adaptive LWSC and LWLC properties are also satisfied, albeit with worse constants due to the application of union bounds over all possible ``good'' sets in the collection $\cT_{\bar\alpha}$ that the adversary could have left untouched to its corruption.

Lemma~\ref{lem:wsc-wss} gives us, for any fixed set $G \in \cT_{\bar\alpha}$
\begin{align*}
	\P{\lambda_{\min}(X_GS_GX_G^\top) < (1-\zeta) c\cdot G\sqrt{\frac{\beta}{2\pi}}} \leq \exp\br{-\Om{n\zeta^2 - d\log\frac1\zeta - d\log(n)}}
\end{align*}	

By taking union bound over all sets $G \in \cT_{\bar\alpha}$, and observing that \[
{n\choose k}={n\choose n-k}\leq \br{\frac{ne}{n-k}}^{n-k} = \br{\frac{e}{\alpha}}^{\alpha n}= \exp(\alpha n(1 - \log \alpha )),
\]
we have
\[
\P{ \tilde{\lambda}_{\beta} < (1-\zeta)c\cdot G\sqrt{\frac{\beta}{2 \pi}}} \leq  \exp\br{-\Om{n(\zeta^2  + \alpha \log\alpha - \alpha) - d\log\frac1\zeta - d\log(n)}}
\]
which requires, setting $\zeta \geq \Omega(\sqrt{\frac{d \log n}{n} + \alpha - \alpha \log \alpha})$ to obtain a confidence of $1 - \exp(-\Om d)$. This establishes the adaptive LWSC guarantee with a confidence level similar to the one we had for the partially adaptive case. We now establish the adaptive LWLC guarantee.

We notice that Lemma~\ref{lem:ssc-sss} can be extended to the ``adaptive'' setting using \cite[Lemma 15]{BhatiaJK2015} to show that with probability at least $1 - \exp(-\Om d)$, we have
\begin{align*}
    \max_{G \in \cT_{\bar\alpha}} \lambda_{\max}(X_BX_B^\top) &\leq \alpha n\br{1 + 3e\sqrt{6\log\frac e\alpha}} + \bigO{\sqrt{nd}}\\
    &\leq B\br{1 + 3.01e\sqrt{6\log\frac e\alpha}}
\end{align*}
where we continue to use the notation $B = [n]\backslash G$ and for large enough $n$, absorbed the $\sqrt{nd}$ term into the first term by increasing the constant 3 to 3.01 since the second term asymptotically vanishes in comparison to the first term that is linear in $n$. Now, following steps similar to those in the proof of Lemma \ref{lem:rr-LWLC} gives us
\begin{align*}
\max_{G \in \cT_{\bar\alpha}} \norm{X_BS_B\vb}_2 &\leq \norm{X_B}_2\norm{S_B\vb_B}_2\\
&\leq \max_{G \in \cT_{\bar\alpha}}\norm{X_B}_2[(\frac{\kappa^2\lambda_{max}(X_BX_B^T)}{2\pi})^{1/3}+(\frac{B}{2\pi e})^{1/3}]^{\frac32}\\
		&\leq \max_{G \in \cT_{\bar\alpha}} \lambda_{\max}(X_BX_B^\top)[(\frac{\kappa^2}{2\pi})^{1/3}+(\frac{1}{2\pi e})^{1/3}]^{\frac32}\\
		&\leq B\br{1 + 3.01e\sqrt{6\log\frac e\alpha}}[(\frac{\kappa^2}{2\pi})^{1/3}+(\frac{1}{2\pi e})^{1/3}]^{\frac32},
\end{align*}
where the second last step uses the fact that we our upper bound on $\lambda_{\max}(X_BX_B^\top)$ is at least $B$. This establishes the adaptive LWLC property with confidence at least $1 - \exp(-\Om d)$.

\begin{theorem}[Theorem~\ref{thm:rr-fully-adaptive} restated -- Fully Adaptive Adversary]
\label{thm:rr-fully-adaptive-restated}
Suppose data is corrupted by a fully adaptive adversary that is able to decide the location of the corruptions as well as the corrupted labels using complete information of the true model, data features and clean labels, and \gemrr is initialized and executed as described in the statement of Theorem~\ref{thm:rr-main}. Then \gemrr enjoys a breakdown point of $\alpha \leq 0.0036$, i.e. it ensures model recovery even if $k = \alpha\cdot n$ corruptions are introduced by a fully adaptive adversary where the value of $\alpha$ can go upto at least $0.0036$. More specifically, in the noiseless setting where $\betao \rightarrow \infty$ where clean data points do not experience any Gaussian noise i.e. $\epsilon_i = 0$ and $y_i = \ip\vwo{\vx^i}$ for clean points, with probability at least $1 - \exp(-\Om d)$, the LWSC/LWLC conditions are satisfied for all $\beta \in (0,\infty)$ i.e. $\beta_{\max} = \infty$. Consequently, for any $\epsilon > 0$, within $T \leq \bigO{\log\frac 1{\epsilon\beta^1}}$ iterations, we have $\norm{\hvw^T - \vwo}_2^2 \leq \epsilon$.
\end{theorem}

\begin{proof}
The above arguments establishing the adaptive LWSC and adaptive LWLC properties allow us to obtain the following result in a manner similar to that used in the proof of Theorem~\ref{repthm:rr-main} (but in the noiseless setting)
\begin{align*}
	&\norm{\hvw^{t+1} - \vwo}_2 \leq \frac{2\tilde{\Lambda}_{\beta} }{\tilde{\lambda}_{\beta}}\leq \frac{B\br{1 + 3.01e\sqrt{6\log\frac e\alpha}}\sqrt{\frac1{2\pi}}[(1.01\kappa^2)^{1/3}+(\frac{1}{e})^{1/3}]^{\frac{3}{2}}}{\sqrt\frac\beta{2\pi}(1-\zeta)\frac1\rcons\cdot G}\\
	&\leq \frac\kappa{\sqrt\beta}\cdot\underbrace{\frac1{1-\zeta}\br{\frac{\alpha\br{1 + 3.01e\sqrt{6\log\frac e\alpha}}}{1-\alpha}\sqrt{1 + \frac\beta\betao}{\br{1+\kappa^2}^{3/2}}\sqrt{1.01}\bs{(1.01)^{1/3}+\br{e\kappa^2}^{-1/3}}^{\frac32} }}_{(A)}
\end{align*}
Applying the limit $\betao \rightarrow \infty$ (since we are working in the pure corruption setting without any Gaussian noise on labels of the uncorrupted points) transforms the requirement $(A) \leq 1$ (which as before, assures us of the existence of a scale increment $\step > 1$ satisfying the requirements of Theorem~\ref{thm:gem-main}) to:
\begin{align*} \frac1{1-\zeta}\br{\frac{\alpha\br{1 + 3.01e\sqrt{6\log\frac e\alpha}}}{1-\alpha}{\br{1+\kappa^2}^{3/2}}\sqrt{1.01}\bs{(1.01)^{1/3}+\br{e\kappa^2}^{-1/3}}^{\frac32}} \leq 1
\end{align*}
Setting $\kappa=0.47$ as done before further simplifies this requirement to
\[
\frac1{1-\zeta}\br{\frac{\alpha\br{1 + 3.01e\sqrt{6\log\frac e\alpha}}}{1-\alpha}} \leq \frac{1}{4.38}
\]
However, unlike earlier where we could simply set $\zeta$ to an arbitrarily small value for large enough $n$, we cannot do so now since as noted earlier, we must set $\zeta \geq \Omega(\sqrt{\frac{d \log n}{n} + \alpha - \alpha \log \alpha})$ to obtain a confidence of $1 - \exp(-\Om d)$ in the LWSC guarantee. However, for large enough $n$ we can still obtain $\zeta \rightarrow \sqrt{\alpha - \alpha \log \alpha}$ which transforms the requirement further to
\[
\frac1{1-\sqrt{\alpha - \alpha \log \alpha}}\br{\frac{\alpha\br{1 + 3.01e\sqrt{6\log\frac e\alpha}}}{1-\alpha}} \leq \frac{1}{4.38}
\]
which is satisfied at all values of $\alpha = 0.0036$ or smaller. This finishes the proof.
\end{proof}

\section{Robust Mean Estimation}
\label{app:me} 

We will let $G, B$ respectively denote the set of ``good'' uncorrupted points and ``bad'' corrupted points. We will abuse notation to let $G = (1-\alpha)\cdot n$ and $B = \alpha\cdot n$ respectively denote the number of good and bad points too.

\begin{theorem}[Theorem~\ref{thm:me-main} restated]
\label{repthm:me-main}
For data generated in the robust mean estimation model as described in \S\ref{sec:me-rr-lr}, suppose corruptions are introduced by a partially adaptive adversary i.e. the locations of the corruptions (the set $B$) is not decided adversarially but the corruptions are decided jointly, adversarially and may be unbounded, then \gemme enjoys a breakdown point of $0.2621$, i.e. it ensures a bounded $\bigO1$ error even if $k = \alpha\cdot n$ corruptions are introduced where the value of $\alpha$ can go upto at least $0.2621$. More generally, for corruption rates $\alpha \leq 0.2621$, there always exists values of scale increment $\step > 1$ s.t. with probability at least $1 - \exp(-\Om d)$, LWSC/LWLC conditions are satisfied for the $\tilde Q_\beta$ function corresponding to the robust mean estimation model for $\beta$ values at least as large as $\beta_{\max} = \bigO{\frac\betao d\min\bc{\log\frac1\alpha, \sqrt{nd}}}$. If initialized with $\hvmu^1, \beta^1$ s.t. ${\beta_1}\cdot\norm{\hvmu^1 - \vmuo}_2^2 \leq 1$, \gemme assures $
\norm{\hvmu^T - \vmuo}_2^2 \leq \epsilon$ for any $\epsilon \geq \bigO{\trace^2(\Sigma)\cdot\max\bc{{\frac1{\ln(1/\alpha)}}, \frac1{\sqrt{nd}}}}$ within $T \leq \bigO{\log\frac n{\beta^1}}$ iterations.
\end{theorem}
\begin{proof}
For any two models $\vmu,\vdelta$, the $\tilde Q_\beta$ function for robust mean estimation has the following form
\[
\tilde Q_\beta(\vdelta\cond\vmu) = \sum_{i=1}^n s_i\cdot\norm{\vx^i - \vdelta}_2^2,
\]
where $s_i \leftarrow \exp\br{-\frac{\beta}2\norm{\vx^i - \vmu}_2^2}$. We first outline the proof below. 

\emph{Proof Outline.} This proof has four key elements
\begin{enumerate}
	\item We will first establish this result for $\Sigma = \frac1\betao\cdot I$ for $\betao = d$, then generalize the result for arbitrary $\betao > 0$. Note that for $\betao = d$, we have $\trace(\Sigma) = 1$.
	\item To establish the LWSC and LWLC properties, we will first consider as before, a fixed value of $\beta > 0$ for which the properties will be shown to hold with probability $1 - \exp(-\Om d)$. As promised in the statement of Theorem~\ref{thm:me-main}, we will execute \gemme for no more than $\bigO{\log n}$ iterations, taking a naive union bound would offer a confidence level of $1 - \log n\exp(-\Om d)$. However, this can be improved by noticing that the confidence levels offered by the LWSC/LWLC results are actually of the form $1 - \exp(-\Om{n\zeta^2 - d\log n})$. Thus, a union over $\bigO{\log n}$ such events will at best deteriorate the confidence bounds to $1 - \log n\exp(-\Om{n\zeta^2 - d\log n}) = 1 - \exp(-\Om{n\zeta^2 - d\log n - \log\log n})$ which is still $1 - \exp(-\Om d)$ for the values of $\zeta$ we shall set.
	\item The key to this proof is to maintain the invariant $\sqrt{\beta_t}\cdot\norm{\hvmu^t - \vmuo}_2 \leq 1$. Recall that initialization ensures ${\beta_1}\cdot{\norm{\hvmu^1-\vmuo}_2^2}\leq 1$ to start things off. \S\ref{sec:gem} gives details on how to initialize in practice. This establishes the base case of an inductive argument. Next, inductively assuming that ${\beta_t}\cdot{\norm{\hvmu^t-\vmuo}_2^2}\leq 1$ for an iteration $t$, we will establish that $\norm{\hvmu^{t+1} - \vmuo}_2 \leq \frac{2\Lambda_{\beta_t}}{\lambda_{\beta_t}} \leq \frac{(A)}{\sqrt\beta_t}$ where $(A)$ will be an application-specific expression derived below.
	\item We will then ensure that $(A) < 1$, say $(A) = 1/\sqrt\step$ for some $\step > 1$, whenever the number of corruptions are below the breakdown point. This ensures $\norm{\hvmu^{t+1} - \vmuo}_2^2 \leq \frac1{{\step\beta_t}}$, in other words, ${\beta_{t+1}}\cdot{\norm{\hvmu^{t+1}-\vmuo}_2^2}\leq 1$ for $\beta_{t+1} = \step\cdot\beta_t$ so that the invariant is preserved. However, notice that the above step simultaneously ensures that $\frac{2\Lambda_{\beta_t}}{\lambda_{\beta_t}} \leq \frac1{\sqrt{\step\beta_t}}$. This ensures that a valid value of scale increment $\step$ can always be found till $\beta_t \leq \beta_{\max}$. Specifically, we will be able assure the existence of a scale increment $\step > 1$ satisfying the conditions of Theorem~\ref{thm:gem-main} w.r.t the LWSC/LWLC results only till $\beta < \bigO{\frac\betao d\min\bc{\log\frac1\alpha, \sqrt{nd}}}$.
\end{enumerate}

We now present the proof. Lemmata~\ref{lem:me-lwsc},\ref{lem:me-LWLC} establish the LWSC/LWLC properties for the $\tilde Q_\beta$ function for robust mean estimation. Let $\vDelta = \hvmu^t - \vmuo$ and $\vDelta^+ = \hvmu^{t+1} - \vmuo$. To simplify the notation, we will analyze below the updates made with weights scaled up by the constant $\cons = \br{\sqrt{\frac{\beta+{\betao}}{\betao}}}^d\exp\br{\frac{\beta{\betao}}{2(\beta+{\betao})}\norm{\vDelta}_2^2}$ as defined in Lemma~\ref{lem:change-of-exp}. This in no way affects the execution of the algorithm since this scaling factor appears in both the numerator and the denominator of the update terms and simply cancels away. However, it will simplify our presentation. We also, w.l.o.g., first analyze the case of $\betao = d$ first and then scale the space to retrieve a result for general $\betao$.

Using the bound $\cons \leq \exp(\beta + 0.5)$ from Lemma~\ref{lem:cons-bound}, we gives us, upon using Lemmata~\ref{lem:me-lwsc} and \ref{lem:me-LWLC}
\begin{align*}
	\norm{\vDelta^+}_2 \leq \frac{\norm{\sum_{i\in B}s_i\vb^i}_2 + \norm{\sum_{i\in G}s_i\vepsilon^i}_2}{\sum_{i=1}^ns_i} \leq \frac{\frac{B\cons\kappa}{\sqrt\beta} + \norm{\vm}_2 + \frac{\nu G}{\sqrt\beta}}{G(1-\zeta)}\\
	\leq \sqrt\frac1\beta\cdot\frac{\kappa\cons B + G\frac{\beta}{\beta+{d}} + \nu G}{G(1 - \zeta)}\\
	\leq \sqrt\frac1\beta\cdot\underbrace{\br{\frac1{1-\zeta}\bs{\kappa\cons\frac BG + \frac{\beta}{\beta+{d}} + \nu}}}_{(A)},
\end{align*}
where $\kappa = 1 + \sqrt\frac12$. As was the case of robust least squares regression in the proof of Theorem~\ref{repthm:rr-main}, to assure the existence of a scale increment $\step > 1$ satisfying the requirements of Theorem~\ref{thm:gem-main} and hence a linear rate of convergence, all we require is to ensure $(A)$ has a value of the form $\frac1\step < 1$ where $\step > 1$. Now, noting that $R_X \leq \sqrt n$, and promising that we will never set $\beta \geq n$ as well as never set $\nu, \zeta \leq \frac1n$, we note that we need to set $\zeta \geq \Om{\sqrt\frac{d\log(n)}n}$ as well as $\nu \geq \Om{\sqrt\frac{d\beta\log(n)}{n(\beta + d)}}$, in order to ensure a confidence of $1 - \exp(-\Om d)$ in the tail bounds we have established.

\begin{enumerate}
	\item \textbf{Breakdown Point}: as observed above, with large enough $n$, we can set $\zeta, \nu$ to be small, yet positive, constants. For large enough $d$, if we set $\beta = \bigO1$ to be a small enough constant then we have $\frac\beta{\beta + d} \rightarrow 0$, as well as $\cons \leq \exp(0.5 + \beta) \approx \sqrt e$. This means we need only ensure $\br{1+\sqrt\frac12}\sqrt e\frac\alpha{1-\alpha} \leq 1$. The above is satisfied for all $\alpha \leq 0.26.21$ which establishes a breakdown point of $26.21\%$. Note that even at this breakdown point, since we still set $r = \beta = \Om1$, and thus, can still assure $\norm\vDelta_2 \leq \sqrt\frac1\beta = \bigO1$.
	\item \textbf{Consistency}: To analyze the consistency properties of the algorithm, we recall that, ignoring universal constants, to obtain a linear rate of convergence, we need ensure
	\[
	\frac1{1-\zeta}\bs{\cons\alpha + \frac{\beta}{\beta+{d}} + \nu} < 1
	\]
	which can be rewritten as $\frac\beta d + 1\leq \frac1{\zeta + \nu + \cons\alpha}$. Recall from above that we need to set $\zeta \geq \Om{\sqrt\frac{d\log(n)}n}$ as well as $\nu \geq \Om{\sqrt\frac{d\beta\log(n)}{n(\beta + d)}}$. Setting them at these lower bounds, using $\cons \leq \sqrt e\exp(\beta)$, ignoring universal constants (since we are only interested in the asymptotic behavior of the algorithm) and some simple manipulations, we can show that, for all $n \geq d$, we can allow values of $\beta$ as large as	
	\[
	\beta \leq \min\bc{\bigO{\log\frac1\alpha}, \bigO{\sqrt{nd}}}
	\]
	Note that the above assures us, when $\alpha = 0$ i.e. corruptions are absent, an error of $\norm\vDelta_2^2 \leq \frac1\beta \leq \frac1{\sqrt{nd}} \rightarrow 0$ as $n \rightarrow \infty$. Thus, the method is consistent when corruptions are absent.
\end{enumerate}
Scaling the space back up by a factor of $\sqrt\frac d\betao$ gives us the desired result.
\end{proof}

\begin{lemma}[LWSC for Robust Mean Estimation]
\label{lem:me-lwsc}
For any $0 \leq \beta \leq n$, the $\tilde Q_\beta$-function for robust mean estimation satisfies the LWSC property with constant $\lambda_\beta \geq G(1-\zeta)$ with probability at least $1 - \exp(-\Om d)$ for any $\zeta \geq \Om{\sqrt\frac{d\log(n)}n}$.
\end{lemma}
\begin{proof}
It is easy to see that $\nabla^2\tilde Q_\beta(\hvmu\cond\vmu) = (\sum_{i=1}^ns_i)\cdot I$ for any $\hvmu \in \cB_2\br{\vmuo,\sqrt\frac1\beta}$. Applying Lemma~\ref{lem:change-of-exp} gives us
\[
\E{\cons\exp\br{-\frac\beta2\norm{\vepsilon - \vDelta}_2^2}} = 1,
\]
where $\cons = \br{\sqrt{\frac{\beta+{\betao}}{\betao}}}^d\exp\br{\frac{\beta{\betao}}{2(\beta+{\betao})}\norm{\vDelta}_2^2}$ is defined in Lemma~\ref{lem:change-of-exp}. The analysis in the proof of Lemma~\ref{lem:subGaussian-cons}, on the other hand tells us that the random variable $s = \cons\exp\br{-\frac\beta2\norm{\vepsilon - \vDelta}_2^2}$ has a subGaussian constant at most unity. Applying the Hoeffding's inequality for subGaussian variables and noticing $G \geq n/2$ gives us
\[
\P{\sum_{i\in G}s_i \leq G\br{1-\frac\zeta2}} \leq \exp\br{-\frac{m\zeta^2n}{8}},
\]
where $m > 0$ is a universal constant. Again notice that the above result holds for a fixed error vector $\vDelta$. Suppose $\vDelta_1, \vDelta_2 \in \cB\br{\frac1{\sqrt\beta}}$ are two error vectors such that $\norm{\vDelta_1 - \vDelta_2}_2 \leq \tau$. Then, denoting $s^1_i, s^2_i$ as the weights assigned by these two error vectors, for all $\tau \leq \frac2{3\sqrt\beta R_X}$, by applying Lemma~\ref{lem:weight-lip},  we get
\[
\abs{\sum_{i\in G}s^1_i - \sum_{i\in G}s^2_i} \leq 3G\cons\tau\sqrt\beta R_X,
\]
where $R_X$ is the maximum length in a set of $n$ vectors, each sampled from a $d$-dimensional Gaussian (see Lemma~\ref{lem:rx-bound}). Applying a union bound over a $\tau$-net over $\cB_2\br{\sqrt\frac1\beta}$ with $\tau = \frac\zeta{6\sqrt\beta\cons R_X}$ gives us
\[
\P{\exists \vDelta \in \cB_2\br{\sqrt\frac1\beta}: \sum_{i\in G}s_i \leq G(1-\zeta)} \leq \br{\frac{12R_X\cons\sqrt\beta}{\zeta}}^d\exp\br{-\frac{m\zeta^2n}{8}},
\]
Promising that we will always set $\beta < \sqrt\frac nd$ and noting that Lemma~\ref{lem:cons-bound} gives us $\cons \leq \sqrt e\exp(\beta)$ for $\betao = d$ and noting that Lemma~\ref{lem:rx-bound} gives us $R_X \leq \sqrt n$ finishes the proof.
\end{proof}

\begin{lemma}[LWLC for Robust Mean Estimation]
\label{lem:me-LWLC}
For any $0 \leq \beta \leq n$, the $\tilde Q_\beta$-function for robust mean estimation satisfies the LWLC property with constant $\Lambda_\beta \leq G(1+\nu)+\frac{B\cons\kappa}{\sqrt\beta}$ with probability at least $1 - \exp(-\Om d)$ for any $\nu \geq \Om{\sqrt\frac{d\beta\log(n)}{n(\beta + d)}}$.
\end{lemma}
\begin{proof}
It is easy to see that $\nabla \tilde Q_\beta(\vmuo\cond\vmu) = \sum_{i \in G}s_i\vepsilon_i + \sum_{i \in B}s_i\vb^i$. We bound these separately below. Recall that we are working with weights that are scaled by a factor of $\cons$, where $\cons = \br{\sqrt{\frac{\beta+{\betao}}{\betao}}}^d\exp\br{\frac{\beta{\betao}}{2(\beta+{\betao})}\norm{\vDelta}_2^2}$ is defined in Lemma~\ref{lem:change-of-exp}.
\paragraph{Bad Points.} We have $s_i = \cons\exp(-\frac\beta2\norm{\vb^i - \vDelta}_2^2)$ for $i \in B$. Let $\kappa = 1 + \sqrt\frac12$. This gives us two cases
\begin{enumerate}
	\item $\norm{\vb^i}_2 \leq \kappa\norm{\vDelta}$: in this case we use $s_i \leq \cons$ and thus, $s_i\cdot\norm{\vb^i}_2 \leq \kappa\norm{\vDelta} \leq \frac{\cons\kappa}{\sqrt\beta}$
	\item $\norm{\vb^i}_2 > \kappa\norm{\vDelta}$: in this case we have $\norm{\vb^i - \vDelta}_2 \geq \frac{\norm{\vb^i}_2}2$ and thus we also have $s_i \leq \cons\exp(-\frac\beta2\br{1-\frac1\kappa}^2\norm{\vb^i}_2^2)$ which gives us $s_i\cdot\norm{\vb^i}_2 \leq \frac{\cons\kappa}{\sqrt\beta}$ upon using the fact that $x\cdot\exp(-x^2) < \frac12$ for all $x$.
\end{enumerate}
The above tells us, by an application of the triangle inequality, that $\norm{\sum_{i\in B}s_i\vb^i}_2 \leq \frac{B\cons\kappa}{\sqrt\beta}$.

\paragraph{Good Points.} We have $s_i = \cons\exp(-\frac\beta2\norm{\vepsilon^i - \vDelta}_2^2)$ for $i \in G$. Thus, Lemma~\ref{lem:change-of-exp} gives us
\begin{align*}
	\E{\sum_{i \in G}s_i\vepsilon^i} = G\cdot\E{\vx} = G\cdot\frac{\beta}{\beta+{d}}\vDelta =: \vm,
\end{align*}
where $\vx \sim \cN\br{\frac{\beta}{\beta+{d}}\vDelta, \frac1{\beta+{d}}\cdot I}$. Note that since $\beta\norm\vDelta_2^2 \leq 1$, we have
\[
\norm\vm_2 \leq G\cdot\frac{\sqrt\beta}{\beta+{d}}
\]
Using Lemma~\ref{lem:subGaussian-cons} and the linearity of the subexponential norm tells us that the subexponential norm of the random variable $s\cdot\vepsilon^\top\vv = \cons\exp\br{-\frac\beta2\norm{\vepsilon - \vDelta}_2^2}\vepsilon^\top\vv$, for a fixed unit vector $\vv$, is at most $\frac2{\sqrt{\beta + d}}$ (where $\cons = \br{\sqrt{\frac{\beta+{\betao}}{\betao}}}^d\exp\br{\frac{\beta{\betao}}{2(\beta+{\betao})}\norm{\vDelta}_2^2}$ is defined in Lemma~\ref{lem:change-of-exp}). Applying the Bernsteins's inequality for subexponential variables gives us, for some universal constant $m$,
\[
\P{\sum_{i \in G}s_i{\vepsilon^i}^\top\vv > \vm\cdot\vv + t} \leq \exp\br{-m\cdot\min\bc{\frac{t^2(\beta + d)}{G}, t\sqrt{\beta + d}}},
\]
for some universal constant $m > 0$. Now, if $\vv^1, \vv^2 \in S^{d-1}$ are two unit vectors such that $\norm{\vv^1 - \vv^2}_2 \leq \tau$, we have
\[
\abs{\sum_{i \in G}s_i{\vepsilon^i}^\top\vv^1 - \sum_{i \in G}s_i{\vepsilon^i}^\top\vv^2} \leq \norm{\sum_{i \in G}s_i\vepsilon^i}_2\cdot\tau \leq GR_X\tau,
\]
where $R_X$ is the maximum length in a set of $n$ vectors, each sampled from a $d$-dimensional Gaussian (see Lemma~\ref{lem:rx-bound}) and where in the last step we used the triangle inequality and the bounds $s_i \leq 1$ and $\norm{\vepsilon^i}_2 \leq R_X$ for all $i$. Thus, taking a union bound over a $\tau$-net over the surface of the unit sphere $S^{d-1}$ gives us
\[
\P{\exists \vv \in S^{d-1}, \sum_{i \in G}s_i{\vepsilon^i}^\top\vv > \vm\cdot\vv + t + GR_X\tau} \leq \br{\frac2\tau}^d\exp\br{-m\cdot\min\bc{\frac{t^2(\beta + d)}{G}, t\sqrt{\beta + d}}},
\]
The above can be seen as simply an affirmation that $\norm{\sum_{i \in G}s_i\vepsilon^i}_2 \leq \norm\vm_2 + t + GR_X\tau$ with high probability. Setting $t = \frac{\nu G}{4\sqrt\beta}$ and $\tau = \frac{\nu}{4R_X}$, and noticing $G \geq n/2$ gives us, upon promising that we always set $\nu \leq 1$,
\[
\P{\norm{\sum_{i \in G}s_i\vepsilon^i}_2 > \norm{\vm}_2 + \frac{\nu G}{2\sqrt\beta}} \leq \br{\frac{8R_X}\nu}^d\exp\br{-\frac{m\nu^2n}8\cdot\frac{\beta+d}\beta}.
\]
Now notice that this result holds for a fixed error vector $\vDelta \in \cB_2\br{\sqrt\frac1\beta}$. Suppose now that we have two vectors $\vDelta^1, \vDelta^2 \in \cB\br{\frac1{\sqrt\beta}}$ such that $\norm{\vDelta^1 - \vDelta^2}_2 \leq \tau$. If we let $s^1_i$ and $s^2_i$ denote the weights with respect to these two error vectors, then Lemma~\ref{lem:weight-lip} tells us that, for any $\tau$, then we must have
\[
\norm{\sum_{i \in G}(s^1_i - s^2_i)\vepsilon^i}_2 \leq 3GR_X\cons\tau\br{\beta R_X + 2\sqrt\beta}.
\]
Taking a union bound over a $\tau$-net over $\cB_2\br{\sqrt\frac1\beta}$ for $\tau = \frac\nu{6\cons R_X\br{\beta R_X + 2\sqrt\beta}\sqrt\beta}$ gives us
\[
\P{\exists \vDelta \in \cB_2\br{\sqrt\frac1\beta}: \norm{\sum_{i \in G}s_i\vepsilon^i}_2 > \norm{\vm}_2 + \frac{\nu G}{\sqrt\beta}} \leq \br{\frac{24 R_X^2\beta^2\cons}{\nu}}^d\br{\frac{8R_X}\nu}^d\exp\br{-\frac{m\nu^2n}8\cdot\frac{\beta+d}\beta}.
\]
This finishes the proof upon simple modifications and promising that we will always set $\beta < \sqrt\frac nd$ and noting that Lemma~\ref{lem:cons-bound} gives us $\cons \leq \sqrt e\exp(\beta)$ for $\betao = d$.
\end{proof}

\begin{lemma}
\label{lem:weight-lip}
Now, suppose $\vDelta^1, \vDelta^2 \in \cB_2\br{\vzero,\frac1{\sqrt\beta}}$ are two error vectors such that $\norm{\vDelta^1 - \vDelta^2}_2 \leq \tau$ and, for some vector $\vepsilon$, we define $s_i = \cons\exp\br{-\frac\beta2\norm{\vepsilon - \vDelta^i}_2^2}$. Then, for all $\tau$, then we must have
\[
\abs{s^1 - s^2} \leq 3\cons\tau\br{\beta\norm\vepsilon_2 + 2\sqrt\beta}
\]
\end{lemma}
\begin{proof}
Since $\exp(-x)$ is a $1$-Lipschitz function in the region $x \geq 0$, we have
\begin{align*}
	\abs{\exp\br{-\frac\beta2\norm{\vepsilon - \vDelta^1}_2^2} - \exp\br{-\frac\beta2\norm{\vepsilon - \vDelta^2}_2^2}} \leq \abs{-\frac\beta2\norm{\vepsilon - \vDelta^1}_2^2 + \frac\beta2\norm{\vepsilon - \vDelta^2}_2^2}\\
	\leq \frac\beta2\abs{(\vDelta^1 + \vDelta^2 + 2\vepsilon)^\top(\vDelta^1 - \vDelta^2)} \leq 3\tau\br{\beta\norm\vepsilon_2 + 2\sqrt\beta},
\end{align*}
where in the last step we applied the Cauchy-Schwartz inequality and used $\sqrt{\beta}\norm{\vDelta^i}_2\leq 1$ for $i = 1, 2$.
\end{proof}

\begin{lemma}
\label{lem:cons-bound}
For $\betao = d$, we have $\cons \leq \sqrt e\exp(\beta)$ where $\cons = \br{\sqrt{\frac{\beta+{\betao}}{\betao}}}^d\exp\br{\frac{\beta{\betao}}{2(\beta+{\betao})}\norm{\vDelta}_2^2}$ is defined in Lemma~\ref{lem:change-of-exp}.
\end{lemma}
\begin{proof}
We have
\[
\br{\sqrt\frac{\beta+d}d}^d \leq \exp\br{\frac\beta{2}} = \exp\br{\frac{\beta}{2}}.
\]
We also have, using $\beta\norm\vDelta_2^2 \leq 1$,
\[
\exp\br{\frac{\beta{d}}{2(\beta+{d})}\norm{\vDelta}_2^2} \leq \exp\br{\frac{{d}}{2(\beta+{d})}},
\]
by using $\beta\norm{\vDelta}_2^2 \leq 1$. Thus, we have
\[
\cons \leq \exp\br{\frac{{d}}{2(\beta+{d})}+\frac{d\beta}{d}} \leq \exp(\beta + 0.5) 
\]
\end{proof}

\section{Gamma Regression}
\label{app:gam}

\newcommand{\ti}[1]{\tilde{#1}}

For this proof we use the notation $X = [\vx^1, \ldots, \vx^n] \in \bR^{d \times n}, \vy = [y_1, \ldots, y_n] \in \bR_+^n, \vb = [b_1,\ldots,b_n] \in \bR_+^n$. Recall that the labels for gamma distribution are always non-negative and, as specified in Section~\ref{sec:me-rr-lr}, the corruptions are multiplicative in this case. We will let $G, B$ respectively denote the set of ``good'' uncorrupted points and ``bad'' corrupted points. We will abuse notation to let $G = (1-\alpha)\cdot n$ and $B = \alpha\cdot n$ respectively denote the number of good and bad points too.

To simplify the analysis, we assume that data features $\vx^i$ are sampled uniformly from the surface of the unit sphere in $d$-dimensions i.e. $S^{d-1}$. We will also assume that for clean points, labels were generated from mode of Gamma distribution as $y_i = \exp(-\ip\vwo{\vx^i})(1-\phi)$ i.e. in the \emph{no-noise} model. For corrupted points, labels are $\tilde y_i = y_i\cdot b_i$ where $b_i > 0$ but otherwise arbitrary or even unbounded. To simplify the presentation of the results, we will additionally assume $\norm\vwo_2 = 1$ and $\norm{\vw}_2\leq R$.

For any vector $\vv \in \bR^m$ and any set $T \subseteq [m]$, $\vv_T$ denotes the vector with all coordinates other than those in the set $T$ zeroed out. Similarly, for any matrix $A \in \bR^{k \times m}, A_T$ denotes the matrix with all columns other than those in the set $T$ zeroed out.

\subsection{Variance Reduction with the Gamma distribution}
Since the variance reduction step is a bit more involved with the Gamma distribution, we give a detailed derivation here. Consider the gamma distribution:

\begin{align*}
{\cal G}(y_i; \eta_i, \phi)&= \frac{1}{y_i \Gamma(\frac{1}{\phi})}\left(\frac{y_i\eta_i}{\phi}\right)^\frac{1}{\phi}\exp(-\frac{y_i\eta_i}{\phi})\\
&=\exp\left(\frac{y_i\eta_i-\ln\eta_i}{-\phi} + (\frac{1}{\phi}-1)\ln y_i - \frac1{\phi}\ln \phi - \ln \Gamma(\frac{1}{\phi}) \right)\\
\end{align*}

where natural parameter $\eta_i =\exp(\ip{\vwo}{\vx_i})$. The mode preserving variance reduced distribution would be:

\begin{align*}
{\cal G}(y_i; \eta_i, \phi, \beta_t)
&=\exp\left(\beta_t\left(\frac{y_i\eta_i-\ln\eta_i}{-\phi} + (\frac{1}{\phi}-1)\ln y_i - \frac1{\phi}\ln \phi - \ln \Gamma(\frac{1}{\phi})\right)-W(\eta_i,\beta_t, \phi)\right)\\
\end{align*}

writing $\frac{1}{\ti\phi_{\beta_t}}:=\frac{\beta_t}{\phi} -\beta_t+1$ we have,

\begin{align*}
W(\eta_i,\beta_t, \phi)&= \ln\left(\int_{0}^{\infty}\exp\left(\beta_t\left(\frac{y_i\eta_i-\ln\eta_i}{-\phi} + (\frac{1}{\phi}-1)\ln y_i - \frac1{\phi}\ln \phi - \ln \Gamma(\frac{1}{\phi})\right)\right) dy \right)\\
&=\frac{\beta_t}{\phi}\ln\frac{\eta_i}{\phi} - \beta_t\ln \Gamma(\frac{1}{\phi}) +
\frac{1}{\ti\phi_{\beta_t}}\ln \frac{\phi}{\beta_t\eta_i} +\ln \Gamma\left(\frac{1}{\ti\phi_{\beta_t}}\right)
\end{align*}

So that,
\begin{align*}
{\cal G}(y_i; \eta_i, \phi, \beta_t)
&=\exp\left(
-\frac{\beta_ty_i\eta_i}{\phi} +(\frac1{\ti\phi_t}-1)\ln y_i 
-\frac1{\ti\phi_t}\ln \frac{\phi}{\beta_t\eta_i} -\ln \Gamma(\frac1{\ti\phi_t}) \right) \\ 
&=\exp\left(
-\frac{y_i\ti\eta_i}{\ti\phi_{\beta_t}} +(\frac1{\ti\phi_t}-1)\ln y_i-\frac1{\ti\phi_{\beta_t}}\ln \frac{\ti\phi_{\beta_t}}{\ti\eta_i}  -\ln \Gamma(\frac1{\ti\phi_t})\right)\quad \text{setting, }\ti\eta_i:=\frac{\eta_i\beta_t\ti\phi_t }{\phi}\\
&=\cG(y_i; \ti\eta_i,\ti\phi_t)
\end{align*}

Hence to perform variance reduction, following parameter update is sufficient:

\begin{align*}
\ti\eta_i=\frac{\eta_i\beta_t\ti\phi_t }{\phi}= \eta_i \frac{\beta_t}{\phi + \beta_t-\phi \beta_t};\quad \ti\phi_{\beta_t}=\frac{\phi}{\beta_t -\phi \beta_t +\phi}
\end{align*}

We will give the proof for the setting where clean points suffer no noise. Let $(\hat\eta, \hat\phi)$ be the no-noise parameters.  Assuming $0< \phi < 1$ we have $\hat{\eta_i}=\lim\limits_{\beta\rightarrow \infty}\frac{\eta_i \beta}{\phi+\beta- \phi \beta}=\frac{\eta_i }{1 - \phi}=\frac{\exp(\ip{\vwo}{\vx_i}) }{1 - \phi}$ and  $\hat{\phi}=\lim\limits_{\beta\rightarrow \infty}\frac{\phi}{\phi+\beta- \phi \beta}= 0$. Giving, 
\begin{align*}
mode(\cG(y_i; \hat\eta_i,\hat\phi))=\frac{1-\hat\phi}{\hat\eta_i}= \frac{1-\phi}{\eta_i}=mode(\cG(y_i; \eta_i,\phi))
\end{align*}

Let be $b_i \in [0,\infty)$ denote the multiplicative corruption, un-corrupted points having $b_i=1$.\\
We have, 
\begin{align*}
y_i=mode(\cG(y_i; \hat\eta_i,\hat\phi)) \cdot b_i = b_i\exp(-\ip{\vwo}{\vx_i})(1-\phi)
\end{align*}

Let $c:=\frac{1 }{\phi}-1$ and $\Delta^t: = \hvw^t -\vwo$ 
\begin{align*}
s_i&=\cG(y_i; \ti\eta_i,\ti\phi_t)= \frac{\exp(\ip{\vwo}{\vx_i})}{(1-\phi) \Gamma(\frac{1}{\ti\phi_t})}\left(b_ic\beta_t\exp(\ip{\vDelta^t}{\vx_i})\right)^\frac{1}{\ti\phi_t}\exp(-b_ic\beta_t\exp(\ip{\vDelta^t}{\vx_i}))\\
\end{align*}

We may write,

\begin{figure}[H]
\begin{adjustbox}{max width=\textwidth}
\parbox{\linewidth}{
\begin{align*}
\ti Q_{\beta_t}(\vw|\hvw^t) 
&= - \log \prod\limits_{i=1}^{n} \cG(y_i; \frac{\eta_i}{1-\phi}, \hat{\phi})^{s_i}\\
&= \sum\limits_{i=1}^{n} s_i\left(\frac{b_i\exp(\ip{\vw - \vwo}{\vx_i})-\ip{\vw}{\vx_i}+\ln(1-\phi)}{\hat{\phi}} - (\frac{1}{\hat{\phi}}-1)\ln y_i + \frac1{\hat{\phi}}\ln \hat{\phi} + \ln \Gamma(\frac{1}{\hat{\phi}}) \right)\\
\nabla_{\vw}\ti Q_{\beta_t}(\vw|\hvw^t) &= \frac{1}{\hat{\phi}}\sum\limits_{i\in G} \left(b_i\exp(\ip{\vw - \vwo}{\vx_i})-1   \right)s_i\vx_i
\end{align*}
}
\end{adjustbox}
\end{figure}

\begin{theorem}[Theorem~\ref{thm:gam-main} restated]
\label{repthm:gam-main}
For data generated in the robust gamma regression model as described in \S\ref{sec:me-rr-lr}, suppose corruptions are introduced by a partially adaptive adversary i.e. the locations of the corruptions (the set $B$) is not decided adversarially but the corruptions are decided jointly, adversarially and may be unbounded, then \gemgam enjoys a breakdown point of $\frac{0.002}{\sqrt{d}}$, i.e. it ensures a bounded $\bigO1$ error even if $k = \alpha\cdot n$ corruptions are introduced where the value of $\alpha$ can go upto at least $\frac{0.002}{\sqrt{d}}$. More generally, for corruption rates $\alpha \leq \frac{0.002}{\sqrt{d}}$, there always exists values of scale increment $\step > 1$ such that with probability at least $1 - \exp(-\Om d)$, LWSC/LWLC conditions are satisfied for the $\tilde Q_\beta$ function corresponding to the robust gamma regression model for $\beta$ values at least as large as $\beta_{\max} = \bigO{1/\br{\exp\br{\bigO{\alpha \sqrt{d}}} - 1}}$. Specifically, if initialized at $\hvw^1,\beta_1 \geq 1$ such that ${\beta_1}\cdot\br{\exp\br{\norm{\hvw^1 - \vwo}_2} - 1}^2 \leq \frac\phi{1-\phi}$, for any $\epsilon \geq \bigO{\alpha \sqrt{d}}$ \gemgam assures
\[
\norm{\hvw^T - \vwo}_2 \leq \epsilon
\]
within $T \leq \bigO{\log\frac1\epsilon}$ iterations.
\end{theorem}
\begin{proof}
We first outline the proof below. We note that Theorem~\ref{thm:gam-main} is obtained by setting $\phi = 0.5$ in the above statement.

\paragraph{Proof Outline.} The proof outline is similar to the one followed for robust least squares regression and robust mean estimation in Theorems~\ref{repthm:rr-main} and \ref{repthm:me-main} but adapted to suit the alternate parametrization and invariant used by \gemgam for gamma regression. Lemmata~\ref{lem:gam-lwsc},\ref{lem:gam-LWLC} below establish the LWSC and LWLC properties with high probability for $\beta_t \geq 1$. Let $\Delta^t= \hvw^t-\vwo$, Unlike in mean estimation and robust regression where we maintained the invariant $\beta_t\cdot\norm{\Delta^t}_2^2 \leq 1$, in this setting we will instead maintain the invariant $c\beta_t(\exp(\norm{\Delta^t})-1)\leq \frac\phi{1-\phi}$. This is because gamma regression uses a non-standard canonical parameter to enable support only over non-negative labels. Note however that this altered invariant still ensures that as $\beta_t \rightarrow \infty$, we also have $\norm{\Delta^t}_2 \rightarrow 0$. Since we will always use $\beta_t \geq 1$ (since we set $\beta_1 = 1$), we correspondingly require $\norm{\Delta^1}_2 \leq \ln(\frac{1}{c}+1)$ during initialization as mentioned in the statement of Theorem~\ref{thm:gam-main}, where $c=\frac{1}{\phi}-1$. Since we will analyze the special case $\phi = 0.5$ for sake of simplicity, we get $c = 1$.

Lemmata~\ref{lem:gam-lwsc},\ref{lem:gam-LWLC} establish the LWSC/LWLC properties for the $\tilde Q_\beta$ function for robust gamma regression. Given the above proof outline and applying Lemmata~\ref{lem:gam-lwsc},\ref{lem:gam-LWLC} gives us (taking $\norm\vwo_2 = 1$ and $\norm\vw_2 \leq R$)
\begin{align*}
\Delta^{t+1} &\leq \frac{2\Lambda_{\beta_t}}{\lambda_{\beta_t}} \leq  \frac{2\frac{m(\beta_t)B }{\hat\phi}\sqrt{\frac{2}{d}}}{(1-\zeta)\mu_c} \\
&\leq
\frac{ \frac{2B}{\hat{\phi}} \sqrt{\frac{2}{d}}  \frac{1+c\beta_t}{(c\beta_t)^2}\frac{\exp(1)}{(1-\phi) \Gamma(\frac{1}{\ti\phi_{\beta_t}})}\left(c\beta_t+2\right)^{c\beta_t+2}\exp(-c\beta_t-2)}{(1-\zeta)\frac{\exp(\frac{1}{\ti\phi_{\beta_t}}\ln(\frac1{\ti\phi_{\beta_t}}-1))}{\hat{\phi}(1-\phi) \Gamma(\frac{1}{\ti\phi_{\beta_t}})} \frac{G}{d}\exp\left(-R - (c\beta_t+1)\ln(1+\frac{1}{c\beta_t})-1-c\beta_t\right)} \\
&\leq\frac{2\sqrt{2d}B}{G(1-\zeta)}(1+\frac{2}{c\beta_t})^2 \exp(R+3)\\
\end{align*}

For break down point calculation we set $\zeta=0.5$ and $R=1$,

\begin{align*}
\frac{2\sqrt{2d}B}{G(1-\zeta)}(1+\frac{2}{c\beta_t})^2 \exp(R+3) \leq \ln(\frac{1}{\step \beta_t}+1)
\end{align*}

to get $\alpha \leq\frac{0.002}{\sqrt{d}} $

\end{proof}

\begin{lemma}[LWSC for Robust Gamma Regression]
\label{lem:gam-lwsc}
For any $0 \leq \beta_t \leq n$, the $\tilde Q_\beta$-function for gamma regression satisfies the LWSC property with constant $\lambda_\beta \geq (1-\zeta)\mu_c $ with probability at least $1 - \exp(-\Om d)$ where $\mu_c$ is defined in the proof.
\end{lemma}
\begin{proof}
It is easy to see that having  $\frac{1}{\ti\phi_{\beta_t}}= c\beta_t + 1$ we have at the good points,
\begin{align*}
	\nabla^2\tilde Q_G(\vw\cond\hvw^t)&= \frac{1}{\hat\phi}\sum\limits_{i\in G} \exp(\ip{\vw - \vwo}{\vx_i})s_i\vx_i \vx_i^\top\\
	&=\kappa(\phi,\hat\phi,\ti\phi_{\beta_t})\sum\limits_{i\in G} \exp\left(\ip{\vw}{\vx_i}-(\frac1{\ti\phi_{\beta_t}}-1)\exp(\ip{\Delta^t}{\vx_i})+ \frac{1}{\ti\phi_{\beta_t}}\ip{\Delta^t}{\vx_i}\right)\vx_i \vx_i^\top,
\end{align*}
where $\kappa(\phi,\hat\phi,\ti\phi_{\beta_t})=\frac{\exp(\frac{1}{\ti\phi_{\beta_t}}\ln(\frac1{\ti\phi_{\beta_t}}-1))}{\hat{\phi}(1-\phi) \Gamma(\frac{1}{\ti\phi_{\beta_t}})}$.

Let us write, 
\begin{align*}
\varphi(\vx_1,..,\vx_i,.. \vx_n) := \vv^\top \nabla_{\vw}^2\tilde Q_G(\vw|\hvw^t) \vv= \kappa(\phi,\hat\phi,\ti\phi_{\beta_t})\left[\sum\limits_{i\in G} g(\vx_i, \vw, \Delta^t) \ip{\vv}{\vx_i}^2 \right]
\end{align*}

with, 
\begin{align*}
g(\vx_i, \vw, \Delta^t) :=& \exp\left(\ip{\vw}{\vx_i}-(\frac1{\ti\phi_{\beta_t}}-1)\exp(\ip{\Delta^t}{\vx_i})+ \frac{1}{\ti\phi_{\beta_t}}\ip{\Delta^t}{\vx_i}\right)\\
\geq &\exp\left(-R - (c\beta_t+1) \norm{\Delta^t}-c\beta_t\exp(\norm{\Delta^t})\right)\quad \text{using, }\ip{\Delta^t}{\vx_i} \geq -\norm{\Delta^t}\\
\geq &\exp\left(-R - (c\beta_t+1)\ln(1+\frac{1}{c\beta_t})-1-c\beta_t\right)=:g_{min}\\
\end{align*}

and, \begin{align*}
g(\vx_i, \vw, \Delta^t) \leq &\exp\left(R +(c\beta_t+1)\norm{\Delta^t}-c\beta_t\exp(-\norm{\Delta_t})\right)\\
\leq &\exp\left(R +(c\beta_t+1)\ln(1+\frac{1}{c\beta_t})-\frac{(c\beta_t)^2}{1+c\beta_t}\right)=:g_{\max}
\end{align*}

We have,

\begin{align*}
\mu:=\Ee{\vx_i\sim {\cal S}^{d-1}}{\varphi(\vx_1,..,\vx_i,.. \vx_n)} \geq  G \kappa(\phi,\hat\phi,\ti\phi_{\beta_t}) g_{\min} \Ee{\vx_i\sim {\cal S}^{d-1}}{\ip{\vv}{\vx_i}^2}=\frac{G\kappa(\phi,\hat\phi,\ti\phi_{\beta_t})g_{\min}}{d}=:\mu_c
\end{align*}

To get a high-probability lower bound on the LWSC constant $\lambda_\beta$  we would like to apply McDiarmid's inequality. Having, 
\begin{align*}
    &\abs{\varphi(\vx_1,..,\vx_i,..,\vx_n)-\varphi(\vx_1,..,\vx_i',..,\vx_n)}\\
    &=\kappa(\phi,\hat\phi,\ti\phi_{\beta_t})\left| g(\vx_i, \vw, \Delta^t) \ip{\vv}{\vx_i}^2 - g(\vx_i', \vw, \Delta^t) \ip{\vv}{\vx_i'}^2\right|\\
    &\leq\kappa(\phi,\hat\phi,\ti\phi_{\beta_t})( g(\vx_i, \vw, \Delta^t) + g(\vx_i', \vw, \Delta^t))\\
    &\leq 2\kappa(\phi,\hat\phi,\ti\phi_{\beta_t})g_{\max}
\end{align*},
we may write for any fixed $\vv \in {\cal S}^{d-1}$ and $\vw \in {\mathbb R}^d$,
\begin{align*}
&{\mathbb P}(\abs{\varphi(\vx_1,..,\vx_i,..,\vx_n)- \mu}  \geq t) \leq 2\exp\left(\frac{-2t^2}{2G\kappa(\phi,\hat\phi,\ti\phi_{\beta_t})g_{\max}}\right)
\end{align*}

For any square  symmetric matrix $F\in \mathbb{R}^{d \times d}$, we have, $\norm{F}_2\leq \frac{1}{1-2\epsilon}\sup\limits_{\vv \in {\cal N}_\epsilon}\abs{\vv^\top F \vv}$. Taking, $  F=\nabla_{\vw}^2\tilde Q_G(\vw|\hvw^t) -\mu I_d$ gives $\norm{\nabla_{\vw}^2\tilde Q_G(\vw|\hvw^t) -\mu I_d}_2\leq 2 \sup\limits_{\vv \in {\cal N}_\frac{1}{4}}\abs{\vv^\top\nabla_{\vw}^2\tilde Q_G(\vw|\hvw^t)\vv -\mu}$

Taking union bound over $\frac{1}{4}$-net of $\vv\in \mathbb{R}^d$ gives
\begin{align*}
{\mathbb P}\left(\norm{\nabla_{\vw}^2\tilde Q_G(\vw|\hvw^t) -\mu I_d}_2 \geq t \right) &\leq {\mathbb P}\left(2 \sup\limits_{\vv \in {\cal N}_\frac{1}{4}}\abs{\vv^\top\nabla_{\vw}^2\tilde Q_G(\vw|\hvw^t)\vv - \mu} \geq t \right)\\
&\leq 2 \cdot 9^d \exp\left(\frac{-t^2}{4G\kappa(\phi,\hat\phi,\ti\phi_{\beta_t})g_{\max}}\right)
\end{align*}

Since, $\lambda_{\beta_t}:=\lambda_{\min}(\nabla_{\vw}^2\tilde Q(\vw|\hvw^t)) \geq \lambda_{\min}(\nabla_{\vw}^2\tilde Q_G(\vw|\hvw^t))$ and $\mu \geq \mu_c$, we have, 

${\mathbb P}(\lambda_{\beta_t}-\mu_c \leq -t )\leq {\mathbb P}(\lambda_{\min}(\nabla_{\vw}^2\tilde Q_G(\vw|\hvw^t))-\mu_c \leq -t )\leq {\mathbb P}(\lambda_{\min}(\nabla_{\vw}^2\tilde Q_G(\vw|\hvw^t))-\mu \leq -t ) \leq {\mathbb P}(\abs{\lambda_{\min}(\nabla_{\vw}^2\tilde Q_G(\vw|\hvw^t))-\mu} \geq t )\leq {\mathbb P}(\norm{\nabla_{\vw}^2\tilde Q_G(\vw|\hvw^t)-\mu I_d}_2 \geq t )\leq 2 \cdot 9^d \exp\left(\frac{-t^2}{4G\kappa(\phi,\hat\phi,\ti\phi_{\beta_t})g_{\max}}\right)$

Put,  $t=\frac{\zeta\mu_c}{2}$ with $0<\zeta<1$, for any fixed $\vw \in \mathbb{R}^d$ we have: 
\begin{align*}
&{\mathbb P}\left(\lambda_{\min}(\nabla_{\vw}^2\tilde Q_G(\vw|\hvw^t)) \leq (1-\frac{\zeta}{2})\mu_c \right)
\leq 2 \cdot 9^d \exp\left(\frac{-\zeta^2 \mu_c^2}{16G\kappa(\phi,\hat\phi,\ti\phi_{\beta_t})g_{\max}}\right)
\end{align*}

In order to take union bound over $\vw$ we observe, let $\tau=\norm{\hvw^1 - \hvw^2}_2$ and using $\exp(x) \leq 1 + 2x$ for $0\leq x \leq 1$

$g(\vx_i, \hvw^1, \Delta^t) -g(\vx_i, \hvw^2, \Delta^t) = g(\vx_i, \hvw^2, \Delta^t)(\exp\left(\ip{\hvw^1 - \hvw^2}{\vx_i}\right) -1) \leq 2\tau g(\vx_i, \hvw^2, \Delta^t)\leq 2\tau g_{\max}$

So that,

\begin{align*}
&\abs{\lambda_{\min}(\nabla_{\hvw^1}^2\tilde Q_G(\vw|\hvw^t)) - \lambda_{\min}(\nabla_{\hvw^2}^2\tilde Q_G(\vw|\hvw^t))}\leq \sup_{\norm{\vv}_2=1}\abs{\vv^\top (\nabla_{\hvw^1}^2\tilde Q_G(\vw|\hvw^t)-\nabla_{\hvw^2}^2\tilde Q_G(\vw|\hvw^t)) \vv}\\
&= \kappa(\phi,\hat\phi,\ti\phi_{\beta_t})\sup_{\norm{\vv}_2=1}\abs{\left[\sum\limits_{i\in G} \left(g(\vx_i, \hvw^1, \Delta^t) -g(\vx_i, \hvw^2, \Delta^t)\right) \ip{\vv}{\vx_i}^2 \right]}\leq 2\kappa(\phi,\hat\phi,\ti\phi_{\beta_t})\tau G g_{\max}
\end{align*}

In order to set the $\tau$-net, we would require, $(1-\zeta)\mu_c \leq  (1-\frac{\zeta}{2})\mu_c - 2\kappa(\phi,\hat\phi,\ti\phi_{\beta_t})\tau G g_{\max}$

\begin{align*}
\tau \leq   \frac{\zeta\mu_c}{4\kappa(\phi,\hat\phi,\ti\phi_{\beta_t})Gg_{\max}}
=\frac{\zeta\frac{G\kappa(\phi,\hat\phi,\ti\phi_{\beta_t})}{d}g_{\min}}{4\kappa(\phi,\hat\phi,\ti\phi_{\beta_t})Gg_{\max}}=\frac{\zeta g_{\min}}{4d g_{\max}}
\end{align*}

Taking covering number of ${\cal B}(\vwo, R) \leq R^d(\frac{3}{\tau})^d=(\frac{12dRg_{\max}}{\zeta g_{\min}})^d$ and observing,

$\frac{g_{\max}}{g_{\min}} \leq (1+\frac{1}{c\beta_t})^{2}\exp\left(2R+3+\frac{c\beta_t}{1+c\beta_t}\right)$,  $\frac{\kappa(\phi,\hat\phi,\ti\phi_{\beta_t})g_{\min}^2}{g_{\max}}\geq \frac{\beta_t}{\phi\hat{\phi}}(1+\frac{1}{c\beta_t})^{-3}\exp\left(-3R-5-\frac{c\beta_t}{1+c\beta_t}\right)$

We may write,
\begin{figure}[H]
\begin{adjustbox}{max width=\textwidth}
\parbox{\linewidth}{
\begin{align*}
&{\mathbb P}\left(\exists \vw \in {\cal B}(0, R): \lambda_{\min}(\nabla_{\vw}^2\tilde Q_G(\vw|\hvw^t)) \leq (1-\zeta)\mu_c \right)
\leq 2 \cdot 9^d \cdot (\frac{12 dRg_{\max}}{\zeta g_{\min}})^d \exp\left(\frac{-\zeta^2 \mu_c^2}{16G\kappa(\phi,\hat\phi,\ti\phi_{\beta_t})g_{\max}}\right)\\
&=2 \cdot  \exp\left(\frac{-\zeta^2 G\kappa(\phi,\hat\phi,\ti\phi_{\beta_t})g_{\min}^2}{16d^2g_{\max}} +d\ln(\frac{108 dRg_{\max}}{\zeta g_{\min}})\right)=\exp\left(-\Omega(\frac{\zeta^2 G\beta_t }{d^2\phi\hat{\phi}}- \frac{d\ln d}{\zeta})\right)
\end{align*}
}
\end{adjustbox}
\end{figure}
This finishes the proof.
\end{proof}

\begin{lemma}[LWLC for Robust Gamma Regression]
\label{lem:gam-LWLC}
For any $0 \leq \beta \leq n$, the $\tilde Q_\beta$-function for robust gamma regression satisfies the LWLC property with constant $\Lambda_\beta \leq \frac{m(\beta_t)B }{\hat\phi}\sqrt{\frac{2}{d}}$ with probability at least  $1 - \exp(-\Om d)$, where $m(\beta_t)$ is defined in proof.
\end{lemma}
\begin{proof}
It is easy to see that $\nabla\tilde Q_{\beta_t}(\vwo\cond\vw) = \frac{1}{{\hat\phi}}\sum\limits_{i \in B} (b_i - 1)s_i\vx_i$. Since $b_i = 1$ for good points in the no-noise setting we have assumed, good points do not contribute to the gradient at all. Thus, we get
\[
\norm{\nabla\tilde Q_{\beta_t}(\vwo\cond\vw)}_2 = \frac{1}{\hat\phi}\norm{X_B\vt}_2 \leq \frac{1}{\hat\phi}\norm{X_B}_2\norm\vt_2,
\]
where $\vt = [t_i]_{i \in B}$ and $t_i = (b_i - 1)\cdot s_i$. Now since $\vx_i \in {\cal S}^{d-1}$ with probability at least $1 - \exp(-\Om d)$, we have $\norm{X_B}_2 \leq \sqrt \frac{B}{d}$ (since the locations of the corruptions were chosen randomly without looking at data). Thus, we are left to bound $\norm\vt_2$. We have $t_i \leq (b_i-1)^2s_i^2 \leq (b_i^2 + 1)s_i^2$.

Let $z_i:=c\beta_t\exp(\ip{\vDelta^t}{\vx_i})$, we have,  $\frac{\exp(\ip{\vwo}{\vx_i})}{z_i}=\frac{1}{c\beta_t}\frac{\exp(\ip{\vwo}{\vx_i})}{\exp(\ip{\vDelta^t}{\vx_i})} \leq \frac{1}{c\beta_t}\exp(\norm{\vwo-\Delta^t}_2)\leq \frac{1+c\beta_t}{(c\beta_t)^2} \exp(1)$

\begin{align*}
b_is_i&= b_i\frac{\exp(\ip{\vwo}{\vx_i})}{(1-\phi) \Gamma(\frac{1}{\ti\phi_{\beta_t}})}\left(b_iz_i\right)^\frac{1}{\ti\phi_{\beta_t}}\exp(-b_iz_i)
= \frac{1}{z_i}\frac{\exp(\ip{\vwo}{\vx_i})}{(1-\phi) \Gamma(\frac{1}{\ti\phi_{\beta_t}})}\left(b_iz_i\right)^{\frac{1}{\ti\phi_{\beta_t}}+1}\exp(-b_iz_i)\\
&\leq \frac{1}{z_i}\frac{\exp(\ip{\vwo}{\vx_i})}{(1-\phi) \Gamma(\frac{1}{\ti\phi_{\beta_t}})}\left(\frac{1}{\ti\phi_{\beta_t}}+1\right)^{\frac{1}{\ti\phi_{\beta_t}}+1}\exp(-\frac{1}{\ti\phi_{\beta_t}}-1) \quad \text{since, }  x^a \exp(-x) \leq a^a \exp(-a) \text{ for, } x>0\\
&\leq \frac{1+c\beta_t}{(c\beta_t)^2}\frac{\exp(1)}{(1-\phi) \Gamma(\frac{1}{\ti\phi_{\beta_t}})}\left(c\beta_t+2\right)^{c\beta_t+2}\exp(-c\beta_t-2)=:m(\beta_t)\\
\end{align*}

And using,	$\exp(\ip{\hvw^t}{\vx_i}) \leq \exp(\norm{\hvw^t}_2) \leq \exp(\norm{\vwo +\Delta^t}_2) \leq (1+\frac{1}{c\beta_t}) \exp(1) $ and density at the mode of Gamma,
\begin{align*}
s_i&=f(y_i; \ti\eta_i,\ti\phi_{\beta_t}) \leq f( \frac{1-\ti\phi_{\beta_t}}{\ti\eta_i^t}; \ti\eta_i,\ti\phi_{\beta_t})= \frac{1}{\frac{1-\ti\phi_{\beta_t}}{\ti\eta_i} \Gamma(\frac{1}{\ti\phi_{\beta_t}})}\left(\frac{1-\ti\phi_{\beta_t}}{\ti\eta_i}\frac{\ti\eta_i}{\ti\phi_{\beta_t}}\right)^\frac{1}{\phi}\exp(-\frac{1-\ti\phi_{\beta_t}}{\ti\eta_i}\frac{\ti\eta_i}{\ti\phi_{\beta_t}})\\
&= \frac{\exp(\ip{\hvw^t}{\vx_i})}{(1-\phi) \Gamma(\frac{1}{\ti\phi_{\beta_t}})}\left(\frac{1-\ti\phi_{\beta_t}}{\ti\phi_{\beta_t}}\right)^\frac{1}{\ti\phi_{\beta_t}}\exp(-\frac{1-\ti\phi_{\beta_t}}{\ti\phi_{\beta_t}})\\
&= (1+\frac{1}{c\beta_t})\frac{\exp(1)}{(1-\phi) \Gamma(\frac{1}{\ti\phi_{\beta_t}})}\left(c\beta_t\right)^{c\beta_t+1}\exp(-c\beta_t)\\
&\leq m(\beta_t)
\end{align*}

which gives us

\begin{align*}
\Lambda_{\beta_t}\leq \frac{1}{\hat\phi}\norm{X_B}_2\norm{\vt}_2 \leq \frac{1}{\hat\phi}\norm{X_B}_2\sqrt{\sum\limits_{i\in B}(b_is_i)^2+s_i^2} \leq \frac{m(\beta_t)B }{\hat\phi}\sqrt{\frac{2}{d}}
\end{align*}

\end{proof}

\end{document}